\documentclass[10pt,twocolumn,letterpaper]{article}

\newif\ifappendix
\appendixtrue

\usepackage[letterpaper,textwidth=506.295pt,textheight=630pt,centering]{geometry}
\usepackage[T1]{fontenc}
\usepackage[utf8]{inputenc}
\usepackage{libertine}
\usepackage{amsmath}

\usepackage{amssymb}
\usepackage[libertine]{newtxmath}

\usepackage{amsthm}
\usepackage{balance}
\usepackage{graphicx}
\usepackage{algorithm}
\usepackage{algpseudocode}
\usepackage{mathtools}
\providecommand{\Description}[1]{}

\newcommand{\R}{\mathbb{R}}
\newcommand{\E}{\mathbb{E}}
\newcommand{\norm}[1]{\left\lVert #1 \right\rVert}
\newcommand{\inner}[2]{\left\langle #1, #2 \right\rangle}
\newcommand{\gmax}{G}
\newcommand{\Rmax}{R_{\max}}
\newcommand{\GJ}{G_J}
\newcommand{\LJ}{L_J}
\newcommand{\rJ}{\rho_J}
\newcommand{\GF}{G_F}
\newcommand{\LF}{L_F}
\newcommand{\Hs}{\mathsf H}
\newcommand{\bhat}{\hat b}
\newcommand{\shat}{\hat\sigma}

\ifappendix
  \newcommand{\supp}{Appendix~\ref{app:proofs}}
  
  \newcommand{\suppexp}{Appendix~\ref{app:exp}}
  \newcommand{\suppalg}{Appendix~\ref{app:alg}}
  \newcommand{\supplim}{Appendix~\ref{app:limits}}
\else
  \newcommand{\supp}{the supplementary material}
  
  \newcommand{\suppexp}{the supplementary material}
  \newcommand{\suppalg}{the supplementary material}
  \newcommand{\supplim}{the supplementary material} 
\fi

\usepackage{tikz}
\usetikzlibrary{arrows.meta,positioning,calc,shapes.geometric}
\usepackage{booktabs}
\definecolor{srv}{HTML}{0072B2}
\definecolor{agt}{HTML}{E69F00}
\definecolor{ggreen}{HTML}{009E73}
\definecolor{ggray}{HTML}{7F7F7F}
\usepackage{url}
\usepackage[colorlinks=true,allcolors=blue,breaklinks=true]{hyperref}
\theoremstyle{plain}
\newtheorem{theorem}{Theorem}
\newtheorem{lemma}{Lemma}
\newtheorem{proposition}{Proposition}
\newtheorem{corollary}{Corollary}
\theoremstyle{definition}
\newtheorem{definition}{Definition}
\newtheorem{assumption}{Assumption}
\newtheorem{remark}{Remark}
\newtheorem{example}{Example}

\title{\bfseries Personalized Federated Reinforcement Learning via Model-Agnostic Meta-Learning: Convergence of Exact and Hessian-Free Meta-Policy Gradients}
\author{%
Ali Beikmohammadi\thanks{Department of Computer and Systems Sciences, Stockholm University, Stockholm, Sweden. \texttt{beikmohammadi@dsv.su.se}}
\and Sarit Khirirat\thanks{Faculty of Information and Communication Technology, Mahidol University, Nakhon Pathom, Thailand. \texttt{sarit.khi@mahidol.ac.th}}
\and Sindri Magn\'usson\thanks{Department of Computer and Systems Sciences, Stockholm University, Stockholm, Sweden. \texttt{sindri.magnusson@dsv.su.se}}}

\begin{document}
\maketitle

\begin{abstract}
We study personalized federated reinforcement learning, in which $n$ agents, each acting in its own Markov decision process, collaborate through a server to learn a shared MAML-style policy initialization that becomes effective for an individual agent once that agent adapts it with a single local policy-gradient step. We propose \textsc{Per-FedAvg-PG}, in which agents take $\tau$ local stochastic meta-policy-gradient steps between communication rounds, and prove that it reaches an $\varepsilon$-approximate first-order stationary point of the personalized objective in $K=\mathcal O(\varepsilon^{-3/2})$ rounds with $\tau=\Theta(\varepsilon^{-1/2})$ local steps. The analysis rests on a structural feature of the reinforcement learning setting: under standard policy-class regularity, the per-agent objectives have uniformly bounded gradients and Hessians with explicit constants, so the bounded-gradient and bounded-heterogeneity conditions imposed by the supervised theory hold automatically and no separate heterogeneity assumption is needed. The exact meta-gradient requires the inner-loop policy Hessian, which our experiments identify as the practical bottleneck. We therefore analyze the Hessian-free variant, bound its bias, and exhibit fixed points at which the meta-gradient is nonzero and of order $\alpha$, showing that the resulting stationarity floor is a property of the method rather than of the bound. Experiments on tabular and neural navigation confirm the predicted behavior and show transfer to unseen agents at an order of magnitude lower sample cost than independent training. Together these results identify the adaptation step size as a tunable personalization knob and the curvature estimate as the quantity that governs whether exact meta-gradients are affordable.
\end{abstract}

\medskip\noindent\textbf{Keywords:} Federated reinforcement learning; personalization; meta-learning; MAML; policy gradient; convergence analysis\medskip

\section{Introduction}
Many multiagent systems consist of autonomous agents solving \emph{the same kind of task under different circumstances}. A fleet of service robots navigates different buildings; traffic controllers regulate different intersections; edge controllers operate under different network conditions; assistive agents serve different users. Each agent faces its own Markov decision process (MDP); the MDPs share structure but differ in dynamics, reward, or initial conditions. Learning in isolation wastes the population's experience and is sample-inefficient. This matters because reinforcement learning (RL) is data-hungry and each agent's data is generated by real interaction. Federated reinforcement learning (FRL) offers the alternative: agents periodically exchange \emph{parameters} rather than trajectories with a central server, which preserves the locality of raw experience and amortizes sample cost across the population~\cite{lfrl,khodadadian,woo,compfedrl}, in the spirit of the parallel data collection that made distributed deep RL practical~\cite{asyndeepreinforce}.

Under environment heterogeneity, however, FRL inherits a well-documented failure mode. If the population optimizes one shared policy, the result is an average compromise that can be poor for \emph{every} individual agent, and existing guarantees for a single shared policy carry a non-vanishing heterogeneity bias~\cite{jin22,fedsarsa,labbi25,fastfedpg,fedtd}. Because each agent is autonomous and is evaluated on \emph{its own} environment rather than on the population average, it should deploy a policy tailored to itself. This motivates \emph{personalized} FRL, in which the product of collaboration is not a policy but a \emph{starting point} from which each agent quickly obtains such a policy.

In supervised federated learning (FL), the canonical realization of this idea is Per-FedAvg~\cite{perfedavg}, which adopts the Model-Agnostic Meta-Learning (MAML) viewpoint~\cite{maml}: rather than a model that is good on average, clients seek a common \emph{initialization} that becomes good for each client after one or a few local gradient steps. The initialization is trained federatively, but the model each client deploys is its own. Our goal is to carry this construction into RL, with guarantees.

\paragraph{The gap.} The two ingredients have so far been developed separately. Convergence theory for MAML in RL exists only in the \emph{centralized} setting: SG-MRL~\cite{sgmrl} gives the first guarantees when adaptation is a stochastic policy-gradient step, and MAML-LQR~\cite{mamllqr} analyzes the policy-gradient MAML objective for the linear--quadratic regulator. Convergence theory for \emph{federated} RL, meanwhile, has been established for a single non-personalized shared policy~\cite{fednpg,fedsarsa,labbi25,fastfedpg,afedpg}. Personalized FRL guarantees do exist, but for \emph{other} personalization mechanisms: global--local mixing~\cite{zoapfpg,pfopg}, shared representations with local heads~\cite{xiong25}, and per-task policies constrained to a meta-policy~\cite{pmetarl}, the last centralized rather than federated. None of these learns a MAML-style common initialization adapted by a local policy-gradient step, the mechanism that inherits both the Per-FedAvg deployment model and the meta-RL analysis toolbox. We close this gap for full participation.

\paragraph{Three departures from the supervised theory.} 
Replacing a supervised loss $f_i$ by a return $J_i$ is not a renaming exercise. \emph{(i) Distribution shift is intrinsic.} $\nabla J_i(\theta)$ is an expectation over trajectories generated \emph{by the parameter being differentiated}, so the outer batch must be resampled under the \emph{adapted} policy $\tilde\theta$, which is itself random. \emph{(ii) The meta-gradient needs a policy Hessian.} In RL, $\nabla^2 J_i$ has a score-function form whose Monte Carlo estimate is expensive and high-variance in a way with no supervised analogue; Section~\ref{sec:experiments} shows that this, not the outer gradient, is the binding practical constraint. \emph{(iii) RL gives structure back.} Bounded rewards plus policy regularity yield \emph{per-trajectory} bounds on gradients and Hessians, so the bounded-gradient and bounded-dissimilarity conditions assumed in~\cite{perfedavg} become theorems. Item (iii) is what makes the federated analysis clean; items (i)--(ii) are what make it non-trivial.

\paragraph{Contributions.} We formulate personalized FRL as $\max_\theta F(\theta)=\frac1n\sum_i J_i(\theta+\alpha\nabla J_i(\theta))$, the direct RL analogue of the Per-FedAvg objective, and propose \textsc{Per-FedAvg-PG}, in which each agent takes $\tau$ local stochastic meta-policy-gradient steps per communication round (Sections~\ref{sec:setup}--\ref{sec:algo}). On the theory side we prove uniform bounds, with explicit constants, on the per-agent gradients and Hessians (Lemma~\ref{lem:pg-bounds}), so that no bounded-gradient or bounded-hetero\-geneity assumption is needed (Remark~\ref{rem:free}); we prove that \textsc{Per-FedAvg-PG} reaches an $\varepsilon$-approximate first-order stationary point in $K=\mathcal O(\varepsilon^{-3/2})$ rounds with $\tau=\Theta(\varepsilon^{-1/2})$ local steps (Theorem~\ref{thm:main}, Corollary~\ref{cor:rates}), matching the round complexity known for Per-FedAvg; we analyze the first-order (Hessian-free) variant, bounding its bias and showing that the resulting neighborhood is real rather than an analysis artifact, because fixed points of the first-order dynamics carry a residual meta-gradient of order $\alpha$ (Corollary~\ref{cor:fo}, Proposition~\ref{prop:fofix}); and we sharpen the inner-loop bias by a full $\sqrt{m_{\mathrm{in}}}$ factor, improving per-agent trajectory complexity to $\mathcal O(\varepsilon^{-5/2})$ (Proposition~\ref{prop:sharpbias}), with $\alpha\to0$ recovering non-personalized federated policy optimization (Proposition~\ref{prop:alpha0}). Empirically (Section~\ref{sec:experiments}) we validate the theory on tabular navigation, where $F$ and $\nabla F$ are computed exactly by dynamic programming, and on neural navigation, which exposes the curvature bottleneck and demonstrates zero-shot transfer plus few-shot personalization to agents that never participated in training.

\begin{table*}[t]
\centering\footnotesize
\setlength{\tabcolsep}{4pt}
\caption{Position of this paper among the most closely related works. ``$\varepsilon$-FOSP'' abbreviates convergence to an $\varepsilon$-approximate first-order stationary point. 
Unlike Per-FedAvg, our guarantees need no bounded-gradient or bounded-heterogeneity assumption: both are theorems in the policy-gradient setting (Lemma~\ref{lem:pg-bounds}, Remark~\ref{rem:free}).}
\label{tab:related}
\begin{tabular}{@{}lccll@{}}
\toprule
 & federated & RL & personalization mechanism & convergence guarantee \\
\midrule
Per-FedAvg~\cite{perfedavg} & \checkmark & --- & MAML init.\ (1 step) & $\varepsilon$-FOSP, $K=\mathcal O(\varepsilon^{-3/2})$ \\
SG-MRL~\cite{sgmrl} & --- & \checkmark & MAML init.\ ($\ge1$ steps) & $\varepsilon$-FOSP, finite-sample \\
MAML-LQR~\cite{mamllqr} & --- & LQR & MAML init.\ (PG) & stabilizing, up to heterogeneity bias \\
FedNPG~\cite{fednpg} & \checkmark & \checkmark & none (shared policy) & global opt.\ (tabular, entropy-reg.) \\
FedSARSA / FedPG~\cite{fedsarsa,labbi25,fastfedpg} & \checkmark & \checkmark & none (shared policy) & finite-time, up to heterogeneity gap \\
Shared representations~\cite{xiong25} & \checkmark & \checkmark & shared repr.\ + local heads & finite-time, linear speedup \\
ZO-APFPG~\cite{zoapfpg}, PfoPG~\cite{pfopg} & \checkmark & \checkmark & global--local mixing & stationarity (zeroth-/first-order) \\
pMeta-RL~\cite{pmetarl} & --- & \checkmark & meta-constrained per-task policies & tabular convergence \\
\textbf{This paper} & \checkmark & \checkmark & MAML init.\ (1 PG step) & $\varepsilon$-FOSP, $K=\mathcal O(\varepsilon^{-3/2})$; + Hessian-free variant \\
\bottomrule
\end{tabular}
\end{table*}

\paragraph{Related work.}
\textbf{Personalized FL.} Per-FedAvg~\cite{perfedavg} introduced the MAML formulation and proved $\mathcal O(\varepsilon^{-3/2})$-round convergence for nonconvex losses under uniformly bounded gradients and bounded gradient dissimilarity. Follow-up work relaxes these conditions: centralized MAML is analyzed without bounded gradients via an unbounded smoothness parameter and adaptive step sizes~\cite{aistatsmaml}; Moreau-envelope formulations are studied in~\cite{moreau,pfedme}; asynchronous personalized FL is treated in~\cite{persafl}; Per-FedAvg is generalized to $\nu$ adaptation steps with exact, first-order and Hessian-free analyses~\cite{jamali}, complementing the multi-step centralized theory of~\cite{jimulti}; and personalization through a shared representation is studied in~\cite{collins}. Client drift under local updates and its control by server-side corrections are studied in~\cite{fedavg,scaffold}.
\textbf{Meta-RL.} SG-MRL~\cite{sgmrl} gives the first convergence guarantees for stochastic-gradient MAML-RL in the centralized setting, and MAML-LQR~\cite{mamllqr} analyzes the LQR case.
\textbf{FRL.} For a single shared policy, linear speedups are known for value-based methods~\cite{khodadadian,woo}, and finite-time guarantees under environment heterogeneity for on-policy temporal-difference learning~\cite{fedsarsa}, softmax policy gradient~\cite{labbi25}, tabular policy averaging~\cite{jin22}, asynchronous policy gradient~\cite{afedpg}, and bias-corrected federated policy gradient~\cite{fastfedpg}; federated natural policy gradient is analyzed in~\cite{fednpg}, which lists personalization as an open direction.
\textbf{Personalized FRL.} Closest to us are analyses built on different personalization mechanisms: global--local mixing with zeroth- or first-order gradients~\cite{zoapfpg,pfopg}, shared representations with personalized heads~\cite{xiong25}, and per-task policies constrained to a meta-policy~\cite{pmetarl}. The last aggregates parameters without sharing trajectories, but is posed as centralized multi-task meta-RL. None covers the MAML-initialization mechanism we study (Table~\ref{tab:related}). Methodologically, we combine the local-update template of~\cite{perfedavg} with the policy-gradient regularity toolbox of~\cite{sgmrl,shen,papini}; the resulting proof is self-contained and, in places, simpler than either ingredient, precisely because the boundedness the supervised theory must assume holds automatically here.

\section{Problem setup}\label{sec:setup}

\emph{All proofs are given in \supp.}

\subsection{Per-agent MDPs and policy gradients}

Consider a network of $n$ agents coordinated by a central server. Agent $i\in[n]:=\{1,\dots,n\}$ interacts with a finite\footnote{Finiteness of $\mathcal S_i,\mathcal A_i$ makes every interchange of differentiation and expectation elementary, all expectations being finite sums; the results extend to general spaces under standard domination conditions.} MDP $\mathcal M_i=(\mathcal S_i,\mathcal A_i,\mu_i,P_i,r_i)$ with initial distribution $\mu_i$, transition kernel $P_i$, reward $r_i:\mathcal S_i\times\mathcal A_i\to\R$, fixed horizon $H\in\mathbb N_0=\{0,1,2,\dots\}$, and discount $\gamma\in[0,1)$. Agents may differ in \emph{any} of $\mu_i,P_i,r_i$; nothing below requires their MDPs to be close. All agents share a differentiable policy class $\pi(\cdot\,|\,\cdot\,;\theta)$, $\theta\in\R^d$; this shared interface is what makes a common initialization meaningful. A trajectory of $\mathcal M_i$ is $\xi=(s_0,a_0,\dots,s_H,a_H)$ with law
\begin{equation}\label{eq:trajlaw}
q_i(\xi;\theta):=\mu_i(s_0)\prod_{h=0}^{H}\pi(a_h|s_h;\theta)\!\!\prod_{h=0}^{H-1}\!\!P_i(s_{h+1}|s_h,a_h).
\end{equation}
Agent $i$'s expected return is $J_i(\theta):=\E_{\xi\sim q_i(\cdot;\theta)}[R_i(\xi)]$, with $R_i(\xi):=R_i^0(\xi)$ and $R_i^h(\xi):=\sum_{t=h}^{H}\gamma^t r_i(s_t,a_t)$. By the policy gradient theorem (finite-horizon, reward-to-go form; see~\cite{sutton,sgmrl}),
\begin{equation}\label{eq:pg}
\begin{gathered}
\nabla J_i(\theta)=\E_{\xi\sim q_i(\cdot;\theta)}\big[g_i(\xi;\theta)\big],\\
g_i(\xi;\theta):=\sum_{h=0}^{H}\nabla_\theta\log\pi(a_h|s_h;\theta)\,R_i^h(\xi),
\end{gathered}
\end{equation}
and, with $\nu_i(\xi;\theta):=\sum_{h=0}^{H}\log\pi(a_h|s_h;\theta)\,R_i^h(\xi)$ (so that $g_i=\nabla\nu_i$), the Hessian admits the representation~\cite[Eq.~(8)]{sgmrl},~\cite{shen}
\begin{equation}\label{eq:hess}
\begin{gathered}
\nabla^2 J_i(\theta)=\E_{\xi\sim q_i(\cdot;\theta)}\big[u_i(\xi;\theta)\big],\\
u_i(\xi;\theta):=\nabla\nu_i(\xi;\theta)\,\nabla_\theta\log q_i(\xi;\theta)^{\!\top}+\nabla^2\nu_i(\xi;\theta).
\end{gathered}
\end{equation}
Note from~\eqref{eq:trajlaw} that $\nabla_\theta\log q_i(\xi;\theta)=\sum_{h=0}^H\nabla_\theta\log\pi(a_h|s_h;\theta)$, since $\mu_i,P_i$ do not depend on $\theta$. Equation~\eqref{eq:hess} is the source of the curvature cost we analyze: it is a \emph{score-function} Hessian, so its Monte Carlo estimate involves an outer product of sums of $H+1$ score terms.

\subsection{Assumptions}

Throughout, $\norm{\cdot}$ is the Euclidean norm for vectors and the spectral norm for matrices; $\norm{\cdot}_F$ is the Frobenius norm (recall $\norm{A}\le\norm{A}_F$).

\begin{assumption}[Bounded rewards]\label{as:reward}
$|r_i(s,a)|\le \Rmax$ for all $i\in[n]$, $s\in\mathcal S_i$, $a\in\mathcal A_i$.
\end{assumption}

\begin{assumption}[Policy regularity]\label{as:policy}
$\pi(a|s;\theta)>0$ for all $(s,a,\theta)$, $\theta\mapsto\log\pi(a|s;\theta)$ is twice continuously differentiable, and for all $(s,a)$ and all $\theta,\theta'\in\R^d$:
\begin{enumerate}\itemsep0pt\parsep0pt
\item[(a)] $\norm{\nabla_\theta\log\pi(a|s;\theta)}\le \gmax$;
\item[(b)] $\norm{\nabla^2_\theta\log\pi(a|s;\theta)}_F\le M$;
\item[(c)] $\norm{\nabla^2_\theta\log\pi(a|s;\theta)-\nabla^2_\theta\log\pi(a|s;\theta')}_F\le \Lambda\norm{\theta-\theta'}$.
\end{enumerate}
\end{assumption}

Assumption~\ref{as:policy} is the standard regularity condition in the policy-gradient literature~\cite{papini,shen,sgmrl}. It is a condition on the \emph{policy class alone}: it constrains neither the number of agents nor how different their MDPs are, thus implying that the resulting constants are uniform over the population. Parts (b)--(c) are stated in Frobenius norm to obtain dimension-free constants; replacing $\norm{\cdot}_F$ by $\norm{\cdot}$ costs at most a $\sqrt d$ factor in the variance constants. Assumption~\ref{as:policy}(c) is used only for the third-order constant $\rJ$ (Lemma~\ref{lem:pg-bounds}(iv)), hence only for meta-smoothness and the sharpened bias bound.

\begin{example}[Log-linear policies]\label{ex:softmax}
For $\pi(a|s;\theta)\propto\exp(\theta^\top\phi(s,a))$ with $\norm{\phi(s,a)}\le B$, we have $\nabla_\theta\log\pi=\phi(s,a)-\E_{a'\sim\pi}\phi(s,a')$ and $\nabla^2_\theta\log\pi=-\mathrm{Cov}_{a'\sim\pi}(\phi(s,a'))$, so Assumption~\ref{as:policy} holds \emph{globally} with $\gmax\le2B$, $M\le4B^2$, $\Lambda\le8B^3$. The tabular softmax policy used in Section~\ref{sec:experiments} is the one-hot case $B=1$, so the constants there are $\gmax\le2$, $M\le4$, $\Lambda\le8$. Gaussian policies with bounded action set and bounded mean parametrization satisfy the assumption similarly. For deep policies the bounds hold on bounded parameter sets rather than all of $\R^d$; this is the standard caveat in this literature, and it is why we treat our neural experiments as tests of the qualitative predictions rather than of the constants.
\end{example}

Throughout this paper, we denote
\begin{equation}\label{eq:constants}
\begin{gathered}
\bar R:=\tfrac{\Rmax}{1-\gamma},\quad
\GJ:=\tfrac{\gmax \Rmax}{(1-\gamma)^2},\quad
\LJ:=\tfrac{\Rmax((H+1)\gmax^2+M)}{(1-\gamma)^2},\\
\bar M:=\tfrac{M\Rmax}{(1-\gamma)^2},\qquad
\bar\Lambda:=\tfrac{\Lambda\Rmax}{(1-\gamma)^2},\\
\rJ:=(H+1)\gmax(\LJ+\bar M)+(H+1)M\GJ+\bar\Lambda .
\end{gathered}
\end{equation}

\subsection{Regularity of the per-agent objectives}
Under Assumptions~\ref{as:reward} and~\ref{as:policy}, Lemma~\ref{lem:pg-bounds} establishes uniform bounds on the rewards, on each agent's expected return, and on its gradient and Hessian. Crucially these bounds hold \emph{per trajectory} rather than merely in expectation, which is what supplies the regularity needed for the Monte Carlo estimators downstream without any additional heterogeneity assumption.

\begin{lemma}[Uniform policy-gradient regularity]\label{lem:pg-bounds}
Let Assumptions~\ref{as:reward} and~\ref{as:policy}(a)--(b) hold. Then for every $i\in[n]$, $\theta\in\R^d$ and every trajectory $\xi$:
(i) $|J_i(\theta)|\le\bar R$ and $|R_i^h(\xi)|\le \Rmax\gamma^h/(1-\gamma)$ for $0\le h\le H$;
(ii) $\norm{g_i(\xi;\theta)}\le\GJ$, hence $\norm{\nabla J_i(\theta)}\le\GJ$;
(iii) $\norm{u_i(\xi;\theta)}_F\le\LJ$, hence $\norm{\nabla^2 J_i(\theta)}\le\LJ$ and $\nabla J_i$ is $\LJ$-Lipschitz.
If additionally Assumption~\ref{as:policy}(c) holds, then
(iv) $\nabla^2 J_i$ is $\rJ$-Lipschitz: $\norm{\nabla^2J_i(\theta)-\nabla^2J_i(\theta')}\le\rJ\norm{\theta-\theta'}$ for all $\theta,\theta'$.
\end{lemma}

\begin{proof}[Proof sketch]
(i) is a geometric-series bound on the discounted reward-to-go. (ii) follows by bounding each of the $H+1$ score terms in~\eqref{eq:pg} by $\gmax$ and summing $\gamma^h$. (iii) splits $u_i$ in~\eqref{eq:hess} into a rank-one term, whose Frobenius norm is $\norm{g_i}\norm{\nabla\log q_i}\le\GJ(H+1)\gmax$, and $\nabla^2\nu_i$, bounded by $\bar M$. (iv) bounds the \emph{difference} $\nabla^2J_i(\theta_1)-\nabla^2J_i(\theta_0)$ of the finite sums $\sum_\xi q_i(\xi;\theta)u_i(\xi;\theta)$ directly, rather than differentiating in $\theta$, which would require a third derivative of $\log\pi$ that Assumption~\ref{as:policy} does not supply; inserting a mixed term splits it into a \emph{distribution-shift} part (the weights $q_i$ move) and a \emph{curvature-change} part (the summands $u_i$ move), bounded using (iii) and Assumption~\ref{as:policy}(b)--(c) respectively. Full proof in \supp.
\end{proof}

\begin{remark}[Conditions that hold automatically in the policy-gradient setting]\label{rem:free}
Uniform boundedness of the per-agent gradients, and of any gradient-dissimilarity measure, is part of Lemma~\ref{lem:pg-bounds} in the personalized federated policy-gradient setting. This is in stark contrast to personalized FL, where both are additional assumptions. By Lemma~\ref{lem:pg-bounds}(ii) we have the explicit uniform bound$\sup_\theta\norm{\nabla J_i(\theta)}\le \GJ$. In supervised personalized FL this uniform boundedness is an \emph{assumption} (Assumption~2 of~\cite{perfedavg}) that has attracted criticism and follow-up work~\cite{aistatsmaml,persafl,jamali}; here it is a theorem. Likewise any gradient-dissimilarity quantity, such as
\[
\tfrac1n\textstyle\sum_i\big\lVert\nabla J_i(\theta)-\tfrac1n\sum_j\nabla J_j(\theta)\big\rVert^2\le \GJ^2,
\]
is automatically bounded, so no heterogeneity assumption is needed; the constant is $\GJ^2$ rather than the $4\GJ^2$ given by the triangle inequality, by the same variance identity that gives Lemma~\ref{lem:drift}(ii). Two caveats apply. First, these constants are worst case: they scale as $(1-\gamma)^{-2}$ and, for $\LJ$ and $\rJ$, linearly in $H$, so they are informative about \emph{structure} rather than tight. Second, the bound $\GJ^2$ does not \emph{exploit} task similarity; a genuinely small dissimilarity would sharpen constants but is not needed for the rate.
\end{remark}

\section{Objective and algorithm}\label{sec:algo}

\subsection{The meta-objective}
We formalize personalized FRL through a MAML-style meta-objective: agents collaborate to choose a shared initialization, which each then adapts with one local policy-gradient step. Fix an adaptation step size $\alpha>0$. Agent $i$'s \emph{adapted} (personalized) policy from initialization $\theta$ is $\theta_i^+(\theta):=\theta+\alpha\nabla J_i(\theta)$, i.e.\ one exact policy-gradient ascent step on its own MDP. The personalized FRL problem is
\begin{equation}\label{eq:meta}
\max_{\theta\in\R^d}\; F(\theta):=\frac1n\sum_{i=1}^n F_i(\theta),\qquad F_i(\theta):=J_i\big(\theta_i^{+}(\theta)\big).
\end{equation}
This is the exact RL analogue of the Per-FedAvg objective~\cite{perfedavg}: the ``deterministic-adaptation'' MAML objective, as opposed to the ``stochastic-adaptation'' objective of SG-MRL~\cite{sgmrl}, in which the expectation over the adaptation batch is part of the objective (Remark~\ref{rem:sgmrl}). 
We restrict attention to a single adaptation step throughout. This is the Per-FedAvg contract and the regime in which MAML is typically deployed, since adaptation happens on device under a small local budget, and it is what keeps the meta-gradient a single curvature-times-gradient product rather than a product of $\nu$ Jacobians, which in turn is what makes explicit constants and an exact characterization of the first-order bias possible. The extension to $\nu$ steps is mechanical, but the regularity constants compound as $(1+\alpha\LJ)^{\Theta(\nu)}$~\cite{jimulti,jamali}, so the bounds grow heavier without anything qualitative changing, and they remain $\mathcal O(1)$ in $\nu$ only when $\alpha=\Theta(1/(\nu\LJ))$.
Since $F$ is in general nonconcave, we target approximate first-order stationarity.

\begin{definition}[$\varepsilon$-FOSP]\label{def:fosp}
A random vector $\hat\theta$ is an $\varepsilon$-approximate first-order stationary point of~\eqref{eq:meta} if $\E\norm{\nabla F(\hat\theta)}^2\le\varepsilon$.
\end{definition}

By the chain rule (all derivatives exist by Lemma~\ref{lem:pg-bounds}),
\begin{equation}\label{eq:metagrad}
\nabla F_i(\theta)=\big(I+\alpha\nabla^2J_i(\theta)\big)\,\nabla J_i\big(\theta_i^+(\theta)\big).
\end{equation}

\begin{lemma}[Regularity of the meta-objective]\label{lem:meta-reg}
Let Assumptions~\ref{as:reward}--\ref{as:policy} hold and $\alpha>0$. Then for all $i$ and all $\theta$:
(i) $|F_i(\theta)|\le\bar R$; in particular $F^\star:=\sup_\theta F(\theta)\le\bar R<\infty$ and $\Delta:=F^\star-F(\theta^0)\le2\bar R$;
(ii) $\norm{\nabla F_i(\theta)}\le \GF:=(1+\alpha\LJ)\GJ$;
(iii) $\nabla F_i$ is $\LF$-Lipschitz with $\LF:=(1+\alpha\LJ)^2\LJ+\alpha\rJ\GJ$; hence so is $\nabla F$.
\end{lemma}

The constant $\LF$ has the same form as the Per-FedAvg smoothness constant, with the loss-function constants there replaced by their policy-gradient counterparts $\LJ,\GJ,\rJ$; for $\alpha\le1/\LJ$ it reduces to $\mathcal O(\LJ+\alpha\rJ\GJ)$. Item (i) has a further consequence: in RL the meta-objective is automatically bounded above, so the optimality gap in Theorem~\ref{thm:main} is finite without further assumption, and bounded by an explicit problem constant.

\begin{remark}[Stationarity certifies per-agent adapted policies]\label{rem:interp}
By~\eqref{eq:metagrad} and $\norm{\nabla^2J_i(\theta)}\le\LJ$, the symmetric matrix $I+\alpha\nabla^2J_i(\theta)$ has eigenvalues in $[1-\alpha\LJ,1+\alpha\LJ]$; for $\alpha\le 1/(2\LJ)$ these lie in $[\tfrac12,\tfrac32]$, so it is invertible with $\norm{(I+\alpha\nabla^2J_i(\theta))^{-1}}\le(1-\alpha\LJ)^{-1}\le 2$, whence $\norm{\nabla J_i(\theta_i^+(\theta))}\le 2\norm{\nabla F_i(\theta)}$. A small $\norm{\nabla F_i(\theta)}$ therefore certifies that agent $i$'s \emph{adapted} policy is nearly stationary for its own MDP: a per-agent statement, not merely a population one. Our guarantee controls the average $\nabla F=\frac1n\sum_i\nabla F_i$, exactly as in~\cite{perfedavg,sgmrl}.
\end{remark}

\begin{remark}[Relation to the SG-MRL objective]\label{rem:sgmrl}
We adapt with the \emph{deterministic} gradient step $\theta+\alpha\nabla J_i(\theta)$, following Per-FedAvg~\cite{perfedavg}. The debiased stochastic-adaptation objective of SG-MRL~\cite{sgmrl}, which makes the sampled inner gradient part of the objective, is compatible with our analysis and would replace the inner-loop bias term studied here by zero; we keep the deterministic-adaptation objective for directness and because it is the one whose bias we can characterize sharply (Proposition~\ref{prop:sharpbias}).
\end{remark}

\subsection{Stochastic meta-gradient estimator}
Agents cannot evaluate $\nabla F_i$ exactly and instead estimate it from sampled trajectories, following the structure of \eqref{eq:metagrad}. Given the current iterate $\theta$, agent $i$ draws three independent batches. An inner batch $\mathcal D_{\mathrm{in}}$ of $m_{\mathrm{in}}$ trajectories from $q_i(\cdot;\theta)$ yields the policy-gradient estimate $\hat g_{\mathrm{in}}:=\frac{1}{m_{\mathrm{in}}}\sum_{\xi\in\mathcal D_{\mathrm{in}}}g_i(\xi;\theta)$ and the stochastic adapted parameter $\tilde\theta:=\theta+\alpha\,\hat g_{\mathrm{in}}$. An independent curvature batch $\mathcal D_{\mathrm h}$ of $m_{\mathrm h}$ trajectories from $q_i(\cdot;\theta)$ yields the Hessian estimate $\hat H:=\frac{1}{m_{\mathrm h}}\sum_{\xi\in\mathcal D_{\mathrm h}}u_i(\xi;\theta)$. Finally, an outer batch $\mathcal D_{\mathrm{out}}$ of $m_{\mathrm{out}}$ trajectories drawn \emph{under the adapted policy} $q_i(\cdot;\tilde\theta)$ yields $\hat g_{\mathrm{out}}:=\frac{1}{m_{\mathrm{out}}}\sum_{\xi\in\mathcal D_{\mathrm{out}}}g_i(\xi;\tilde\theta)$. The estimator used by agent $i$ is
\begin{equation}\label{eq:estimator}
\widehat{\nabla F}_i(\theta):=\big(I+\alpha\hat H\big)\,\hat g_{\mathrm{out}} ,
\end{equation}
computed by \textsc{MetaGrad} in Algorithm~\ref{alg:main}. All batches are drawn fresh at every local step and are independent across agents and steps. Two design points deserve emphasis. First, $\mathcal D_{\mathrm h}$ must be independent of $\mathcal D_{\mathrm{in}}$ and $\mathcal D_{\mathrm{out}}$: this is what makes $\E[(I+\alpha\hat H)\hat g_{\mathrm{out}}]$ factor, leaving inner-loop sampling as the \emph{only} source of bias. Second, $\mathcal D_{\mathrm{out}}$ is sampled under $\tilde\theta$, not $\theta$; this on-policy resampling is the concrete form taken by the distribution shift noted in Section~\ref{sec:setup}, and it is why no importance weights appear. The \emph{first-order (FO) variant} skips the curvature batch ($\hat H\gets0$) and is analyzed in Section~\ref{sec:fo}.

\begin{lemma}[Estimator properties]\label{lem:estimator}
Let Assumptions~\ref{as:reward}--\ref{as:policy} hold and condition on the current iterate $\theta$ (measurable with respect to the past). Then~\eqref{eq:estimator} satisfies:
\begin{enumerate}\itemsep0pt
    \item[(i)] \emph{(a.s.\ boundedness)} $\norm{\widehat{\nabla F}_i(\theta)}\le(1+\alpha\LJ)\GJ=\GF$ almost surely;
    \item[(ii)] \emph{(bias)} $\big\lVert\E[\widehat{\nabla F}_i(\theta)]-\nabla F_i(\theta)\big\rVert\le \bhat:=\dfrac{(1+\alpha\LJ)\,\alpha\,\LJ\,\GJ}{\sqrt{m_{\mathrm{in}}}}$;
    \item[(iii)] \emph{(second moment)} $\E\lVert\widehat{\nabla F}_i(\theta)-\nabla F_i(\theta)\rVert^2\le\shat^2$, where
    \[
    \shat^2:=3\GJ^2\Big[\tfrac{\alpha^2\LJ^2}{m_{\mathrm h}}+\tfrac{(1+\alpha\LJ)^2}{m_{\mathrm{out}}}+\tfrac{(1+\alpha\LJ)^2\alpha^2\LJ^2}{m_{\mathrm{in}}}\Big].
    \]
\end{enumerate}
Moreover $\E\norm{\widehat{\nabla F}_i(\theta)-\E\widehat{\nabla F}_i(\theta)}^2\le\shat^2$.
\end{lemma}

Part (i) is the RL-specific ingredient on which the drift analysis rests: because the meta-gradient \emph{estimate} (not just its mean) is bounded almost surely, client drift can be controlled deterministically (Lemma~\ref{lem:drift}) with no recourse to a heterogeneity bound. The three terms of $\shat^2$ are, in order, the curvature, outer-gradient and inner-gradient contributions; note the first is the only one an agent cannot reduce by reusing on-policy data it already collects.

The bias bound (ii) is the crude one obtained from Lipschitz continuity of $\nabla J_i$. A second-order argument improves it by an entire $\sqrt{m_{\mathrm{in}}}$ factor, which in turn improves the trajectory complexity.

\begin{proposition}[Sharp inner-loop bias; improved trajectory complexity]\label{prop:sharpbias}
Let $\varsigma_i^2(\theta):=\E\norm{g_i(\xi;\theta)-\nabla J_i(\theta)}^2\le\GJ^2$ denote the per-trajectory policy-gradient variance. Then the bias of~\eqref{eq:estimator} obeys
\[
\begin{aligned}
\big\lVert\E\big[\widehat{\nabla F}_i(\theta)\big]-\nabla F_i(\theta)\big\rVert
&\ \le\ \bhat_\star:=\frac{(1+\alpha\LJ)\,\alpha^2\,\rJ\,\varsigma_i^2(\theta)}{2\,m_{\mathrm{in}}}\\
&\ \le\ \frac{(1+\alpha\LJ)\,\alpha^2\,\rJ\,\GJ^2}{2\,m_{\mathrm{in}}}
\ =\ \mathcal O\!\big(m_{\mathrm{in}}^{-1}\big).
\end{aligned}
\]
The improvement is from $m_{\mathrm{in}}^{-1/2}$ to $m_{\mathrm{in}}^{-1}$, a full $\sqrt{m_{\mathrm{in}}}$ factor in the worst-case bias; its consequence for the bias budget of Corollary~\ref{cor:rates}, and hence for sample complexity, is recorded in Remark~\ref{rem:sharpcx}
\end{proposition}

The mechanism is as follows. Since $\E[\tilde\theta-\theta_i^+(\theta)]=0$, the \emph{first}-order Taylor term of $\nabla J_i$ about $\theta_i^+(\theta)$ vanishes in expectation, so the bias is governed by the second-order remainder $\tfrac{\rJ}{2}\E\norm{\tilde\theta-\theta_i^+}^2=\tfrac{\rJ}{2}\alpha^2\varsigma_i^2/m_{\mathrm{in}}$. The crude bound instead passes through $\E\norm{\tilde\theta-\theta_i^+}$, which decays only as $m_{\mathrm{in}}^{-1/2}$. This requires Assumption~\ref{as:policy}(c).

\subsection{The algorithm}

The overall method is Algorithm~\ref{alg:main}; Figure~\ref{fig:method} in \suppalg\ depicts one communication round. In each round the server broadcasts $\theta^k$; every agent initializes $\theta_i^{k,0}:=\theta^k$, performs $\tau$ local ascent steps $\theta_i^{k,t+1}:=\theta_i^{k,t}+\beta\,\widehat{\nabla F}_i(\theta_i^{k,t})$ with fresh independent batches, and returns $\theta_i^{k,\tau}$; the server averages, $\theta^{k+1}=\frac1n\sum_i\theta_i^{k,\tau}$. We analyze full participation throughout.

\paragraph{Communication and sampling cost.} 

A round costs one downlink and one uplink of $d$ floats per agent, so total communication is $2dKn$ floats, independent of $\tau$, of the horizon $H$, and of the batch sizes. Local sampling, by contrast, costs $\tau(m_{\mathrm{in}}+m_{\mathrm h}+m_{\mathrm{out}})H$ environment steps per agent per round. Raising $\tau$ therefore trades local interaction for communication, and Theorem~\ref{thm:main} quantifies exactly how far that trade can be pushed before client drift dominates. Raw trajectories never leave the agent that collected them.

\begin{algorithm}[t]
\caption{\textsc{Per-FedAvg-PG}}
\label{alg:main}
\small
\begin{algorithmic}[1]
\Require initialization $\theta^0$; rounds $K$; local steps $\tau$; step sizes $\alpha,\beta$; batches $m_{\mathrm{in}},m_{\mathrm h},m_{\mathrm{out}}$
\For{$k=0,\dots,K-1$}
    \State server broadcasts $\theta^k$ to all $n$ agents
    \For{\textbf{each} agent $i\in[n]$ \textbf{in parallel}}
        \State $\theta_i^{k,0}\gets\theta^k$
        \For{$t=0,\dots,\tau-1$}
            \State $\theta_i^{k,t+1}\gets\theta_i^{k,t}+\beta\cdot\textsc{MetaGrad}_i(\theta_i^{k,t})$
        \EndFor
        \State agent sends $\theta_i^{k,\tau}$ to the server
    \EndFor
    \State server aggregates $\theta^{k+1}\gets\frac1n\sum_{i=1}^n\theta_i^{k,\tau}$
\EndFor
\State \Return $\theta^K$ \Comment{deployment: agent $i$ personalizes as $\theta^K+\alpha\hat g_{\mathrm{in}}$}
\Statex\hrulefill
\Function{MetaGrad$_i$}{$\theta$}\Comment{estimator Eq.~\eqref{eq:estimator}; fresh batches}
\State sample $\mathcal D_{\mathrm{in}}\!\sim\! q_i(\cdot;\theta)^{\otimes m_{\mathrm{in}}}$; $\ \hat g_{\mathrm{in}}\!\gets\!\frac{1}{m_{\mathrm{in}}}\!\sum_{\xi\in\mathcal D_{\mathrm{in}}}\! g_i(\xi;\theta)$ \Comment{policy gradient, Eq.~\eqref{eq:pg}}
\State $\tilde\theta\gets\theta+\alpha\,\hat g_{\mathrm{in}}$ \Comment{one-step adaptation}
\State sample $\mathcal D_{\mathrm h}\!\sim\! q_i(\cdot;\theta)^{\otimes m_{\mathrm h}}$ indep.\ of $\mathcal D_{\mathrm{in}}$; $\ \hat H\!\gets\!\frac{1}{m_{\mathrm h}}\!\sum_{\xi\in\mathcal D_{\mathrm h}}\! u_i(\xi;\theta)$
\Statex \hfill\Comment{Eq.~\eqref{eq:hess}; \emph{FO variant:} $\hat H\gets0$}
\State sample $\mathcal D_{\mathrm{out}}\!\sim\! q_i(\cdot;\tilde\theta)^{\otimes m_{\mathrm{out}}}$; $\ \hat g_{\mathrm{out}}\!\gets\!\frac{1}{m_{\mathrm{out}}}\!\sum_{\xi\in\mathcal D_{\mathrm{out}}}\! g_i(\xi;\tilde\theta)$
\Statex \hfill\Comment{under the \emph{adapted} policy}
\State \Return $(I+\alpha\hat H)\,\hat g_{\mathrm{out}}$
\EndFunction
\end{algorithmic}
\end{algorithm}

\section{Convergence analysis}\label{sec:analysis}

Define the \emph{virtual averaged sequence} $\bar\theta^{k,t}:=\frac1n\sum_{i=1}^n\theta_i^{k,t}$ for $0\le t\le\tau$. Since $\bar\theta^{k,0}=\theta^k$ and $\bar\theta^{k,\tau}=\theta^{k+1}=\bar\theta^{k+1,0}$, the sequence $\{\bar\theta^{k,t}\}$ is a single consistent trajectory, even though the server materializes only its round boundaries. Let $\mathcal F^{k,t}$ be the $\sigma$-algebra generated by all randomness up to and including the computation of $\{\theta_i^{k,t}\}_{i\in[n]}$, and $\E_{k,t}[\cdot]:=\E[\cdot\,|\,\mathcal F^{k,t}]$. The argument has three steps: bound client drift within a round, establish a one-step ascent inequality for $F$ along $\bar\theta^{k,t}$, and telescope.

\subsection{Client drift}

Because each meta-gradient estimate is bounded \emph{almost surely} (Lemma~\ref{lem:estimator}(i)), drift grows only linearly in the local-step index, and the bound is pathwise rather than in expectation.

\begin{lemma}[Deterministic drift bound]\label{lem:drift}
Surely (that is, for every realization of the sampled batches), for all $k$ and all
$0\le t\le\tau$:
\begin{enumerate}\itemsep0pt
\item[(i)] every agent stays within $\norm{\theta_i^{k,t}-\theta^k}\le\beta t\,\GF$ of the round's
      starting point, for each $i\in[n]$;
\item[(ii)] the averaged squared client drift obeys
      $\displaystyle\frac1n\sum_{i=1}^n\norm{\theta_i^{k,t}-\bar\theta^{k,t}}^2\le\beta^2t^2\GF^2$.
\end{enumerate}
\end{lemma}

This is the point at which the present analysis diverges from the supervised one. In Per-FedAvg~\cite{perfedavg}, the analogous step needs both a bounded-gradient and a bounded-dissimilarity assumption, because only the \emph{mean} local gradient is controlled; here Lemma~\ref{lem:estimator}(i) controls every realization.

\subsection{Per-step ascent and main theorem}
Using the drift bound, we lower-bound the expected value of the meta-objective after a single averaged local step. This is the central estimate of the analysis: it exposes an ascent term proportional to $\norm{\nabla F(\bar\theta^{k,t})}^2$ together with three error terms, arising from client drift, estimator bias, and estimator variance.

\begin{lemma}[One-step progress]\label{lem:onestep}
Let Assumptions~\ref{as:reward}--\ref{as:policy} hold and $\beta\le\frac{1}{6\LF}$. Then for every round $k$ and local step $0\le t\le\tau-1$,
\[
\begin{aligned}
\E_{k,t}\,F\big(\bar\theta^{k,t+1}\big)\;\ge\;&F\big(\bar\theta^{k,t}\big)
+\frac{\beta}{2}\norm{\nabla F(\bar\theta^{k,t})}^2\\
&-\frac{9\beta}{4}\Big(\LF^2\beta^2t^2\GF^2+\bhat^2\Big)
-\frac{\LF\beta^2\shat^2}{2n}.
\end{aligned}
\]
\end{lemma}

The three error terms are exactly the three sources of error in local training: drift ($t^2$, controlled by $\beta\tau$), estimator bias ($\bhat^2$, controlled only by $m_{\mathrm{in}}$), and estimator variance ($\shat^2/n$, the only term helped by population size).

\begin{theorem}[Convergence of \textsc{Per-FedAvg-PG}]\label{thm:main}
Let Assumptions~\ref{as:reward}--\ref{as:policy} hold, let $\alpha>0$, $\beta\le\frac1{6\LF}$, and let $\Delta:=F^\star-F(\theta^0)\ (\le 2\bar R)$. Then after $K$ rounds with $\tau$ local steps,
\begin{equation}\label{eq:mainbound}
\begin{aligned}
\frac{1}{K\tau}\sum_{k=0}^{K-1}\sum_{t=0}^{\tau-1}\E\norm{\nabla F(\bar\theta^{k,t})}^2
\;\le\;&\frac{2\Delta}{\beta K\tau}
+\tfrac32\,\LF^2\GF^2\,\beta^2\tau^2\\
&+\frac92\,\bhat^2
+\frac{\LF\,\beta\,\shat^2}{n}
,
\end{aligned}
\end{equation}
with $\GF,\LF$ as in Lemma~\ref{lem:meta-reg} and $\bhat,\shat^2$ as in Lemma~\ref{lem:estimator}. The same bound holds for $\E\norm{\nabla F(\hat\theta)}^2$ when $\hat\theta$ is drawn uniformly from $\{\bar\theta^{k,t}\}_{k<K,t<\tau}$.
\end{theorem}

\begin{proof}[Proof sketch]
$\LF$-smoothness of $F$ gives the ascent inequality $F(\bar\theta+\beta P)\ge F(\bar\theta)+\beta\inner{\nabla F(\bar\theta)}{P}-\frac{\LF\beta^2}{2}\norm{P}^2$ for the averaged direction $P=\frac1n\sum_i\widehat{\nabla F}_i(\theta_i^{k,t})$. Taking $\E_{k,t}$, the conditional mean of $P$ is $\nabla F(\bar\theta^{k,t})+e+\bar b$, where $e$ collects the drift error, bounded via Lemma~\ref{lem:meta-reg}(iii) and Lemma~\ref{lem:drift}, and $\norm{\bar b}\le\bhat$. Young's inequality on $\inner{\nabla F}{e+\bar b}$ and independence of the estimators across agents (giving $\shat^2/n$ for the conditional variance of $P$) yield Lemma~\ref{lem:onestep}. Summing over $t<\tau$ and $k<K$, the potential terms telescope to $\Delta$ because $\bar\theta^{k,\tau}=\bar\theta^{k+1,0}$, and $\sum_{k,t}t^2\le K\tau^3/3$; multiplying by $2/(\beta K\tau)$ gives~\eqref{eq:mainbound}. Full proof in \supp.
\end{proof}

\begin{corollary}[Rates]\label{cor:rates}
Fix $\varepsilon>0$ and choose parameters so that
\[
\begin{gathered}
\text{(i) ceiling: }\beta\le\tfrac1{6\LF},\qquad
\text{(ii) drift: }\beta\tau\le\tfrac{\sqrt\varepsilon}{\sqrt6\,\LF\GF}, \\
\text{(iii) bias: }m_{\mathrm{in}}\ge\tfrac{18(1+\alpha\LJ)^2\alpha^2\LJ^2\GJ^2}{\varepsilon},\qquad
\text{(iv) variance: }\beta\le\tfrac{n\varepsilon}{4\LF\shat^2},
\end{gathered}
\]
and run $K\ge 8\Delta/(\beta\tau\varepsilon)$ rounds. Then $\frac{1}{K\tau}\sum_{k,t}\E\norm{\nabla F(\bar\theta^{k,t})}^2\le\varepsilon$. In particular, taking $m_{\mathrm{out}},m_{\mathrm h}=\Theta(1)$, $m_{\mathrm{in}}=\Theta(\varepsilon^{-1})$, $\tau=\lceil\varepsilon^{-1/2}\rceil$ and $\beta=\frac{\sqrt\varepsilon}{\sqrt6\LF\GF\tau}=\Theta(\varepsilon/(\LF\GF))$ (which meets (ii) with equality and, provided $\varepsilon\le\min\{1,\GF/\sqrt6\}$ and $\shat^2\le\frac{\sqrt6}{4}n\GF$, also (i) and (iv)), the algorithm reaches an $\varepsilon$-FOSP with $K=\mathcal O\big(\Delta\LF\GF\varepsilon^{-3/2}\big)$ rounds, $\tau=\Theta(\varepsilon^{-1/2})$ local steps,
and per-agent trajectory complexity $K\tau(m_{\mathrm{in}}+m_{\mathrm h}+m_{\mathrm{out}})=\mathcal O(\varepsilon^{-3})$, improved to $\mathcal O(\varepsilon^{-5/2})$ by Proposition~\ref{prop:sharpbias}.
\end{corollary}

\begin{remark}[Comparison with Per-FedAvg and SG-MRL]\label{rem:compare}
The round complexity $K=\mathcal O(\varepsilon^{-3/2})$ at $\tau=\Theta(\varepsilon^{-1/2})$ matches the guarantee known for Per-FedAvg in the supervised setting~\cite{perfedavg}, and the total step count $K\tau=\mathcal O(\varepsilon^{-2})$ matches the centralized meta-RL complexity of SG-MRL~\cite{sgmrl}: federation yields a $\tau$-fold reduction in communication at no cost in computation, up to the drift condition (ii). We do not claim a linear speedup in $n$, since only the variance term $\LF\beta\shat^2/n$ improves with population size; a $1/n$ gain is therefore available at $\tau=1$ (Section~\ref{sec:regimes}) but not at the $\tau=\Theta(\varepsilon^{-1/2})$ frontier.
\end{remark}

\begin{remark}[What binds the sample complexity, and what would improve it]\label{rem:sharpcx}
Proposition~\ref{prop:sharpbias} enters Corollary~\ref{cor:rates} through the bias condition (iii) alone. With $\bhat_\star=\Theta(m_{\mathrm{in}}^{-1})$ in place of $\Theta(m_{\mathrm{in}}^{-1/2})$ it is met at $m_{\mathrm{in}}=\Theta(\varepsilon^{-1/2})$ rather than $\Theta(\varepsilon^{-1})$; as (i), (ii) and the round count involve no batch size, $K$ and $\tau$ are unchanged and only the per-agent trajectory complexity improves, from $\mathcal O(\varepsilon^{-3})$ to $\mathcal O(\varepsilon^{-5/2})$. The variance condition (iv) contains $m_{\mathrm{in}}$ through $\shat^2$ but cannot bind, since $\shat^2$ is bounded by a problem constant for all batch sizes $\ge1$ (Lemma~\ref{lem:estimator}(iii)) and is absorbed into $\tau$ at fixed round complexity. Inner-loop \emph{bias}, not gradient variance, is therefore what binds, and this sets the direction for improvement: variance reduction such as SVRPG~\cite{papini} or Hessian-aided PG~\cite{shen} targets the \emph{outer} gradient and would leave $\mathcal O(\varepsilon^{-5/2})$ untouched, whereas a variance-reduced \emph{inner} gradient, or the debiased objective of~\cite{sgmrl}, acts on the $\varsigma_i^2$ feeding $\bhat_\star$. Reaching the single-agent-PG-optimal $\mathcal O(\varepsilon^{-2})$ this way is left to future work.
\end{remark}

\section{The Hessian-free variant}\label{sec:fo}

The exact meta-gradient \eqref{eq:metagrad} contains the inner-loop policy Hessian $\nabla^2 J_i(\theta)$, so the estimator \eqref{eq:estimator} must form a curvature estimate $\hat H$; its variance contributes the term $3\GJ^2\alpha^2\LJ^2/m_{\mathrm h}$ in Lemma~\ref{lem:estimator}(iii). Our experiments (Section~\ref{sec:experiments}) identify this term as the binding constraint: reliable exact training requires a Hessian batch orders of magnitude larger than the inner and outer batches. This motivates a variant that removes the curvature estimate entirely.

Define the first-order estimator $\widehat{\nabla F}^{\mathrm{FO}}_i(\theta):=\hat g_{\mathrm{out}}$, i.e.\ \eqref{eq:estimator} with the curvature factor $(I+\alpha\hat H)$ dropped; it needs no Hessian batch ($m_{\mathrm h}=0$) and lowers the per-step trajectory cost accordingly.

\begin{corollary}[Convergence of the Hessian-free variant]\label{cor:fo}
Under the assumptions of Theorem~\ref{thm:main}, \textsc{Per-FedAvg-PG} run with $\widehat{\nabla F}^{\mathrm{FO}}_i$ satisfies~\eqref{eq:mainbound} with $(\bhat,\shat^2)$ replaced by
\[
\bhat_{\mathrm{FO}}:=\alpha\LJ\GJ\Big(1+\tfrac1{\sqrt{m_{\mathrm{in}}}}\Big),\quad
\shat^2_{\mathrm{FO}}:=3\GJ^2\Big[\tfrac{1}{m_{\mathrm{out}}}+\tfrac{\alpha^2\LJ^2}{m_{\mathrm{in}}}+\alpha^2\LJ^2\Big].
\]
The bias term $\frac92\bhat_{\mathrm{FO}}^2=\Theta(\alpha^2\LJ^2\GJ^2)$ in~\eqref{eq:mainbound} does not vanish for any choice of $K,\tau,\beta$ or batch sizes, so the guarantee degrades from convergence to an $\varepsilon$-FOSP to convergence to a neighborhood of that radius; an $\varepsilon$-FOSP is guaranteed only in the small-adaptation regime $\alpha=\mathcal O(\sqrt{\varepsilon}/(\LJ\GJ))$.
\end{corollary}

A non-vanishing term in an upper bound need not correspond to a real obstruction, since it may reflect only the looseness of the analysis. The following proposition rules this out. The exact first-order dynamics has fixed points at which the meta-gradient is nonzero, and the residual there is of the order that Corollary~\ref{cor:fo} predicts, so the floor is a property of the method rather than of the bound.

\begin{proposition}[Residual meta-gradient at first-order fixed points]\label{prop:fofix}
Consider the exact ($m\to\infty$) first-order dynamics, in which the averaged ascent direction is $\frac1n\sum_i\nabla J_i(\theta_i^+(\theta))$. If $\theta^\star$ is a fixed point, i.e.\ $\frac1n\sum_i\nabla J_i(\theta_i^+(\theta^\star))=0$, then
\[
\nabla F(\theta^\star)=\frac{\alpha}{n}\sum_{i=1}^n\nabla^2J_i(\theta^\star)\,\nabla J_i\big(\theta_i^+(\theta^\star)\big),
\qquad
\norm{\nabla F(\theta^\star)}\le\alpha\LJ\GJ .
\]
The residual vanishes identically when all agents share the same curvature at $\theta^\star$, since the fixed-point condition then factors out of the sum; curvature heterogeneity across agents is therefore \emph{necessary} for the discrepancy, which makes it a genuinely multiagent effect: for $n=1$ the residual is always zero. Outside a measure-zero set of the admissible configurations it is nonzero, and \supp\ exhibits a concrete instance: two three-armed bandits sharing a scalar log-linear policy, for which $\norm{\nabla F(\theta^\star)}/\alpha\to0.157$ as $\alpha\to0$.
The first-order method therefore converges to points whose meta-gradient is $\Theta(\alpha)$ in general, matching the $\Theta(\alpha^2\LJ^2\GJ^2)$ order of the squared-norm bound in Corollary~\ref{cor:fo}.
\end{proposition}

The floor is what dropping curvature costs, mirroring the FO-MAML phenomenon in the centralized setting~\cite{aistatsmaml,sgmrl,moreau}, where a first-order approximation is admissible only when the adaptation step is small. The exact variant removes the floor but must estimate $\hat H$. Choosing between them is the curvature bottleneck we study empirically in Section~\ref{sec:experiments}; in the moderate-adaptation regime (Section~\ref{sec:regimes}) the first-order variant is the default choice, and our tabular experiments show what it costs: it plateaus below the exact variant, consistent with the analysis.

\section{Step-size regimes}\label{sec:regimes}

The analysis exposes three distinct roles of $\alpha$ and $\beta$.

\paragraph{(R1) $\alpha\to0$: recovering non-personalized FRL}
\begin{proposition}\label{prop:alpha0}
Let $f(\theta):=\frac1n\sum_iJ_i(\theta)$. Under Assumptions~\ref{as:reward}--\ref{as:policy}, for all $\theta$,
$\norm{\nabla F(\theta)-\nabla f(\theta)}\le 2\alpha\LJ\GJ$ and $\big|F(\theta)-f(\theta)\big|\le\alpha\GJ^2$. Moreover at $\alpha=0$ the estimator~\eqref{eq:estimator} collapses to $\hat g_{\mathrm{out}}$, so \textsc{Per-FedAvg-PG} coincides with \textsc{FedAvg} applied to policy gradient, and \eqref{eq:mainbound} holds with $\GF=\GJ$, $\LF=\LJ$ and $\bhat=0$.
\end{proposition}
Objective, algorithm and guarantee therefore degenerate \emph{jointly} and continuously to non-personalized federated policy gradient~\cite{fednpg,fastfedpg}: the bias floor disappears, condition (iii) of Corollary~\ref{cor:rates} becomes vacuous, and trajectory complexity improves to $\mathcal O(\varepsilon^{-2})$ because the inner batch need not grow. Personalization is thus a continuous design parameter rather than a discrete choice of mechanism, and $\alpha$ is that parameter. As the two regimes below make explicit, its cost appears in the smoothness constant, in the required inner-batch size, and in the Hessian-free variant's bias.

\paragraph{(R2) Moderate adaptation, $\alpha\le\frac1{2\LJ}$.} Here $1+\alpha\LJ\le\frac32$, so $\GF\le\frac32\GJ$ and $\LF \le\tfrac94\LJ+\alpha\rJ\GJ\le\frac94\LJ+\frac{\rJ\GJ}{2\LJ}$; all curvature maps $I+\alpha\nabla^2J_i(\theta)$ are uniformly well-conditioned (singular values in $[\frac12,\frac32]$), and Remark~\ref{rem:interp} converts meta-stationarity into near-stationarity of each agent's \emph{adapted} policy. This is the recommended operating regime: personalization is nontrivial, all constants are $\Theta(1)$ multiples of the base constants, and both the exact and the stochastic variants converge to genuine $\varepsilon$-FOSPs. A secondary transition occurs at $\alpha^\dagger:=\LJ/(\rJ\GJ)$: below it meta-smoothness is dominated by the base smoothness ($\LF=\Theta(\LJ)$); above it the third-order term $\alpha\rJ\GJ$ dominates, and through $\beta\le\frac1{6\LF}$ the round complexity $K=\mathcal O(\Delta\LF\GF\varepsilon^{-3/2})$ grows \emph{linearly in $\alpha$}. Stronger personalization therefore incurs a proportional increase in communication.

\paragraph{(R3) Outer step size and local-step budget.} Bound~\eqref{eq:mainbound} is governed by three products: $\beta K\tau$ (optimization), $\beta\tau$ (drift), and $\beta/n$ (noise). Local work therefore helps only up to $\beta\tau\lesssim\sqrt\varepsilon/(\LF\GF)$, beyond which the drift term $\tfrac32\LF^2\GF^2\beta^2\tau^2$ dominates; balancing the two yields the $\tau=\Theta(\varepsilon^{-1/2})$, $K=\Theta(\varepsilon^{-3/2})$ frontier. At the other extreme $\tau=1$ removes drift entirely, leaving conditions (i) and (iv) of Corollary~\ref{cor:rates} binding and giving $K=\mathcal O(\LF\Delta/\varepsilon+\LF\Delta\shat^2/(n\varepsilon^2))$, a $1/n$ variance reduction obtained by communicating at every step; the two regimes cross where $\LF\beta\shat^2/n$ equals the drift term. The bias floor $\frac92\bhat^2=\Theta(\alpha^2\LJ^2\GJ^2/m_{\mathrm{in}})$ is invariant to $(\beta,\tau,K,n)$ and is governed by the inner batch alone, and since the required $m_{\mathrm{in}}$ shrinks quadratically as $\alpha$ decreases, weak personalization also demands few additional trajectories.

\section{Experimental evaluation}\label{sec:experiments}
We evaluate on two task families. Tabular gridworlds admit exact evaluation of $F$ and $\nabla F$ by dynamic programming, so claims about the objective, the estimator and the dependence on $\alpha$ are tested without sampling noise. A neural continuous-navigation policy (one hidden layer, $d=360$) shares parameters across states, which a tabular per-state policy does not, and so can encode a transferable skill; it is used to test the curvature bottleneck and few-shot personalization to unseen agents. Code and configurations are available at \url{https://github.com/AliBeikmohammadi/Per-FedAvg-PG}.

\paragraph{Setup.} The tabular family has $n=8$ agents on a $5\times5$ grid ($\gamma=0.9$, horizon $15$) differing only in goal cell. The neural family places goals on a shared arc of a circle (a common ``navigate toward the region'' structure); we train on six goals and hold out three unseen, in-distribution goals, evaluating returns by Monte Carlo. Adaptation uses a single inner step ($\nu=1$) throughout. 
All curves average $10$ seeds; shaded bands are $\pm1$ standard deviation, except in Figure~\ref{fig:neural} (left), where the outcome is bimodal and we report medians with per-seed points. Full environment specifications, hyperparameters and metrics are in \suppexp.

\begin{figure}[t]
\centering
\includegraphics[width=0.5\linewidth]{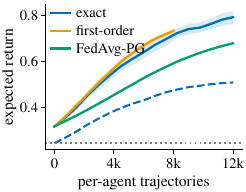}\hfill
\includegraphics[width=0.5\linewidth]{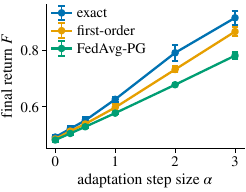}
\caption{\textbf{Tabular gridworld, exact evaluation.} \emph{Left:} post-adaptation value $F$ vs.\ per-agent trajectories; dashed is the exact variant \emph{before} adaptation, dotted a random policy. \emph{Right:} final $F$ vs.\ $\alpha$; at $\alpha=0$ all three coincide, as Proposition~\ref{prop:alpha0} requires.}
\Description{Two line charts. Left: expected return against per-agent trajectories from 0 to 12k, with three solid curves rising and flattening (exact highest, first-order just below, FedAvg-PG lowest) a dashed curve for the exact variant before adaptation well below them, and a flat dotted line near the bottom for a random policy; shaded bands show seed variability. Right: final return against adaptation step size from 0 to 3, where the three curves start together at the left edge and fan apart as the step size grows, exact highest and FedAvg-PG lowest.}
\label{fig:tab}
\end{figure}

\paragraph{Personalization within federation (Figure~\ref{fig:tab}, left).} All methods learn a common initialization, evaluated by the exact post-adaptation objective; non-personalized \textsc{FedAvg-PG} forms a single policy-gradient estimate from $m_{\mathrm{in}}+m_{\mathrm h}+m_{\mathrm{out}}$ trajectories, so it is given the same per-step sampling budget as the exact variant. The exact variant reaches $F=0.79$ and the first-order variant $0.73$, both above non-personalized \textsc{FedAvg-PG} ($0.68$) and far above a uniform-random policy ($0.24$). The gap between the initialization \emph{before} its adaptation step ($f=0.51$) and \emph{after} it ($0.79$) is precisely the value contributed by the single local policy-gradient step that Theorem~\ref{thm:main} analyzes. The exact-versus-first-order gap of $0.06$ is consistent with the bias floor of Corollary~\ref{cor:fo} and Proposition~\ref{prop:fofix}. At matched communication rounds (Figure~\ref{fig:e6} in \suppexp) the Hessian-free variant plateaus below the exact one rather than catching up. On the trajectory axis the two are not directly comparable, since the curvature batch makes each exact step cost $m_{\mathrm{in}}+m_{\mathrm h}+m_{\mathrm{out}}=30$ trajectories against $20$ for the first-order variant; at the matched budget of $8000$ per-agent trajectories the first-order variant is marginally ahead ($0.73$ against $0.71$), and the exact variant overtakes it only once allowed the extra samples its curvature estimate requires. The bias floor is therefore a statement about the limit of each method, not about sample efficiency at a fixed budget, which is the same trade-off the curvature bottleneck exhibits in the neural setting.

\paragraph{The adaptation step size $\alpha$ (Figure~\ref{fig:tab}, right).} Sweeping $\alpha$ confirms its role as a continuous personalization parameter. At $\alpha=0$ the three methods coincide ($0.49/0.49/0.48$): 
adaptation is inert and \textsc{Per-FedAvg-PG} reduces to federated policy optimization, exactly as Proposition~\ref{prop:alpha0} predicts. As $\alpha$ grows the personalized objective pulls ahead ($0.92/0.87/0.78$ at $\alpha=3$), the benefit increasing with adaptation strength. Because the tabular evaluation is exact, these numbers are properties of the objective, not of an estimator. Plotted against communication rounds rather than trajectories, the same comparison confirms that this ordering holds throughout training and not only at the horizon (Figure~\ref{fig:e6} in \suppexp).

\begin{figure}[t]
\centering
\includegraphics[width=0.5\linewidth]{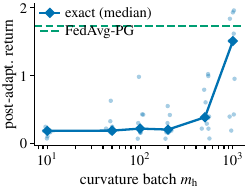}\hfill
\includegraphics[width=0.5\linewidth]{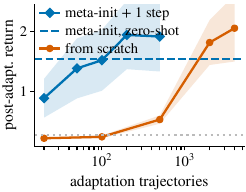}
\caption{\textbf{Neural navigation, Monte Carlo evaluation.} \emph{Left:} curvature bottleneck: return of the \emph{exact} variant vs.\ curvature batch $m_{\mathrm h}$ (per-seed points, median line); destabilized until $m_{\mathrm h}\approx1000$, consistent with the $\alpha^2\LJ^2/m_{\mathrm h}$ term of Lemma~\ref{lem:estimator}. \emph{Right:} few-shot return on \emph{held-out} agents vs.\ adaptation trajectories; success rates in Figure~\ref{fig:success} of \suppexp.}
\Description{Two charts. Left: post-adaptation return against curvature batch on a logarithmic axis from 10 to 1000, with scattered per-seed points and a median line that stays flat and low across the smaller batches and turns sharply upward only at the largest, still ending below a horizontal dashed line marking the curvature-free FedAvg-PG reference. Right: post-adaptation return against adaptation trajectories on a logarithmic axis, with a solid meta-initialization curve dipping at the smallest batch then rising above a horizontal dashed line marking its zero-shot level, and a from-scratch curve far below it that only climbs at the largest budgets.}
\label{fig:neural}
\end{figure}

\paragraph{The curvature-estimation bottleneck (Figure~\ref{fig:neural}, left).} Varying $m_{\mathrm h}$ with all else fixed, the exact variant is destabilized whenever the curvature batch is modest: the median post-adaptation return stays near $0.2$ for all $m_{\mathrm h}\le200$, with no seed reaching the \textsc{FedAvg-PG} level. It recovers only at a large curvature batch: median $0.39$ at $m_{\mathrm h}=500$ and $1.51$ at $m_{\mathrm h}=1000$, where the seeds separate cleanly into $6$ that train ($\ge1.5$, of which $3$ match the curvature-free reference of $1.73$) and $4$ that do not ($\le0.95$), so the outcome is bimodal and we report medians with per-seed points. Reliable exact training thus demands a Hessian batch roughly two orders of magnitude larger than the inner and outer batches ($m_{\mathrm h}\approx1000$ against $m_{\mathrm{in}}=m_{\mathrm{out}}=10$). This is the empirical case for the first-order variant: it removes the quantity that must be estimated, and incurs in exchange the bias characterized in Corollary~\ref{cor:fo} and Proposition~\ref{prop:fofix}.

\paragraph{Few-shot personalization to unseen agents (Figure~\ref{fig:neural}, right).} Our main experiment tests the deployment claim directly: adapting the learned initialization to an agent that never participated in training. We meta-train the \emph{first-order} variant (the variant that trains reliably at these batch sizes) on the six training goals and measure post-adaptation performance on the three held-out goals as a function of the adaptation budget, against \emph{from-scratch} training of each held-out agent (independent learning, no shared initialization) and its large-budget \emph{specialist} limit. The meta-initialization already generalizes \emph{zero-shot}, with return $1.53$ versus $0.27$ for a random policy. Because the analyzed objective adapts with the \emph{deterministic} gradient, we adapt with one policy-gradient step and vary its batch: once the step is estimated from enough trajectories to approximate that gradient it improves the initialization further (return $1.94$ at a $200$-trajectory step), whereas a $20$-trajectory step is too noisy to help, a finite-sample effect the deterministic-adaptation objective does not model, and a concrete instance of the inner-loop bias of Lemma~\ref{lem:estimator}(ii). At every matched budget the meta-initialization dominates from-scratch training by a wide margin (return $1.52$ vs.\ $0.25$ at $100$ trajectories), and from-scratch needs roughly $2000$ trajectories (about an order of magnitude more) to reach what the meta-initialization attains zero-shot. 
In interpretable terms (Figure~\ref{fig:success} in \suppexp) the meta-initialization reaches the goal in $75\%$ of held-out episodes zero-shot and up to $86\%$ after one adequately estimated step, against $15$--$37\%$ for from-scratch across the same range of budgets (the large-budget specialist reaches ${\sim}90\%$); at a $100$-trajectory budget the gap is $80\%$ versus $18\%$.
Figure~\ref{fig:e7} in \suppexp\ shows what these rates correspond to behaviorally, at the level of individual rollouts on an unseen agent.

\paragraph{Training cost.} The trajectory counts above are \emph{deployment-time} adaptation samples. The shared initialization itself consumed about $15{,}000$ trajectories per training agent ($90{,}000$ in total), a one-time cost amortized over every agent that later adapts from it. 
The comparison therefore concerns the marginal cost of onboarding a new agent, not total sample cost, and is favorable only when enough agents adapt for the amortization to pay.

These experiments isolate the mechanism rather than establish scale, and \suppexp\ states their scope and limitations in full.

\section{Conclusion}\label{sec:conclusion}

We formulated personalized FRL with a MAML-style initialization that each agent adapts by a single local policy-gradient step, and showed that \textsc{Per-FedAvg-PG} reaches an $\varepsilon$-first-order stationary point of the personalized meta-objective in $\mathcal O(\varepsilon^{-3/2})$ rounds with $\tau=\Theta(\varepsilon^{-1/2})$ local steps. The bounded-gradient and bounded-heterogeneity conditions the supervised theory imposes hold automatically in the policy-gradient setting, with explicit constants, which makes the federated analysis tractable. Two findings shape the practical picture. The exact meta-gradient attains the fast rate, but estimating its inner-loop Hessian is the binding constraint, which motivates the first-order variant; we bound its bias and show that its fixed points carry a residual of order $\alpha$. The step size $\alpha$ acts as a continuous personalization parameter whose magnitude carries a corresponding cost in smoothness, bias, and communication.

Three caveats bound these conclusions. The guarantee is first-order stationarity of a nonconcave objective rather than global optimality; the automatic constants are uniform worst-case quantities, establishing that no extra assumption is needed rather than that the rates are tight; and the experiments isolate the mechanism at small scale rather than demonstrating it at scale. Client sampling with a control-variate correction, the debiased stochastic-adaptation objective of SG-MRL~\cite{sgmrl}, variance-reduced inner estimators, and a matched-budget comparison of personalization mechanisms are the natural extensions, and \supplim\ treats the limitations and the extensions in full.



\ifappendix\else\balance\fi
\bibliographystyle{plain}
\bibliography{refs}

\ifappendix
\clearpage
\onecolumn
\appendix

\section{Algorithm and notation}\label{app:alg}

\subsection{One communication round}\label{app:round}
Figure~\ref{fig:method} depicts the structure of a single round of Algorithm~\ref{alg:main}. The server holds only the shared initialization $\theta^k$ and never sees a trajectory: it broadcasts $\theta^k$, each agent runs $\tau$ local meta-policy-gradient steps against its own MDP using freshly sampled batches, and the server averages the returned parameters. The averaged sequence $\bar\theta^{k,t}$ analyzed in Section~\ref{sec:analysis} is materialized only at round boundaries, where it coincides with $\theta^k$; between boundaries it is a virtual object, since the agents are apart.

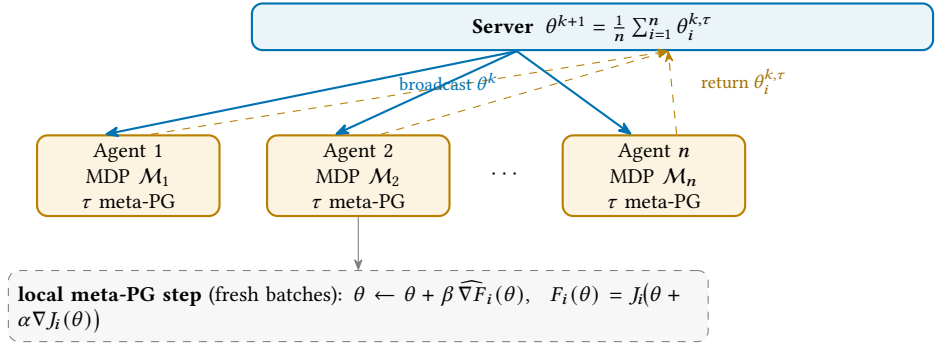
\begin{figure}[ht]\centering
\begin{tikzpicture}[font=\footnotesize,>={Stealth[round]},
  server/.style={draw=srv,thick,rounded corners,fill=srv!8,align=center,inner sep=3.5pt,minimum width=90mm},
  agent/.style={draw=agt!80!black,thick,rounded corners,fill=agt!12,align=center,inner sep=3pt,text width=22mm},
  note/.style={draw=ggray,dashed,rounded corners,fill=ggray!6,align=left,inner sep=3.5pt,text width=90mm}]
  \node[server](S){\textbf{Server}\;\;$\theta^{k+1}=\tfrac1n\sum_{i=1}^n\theta_i^{k,\tau}$};
  \node[agent,below left=11mm and 4mm of S](A1){Agent $1$\\ MDP $\mathcal M_1$\\ $\tau$ meta-PG};
  \node[agent,right=6mm of A1](A2){Agent $2$\\ MDP $\mathcal M_2$\\ $\tau$ meta-PG};
  \node[right=4mm of A2](dots){$\cdots$};
  \node[agent,right=4mm of dots](An){Agent $n$\\ MDP $\mathcal M_n$\\ $\tau$ meta-PG};
  \foreach \a in {A1,A2,An}{\draw[->,srv,thick] ($(S.south)+(-10mm,0)$) -- ($(\a.north)+(-3mm,0)$);}
  \foreach \a in {A1,A2,An}{\draw[->,agt!80!black,dashed] ($(\a.north)+(3mm,0)$) -- ($(S.south)+(10mm,0)$);}
  \node[srv,font=\scriptsize] at ($(S.south)+(-19mm,-4mm)$){broadcast $\theta^k$};
  \node[agt!70!black,font=\scriptsize] at ($(S.south)+(20mm,-4mm)$){return $\theta_i^{k,\tau}$};
  \node[note,below=7mm of A2](N){\textbf{local meta-PG step} (fresh batches):
  $\theta\leftarrow\theta+\beta\,\widehat{\nabla F}_i(\theta)$, \; $F_i(\theta)=J_i\!\big(\theta+\alpha\nabla J_i(\theta)\big)$};
  \draw[->,ggray] (A2.south) -- (N.north);
\end{tikzpicture}
\caption{One communication round of \textsc{Per-FedAvg-PG} (Algorithm~\ref{alg:main}). Only parameters cross the link; trajectories stay local.}
\Description{A schematic of one communication round. A wide server box at the top states the averaging rule over the n returned parameter vectors. Solid arrows fan downward from the server to a row of agent boxes labeled agent 1, agent 2 and agent n, each holding its own MDP and performing tau local meta-policy-gradient steps, and dashed arrows return from each agent back to the server. A dashed note box below the agents gives the local update rule and the per-agent objective.}
\label{fig:method}
\end{figure}

\subsection{Notation and constants}\label{app:notation}

For convenience we repeat the quantities defined in the main text. The trajectory law is \eqref{eq:trajlaw}, the policy gradient and its score-function estimator are \eqref{eq:pg}, and the Hessian representation is \eqref{eq:hess}, where
\[
\nu_i(\xi;\theta)=\sum_{h=0}^{H}\log\pi(a_h|s_h;\theta)\,R_i^h(\xi),
\qquad g_i=\nabla\nu_i,
\qquad \nabla_\theta\log q_i(\xi;\theta)=\sum_{h=0}^H\nabla_\theta\log\pi(a_h|s_h;\theta),
\]
the last identity because $\mu_i$ and $P_i$ do not depend on $\theta$. The base constants are those of \eqref{eq:constants},
\[
\bar R=\tfrac{\Rmax}{1-\gamma},\quad
\GJ=\tfrac{\gmax \Rmax}{(1-\gamma)^2},\quad
\LJ=\tfrac{\Rmax((H+1)\gmax^2+M)}{(1-\gamma)^2},\quad
\bar M=\tfrac{M\Rmax}{(1-\gamma)^2},\quad
\bar\Lambda=\tfrac{\Lambda\Rmax}{(1-\gamma)^2},
\]
\[
\rJ=(H+1)\gmax(\LJ+\bar M)+(H+1)M\GJ+\bar\Lambda,
\]
and the meta-constants of Lemma~\ref{lem:meta-reg} are $\GF=(1+\alpha\LJ)\GJ$ and $\LF=(1+\alpha\LJ)^2\LJ+\alpha\rJ\GJ$. The adapted parameter is $\theta_i^{+}(\theta)=\theta+\alpha\nabla J_i(\theta)$, the meta-gradient is \eqref{eq:metagrad}, and the estimator is \eqref{eq:estimator}, built from $\hat g_{\mathrm{in}},\hat H$ (independent batches of sizes $m_{\mathrm{in}},m_{\mathrm h}$ from $q_i(\cdot;\theta)$), $\tilde\theta=\theta+\alpha\hat g_{\mathrm{in}}$, and $\hat g_{\mathrm{out}}$ ($m_{\mathrm{out}}$ trajectories from $q_i(\cdot;\tilde\theta)$).

We use the following elementary facts without comment: for a random vector $X$ with $\norm{X}\le C$ almost surely, $\E\norm{X-\E X}^2\le\E\norm{X}^2\le C^2$, and the empirical mean $\bar X_m$ of $m$ i.i.d.\ copies satisfies $\E\norm{\bar X_m-\E X}^2\le C^2/m$; for any $k$ vectors, $\norm{\sum_{j\le k}x_j}^2\le k\sum_{j\le k}\norm{x_j}^2$; and $\norm{vw^\top}_F=\norm{v}\norm{w}$.

\section{Proofs}\label{app:proofs}
This appendix proves every result stated in the paper. The notation and constants are collected in Section~\ref{app:notation}. Sections~\ref{app:lem1}--\ref{app:regimes} then give the proofs in the order the results appear. Throughout, $\norm{\cdot}$ is the Euclidean norm for vectors and the spectral norm for matrices, $\norm{\cdot}_F$ is the Frobenius norm, and we use $\norm{A}\le\norm{A}_F$ freely.

\subsection[Proof of the policy-gradient regularity lemma]{Proof of Lemma~\ref{lem:pg-bounds} (uniform policy-gradient regularity)}\label{app:lem1}

\begin{proof}
Fix $i\in[n]$ throughout and abbreviate, for a trajectory $\xi=(s_0,a_0,\dots,s_H,a_H)$,
\[
\psi_h(\theta):=\nabla_\theta\log\pi(a_h|s_h;\theta)\in\R^d,
\qquad
\Psi_h(\theta):=\nabla^2_\theta\log\pi(a_h|s_h;\theta)\in\R^{d\times d},
\]
so that Assumption~\ref{as:policy}(a)--(b) read $\norm{\psi_h(\theta)}\le\gmax$ and $\norm{\Psi_h(\theta)}_F\le M$ for every $h$, every $\xi$ and every $\theta$. Write
\[
z_i(\xi;\theta):=\nabla_\theta\log q_i(\xi;\theta)=\sum_{h=0}^H\psi_h(\theta),
\]
the trajectory score, where the second equality is obtained by taking $\log$ in \eqref{eq:trajlaw} and differentiating: the factors $\mu_i(s_0)$ and
$P_i(s_{h+1}|s_h,a_h)$ do not depend on $\theta$ and therefore drop out. In this notation the reward-to-go policy gradient of \eqref{eq:pg} is $g_i(\xi;\theta)=\sum_{h=0}^HR_i^h(\xi)\,\psi_h(\theta)$, and the quantity $\nu_i$ of the main text satisfies $\nabla\nu_i=g_i$ and $\nabla^2\nu_i(\xi;\theta)=\sum_{h=0}^HR_i^h(\xi)\,\Psi_h(\theta)$, since $R_i^h(\xi)$ is a function of the trajectory alone and carries no $\theta$-dependence.

We use the following elementary facts.
\begin{enumerate}
\item[(E1)] $\displaystyle\sum_{t=h}^{H}\gamma^t\le\sum_{t=h}^{\infty}\gamma^t=\frac{\gamma^h}{1-\gamma}$ and $\displaystyle\sum_{h=0}^{H}\gamma^h\le\frac{1}{1-\gamma}$, both because $\gamma\in[0,1)$.
\item[(E2)] $\norm{vw^{\!\top}}_F=\norm{v}\,\norm{w}$ for $v,w\in\R^d$: an outer product has the single nonzero singular value $\norm{v}\norm{w}$.
\item[(E3)] $\norm{A}\le\norm{A}_F$ for every matrix $A$, since the largest singular value is at most the square root of the sum of squared singular values.
\item[(E4)] $\norm{\E[X]}\le\E\norm{X}$ for any integrable random vector or matrix $X$ and any norm (Jensen's inequality; here all expectations are finite sums).
\item[(E5)] Each $\psi_h$ is $M$-Lipschitz in $\theta$: by (E3), $\norm{\nabla_\theta\psi_h(\theta)}=\norm{\Psi_h(\theta)}\le\norm{\Psi_h(\theta)}_F\le M$ for all $\theta$, so the mean value inequality gives $\norm{\psi_h(\theta_1)-\psi_h(\theta_0)}\le M\norm{\theta_1-\theta_0}$.
\end{enumerate}

\medskip
\noindent\emph{Step 1: statement (i), and a summed reward bound.}
By the triangle inequality, Assumption~\ref{as:reward} and (E1),
\[
\big|R_i^h(\xi)\big|
=\Big|\sum_{t=h}^{H}\gamma^t r_i(s_t,a_t)\Big|
\le\sum_{t=h}^{H}\gamma^t\big|r_i(s_t,a_t)\big|
\;\overset{\text{Asm.~\ref{as:reward}}}{\le}\;
\Rmax\sum_{t=h}^{H}\gamma^t
\;\overset{\text{(E1)}}{\le}\;
\frac{\Rmax\gamma^h}{1-\gamma}.
\]
Taking $h=0$ gives $|R_i(\xi)|=|R_i^0(\xi)|\le\Rmax/(1-\gamma)=\bar R$ for every $\xi$, hence $|J_i(\theta)|=\big|\E[R_i^0(\xi)]\big|\le\E\big|R_i^0(\xi)\big|\le\bar R$ by (E4). This proves (i).

Summing the displayed bound over $h$ and using (E1) once more yields the quantity that recurs in every remaining step:
\begin{equation}\label{eq:SR}
S_R(\xi):=\sum_{h=0}^{H}\big|R_i^h(\xi)\big|
\le\frac{\Rmax}{1-\gamma}\sum_{h=0}^{H}\gamma^h
\le\frac{\Rmax}{(1-\gamma)^2}
\qquad\text{for every }\xi .
\end{equation}
With \eqref{eq:SR} the constants of \eqref{eq:constants} take the compact form $\GJ=\gmax\cdot\Rmax/(1-\gamma)^2$, $\bar M=M\cdot\Rmax/(1-\gamma)^2$ and $\bar\Lambda=\Lambda\cdot\Rmax/(1-\gamma)^2$; that is, each is the relevant policy constant times the same factor $\Rmax/(1-\gamma)^2$.

\medskip
\noindent\emph{Step 2: statement (ii).}
Directly from $g_i(\xi;\theta)=\sum_{h}R_i^h(\xi)\psi_h(\theta)$, the triangle inequality, Assumption~\ref{as:policy}(a) and \eqref{eq:SR},
\[
\norm{g_i(\xi;\theta)}
\le\sum_{h=0}^{H}\big|R_i^h(\xi)\big|\,\norm{\psi_h(\theta)}
\;\overset{\text{Asm.~\ref{as:policy}(a)}}{\le}\;
\gmax\sum_{h=0}^{H}\big|R_i^h(\xi)\big|
=\gmax\,S_R(\xi)
\;\overset{\eqref{eq:SR}}{\le}\;
\frac{\gmax\Rmax}{(1-\gamma)^2}=\GJ ,
\]
uniformly in $\xi$ and $\theta$. Since $\nabla J_i(\theta)=\E_{\xi\sim q_i(\cdot;\theta)}[g_i(\xi;\theta)]$ by \eqref{eq:pg}, (E4) gives $\norm{\nabla J_i(\theta)}\le\E\norm{g_i(\xi;\theta)}\le\GJ$. This proves (ii). The same computation bounds the score, which we record for later:
\begin{equation}\label{eq:scorebound}
\norm{z_i(\xi;\theta)}\le\sum_{h=0}^{H}\norm{\psi_h(\theta)}\le(H+1)\gmax .
\end{equation}

\medskip
\noindent\emph{Step 3: derivation of the Hessian representation \eqref{eq:hess}.}
We derive \eqref{eq:hess} rather than quote it, so that the constant in (iii) can be read off directly. Because $\mathcal S_i$ and $\mathcal A_i$ are finite and the horizon is $H$, there are finitely many trajectories, so $\nabla J_i(\theta)=\sum_\xi q_i(\xi;\theta)\,g_i(\xi;\theta)$ is a \emph{finite} sum of products of continuously differentiable functions and may be differentiated term by term. Differentiating the $j$-th coordinate with respect to $\theta_k$ and using the log-derivative identity $\partial_k q_i=q_i\,\partial_k\log q_i=q_i\,(z_i)_k$,
\[
\partial_k\big(\nabla J_i(\theta)\big)_j
=\sum_\xi\Big[\underbrace{\partial_kq_i(\xi;\theta)}_{=\,q_i(z_i)_k}\,\big(g_i\big)_j
+q_i(\xi;\theta)\,\partial_k\big(g_i\big)_j\Big]
=\E_{\xi\sim q_i(\cdot;\theta)}\Big[\big(g_i\big)_j\big(z_i\big)_k+\big(\nabla_\theta g_i\big)_{jk}\Big].
\]
In matrix form, and using $\nabla_\theta g_i=\sum_{h}R_i^h\,\Psi_h=\nabla^2\nu_i$ (again because $R_i^h$ is $\theta$-free), this is exactly \eqref{eq:hess}:
\begin{equation}\label{eq:hessexplicit}
\nabla^2J_i(\theta)=\E_{\xi\sim q_i(\cdot;\theta)}\big[u_i(\xi;\theta)\big],
\qquad
u_i(\xi;\theta)=\underbrace{g_i(\xi;\theta)\,z_i(\xi;\theta)^{\!\top}}_{\text{rank one}}
+\underbrace{\sum_{h=0}^{H}R_i^h(\xi)\,\Psi_h(\theta)}_{=\,\nabla^2\nu_i(\xi;\theta)} ,
\end{equation}
the reward-to-go form of the stochastic second-order differential of~\cite{shen}. Note that the summand $u_i$ need not be symmetric, whereas its expectation $\nabla^2J_i(\theta)$ is; we only ever use $\norm{\cdot}$ as the operator norm, which is defined for arbitrary matrices, so this causes no difficulty.

\medskip
\noindent\emph{Step 4: statement (iii).}
We bound the two terms of \eqref{eq:hessexplicit} separately. For the rank-one term, (E2) turns the Frobenius norm into a product of vector norms, which Step 2 and \eqref{eq:scorebound} bound:
\[
\norm{g_i\,z_i^{\!\top}}_F
\overset{\text{(E2)}}{=}\norm{g_i}\,\norm{z_i}
\le\GJ\cdot(H+1)\gmax .
\]
For the second term, the triangle inequality, Assumption~\ref{as:policy}(b) and \eqref{eq:SR} give
\[
\Big\lVert\sum_{h=0}^{H}R_i^h\,\Psi_h\Big\rVert_F
\le\sum_{h=0}^{H}\big|R_i^h\big|\,\norm{\Psi_h}_F
\;\overset{\text{Asm.~\ref{as:policy}(b)}}{\le}\;
M\,S_R(\xi)
\;\overset{\eqref{eq:SR}}{\le}\;
\frac{M\Rmax}{(1-\gamma)^2}=\bar M .
\]
Adding the two, and substituting $\GJ=\gmax\Rmax/(1-\gamma)^2$ and $\bar M=M\Rmax/(1-\gamma)^2$ to combine them over the common denominator,
\[
\norm{u_i(\xi;\theta)}_F
\le(H+1)\gmax\GJ+\bar M
=\frac{(H+1)\gmax^2\Rmax}{(1-\gamma)^2}+\frac{M\Rmax}{(1-\gamma)^2}
=\frac{\Rmax\big((H+1)\gmax^2+M\big)}{(1-\gamma)^2}=\LJ ,
\]
uniformly in $\xi$ and $\theta$. Consequently, by \eqref{eq:hessexplicit}, (E3) and (E4),
\begin{equation}\label{eq:hessnorm}
\norm{\nabla^2J_i(\theta)}
\overset{\text{(E3)}}{\le}\norm{\E[u_i]}_F
\overset{\text{(E4)}}{\le}\E\norm{u_i}_F\le\LJ
\qquad\text{for every }\theta .
\end{equation}

\medskip
\noindent\emph{Step 5: $\nabla J_i$ is $\LJ$-Lipschitz.}
Fix $\theta_0,\theta_1$ and put $\delta:=\theta_1-\theta_0$, $\theta_s:=\theta_0+s\delta$. The map $s\mapsto\nabla J_i(\theta_s)$ is continuously differentiable with derivative $\nabla^2J_i(\theta_s)\delta$, so by the fundamental theorem of calculus and \eqref{eq:hessnorm},
\[
\norm{\nabla J_i(\theta_1)-\nabla J_i(\theta_0)}
=\Big\lVert\int_0^1\nabla^2J_i(\theta_s)\,\delta\,ds\Big\rVert
\le\int_0^1\norm{\nabla^2J_i(\theta_s)}\,\norm{\delta}\,ds
\le\LJ\norm{\delta},
\]
which completes (iii).

\medskip
\noindent\emph{Step 6: statement (iv).}
Keep $\delta,\theta_s$ as in Step 5. We bound the \emph{difference} $\nabla^2J_i(\theta_1)-\nabla^2J_i(\theta_0)$ directly rather than differentiating $\nabla^2J_i(\theta_s)$ in $s$. This is deliberate: differentiating would require a third derivative of $\log\pi$, which Assumption~\ref{as:policy} does not provide, since (c) grants only Lipschitz continuity of $\nabla^2\log\pi$ and not its differentiability. The argument below uses (c) exactly as stated and obtains the same constant.

Writing $\nabla^2J_i(\theta)=\sum_\xi q_i(\xi;\theta)u_i(\xi;\theta)$ and inserting the mixed term $\sum_\xi q_i(\xi;\theta_0)u_i(\xi;\theta_1)$,
\begin{equation}\label{eq:AB}
\nabla^2J_i(\theta_1)-\nabla^2J_i(\theta_0)
=\underbrace{\sum_\xi\big[q_i(\xi;\theta_1)-q_i(\xi;\theta_0)\big]u_i(\xi;\theta_1)}_{=:A\ \text{(distribution shift)}}
\;+\;\underbrace{\sum_\xi q_i(\xi;\theta_0)\big[u_i(\xi;\theta_1)-u_i(\xi;\theta_0)\big]}_{=:B\ \text{(curvature change)}} .
\end{equation}

\emph{Bounding $A$.} The trajectory weights move by at most $(H+1)\gmax\norm{\delta}$ in $\ell_1$ distance, equivalently by half that in total variation. Indeed, $s\mapsto q_i(\xi;\theta_s)$ is continuously differentiable with $\frac{d}{ds}q_i(\xi;\theta_s)=q_i(\xi;\theta_s)\inner{z_i(\xi;\theta_s)}{\delta}$ by the log-derivative identity, so
\[
\sum_\xi\big|q_i(\xi;\theta_1)-q_i(\xi;\theta_0)\big|
=\sum_\xi\Big|\int_0^1 q_i(\xi;\theta_s)\inner{z_i(\xi;\theta_s)}{\delta}\,ds\Big|
\le\int_0^1\underbrace{\sum_\xi q_i(\xi;\theta_s)\big|\inner{z_i(\xi;\theta_s)}{\delta}\big|}_{\le\,(H+1)\gmax\norm{\delta}\ \text{by Cauchy--Schwarz and }\eqref{eq:scorebound}}\!\!ds
\le(H+1)\gmax\norm{\delta},
\]
where we used $\sum_\xi q_i(\xi;\theta_s)=1$. Combining with $\norm{u_i}_F\le\LJ$ from Step 4,
\begin{equation}\label{eq:Abound}
\norm{A}_F\le\LJ\sum_\xi\big|q_i(\xi;\theta_1)-q_i(\xi;\theta_0)\big|\le(H+1)\gmax\LJ\norm{\delta}.
\end{equation}

\emph{Bounding $B$.} It suffices to show that $\theta\mapsto u_i(\xi;\theta)$ is Lipschitz uniformly in $\xi$, since $\sum_\xi q_i(\xi;\theta_0)=1$. Using \eqref{eq:hessexplicit} and splitting the rank-one part with the identity
$v_1w_1^{\!\top}-v_0w_0^{\!\top}=(v_1-v_0)w_1^{\!\top}+v_0(w_1-w_0)^{\!\top}$,
\[
u_i(\xi;\theta_1)-u_i(\xi;\theta_0)
=\big(g_i^{(1)}-g_i^{(0)}\big)\big(z_i^{(1)}\big)^{\!\top}
+g_i^{(0)}\big(z_i^{(1)}-z_i^{(0)}\big)^{\!\top}
+\sum_{h=0}^{H}R_i^h\big(\Psi_h(\theta_1)-\Psi_h(\theta_0)\big),
\]
with the shorthand $g_i^{(j)}:=g_i(\xi;\theta_j)$ and $z_i^{(j)}:=z_i(\xi;\theta_j)$. The three terms are bounded as follows.
\begin{itemize}
\item $\norm{g_i^{(1)}-g_i^{(0)}}\le\sum_h|R_i^h|\,\norm{\psi_h(\theta_1)-\psi_h(\theta_0)}
      \overset{\text{(E5)}}{\le}M\norm{\delta}S_R(\xi)\overset{\eqref{eq:SR}}{\le}\bar M\norm{\delta}$, so by (E2) and \eqref{eq:scorebound} the first term has Frobenius norm at most $\bar M(H+1)\gmax\norm{\delta}$.
\item $\norm{z_i^{(1)}-z_i^{(0)}}\le\sum_{h=0}^{H}\norm{\psi_h(\theta_1)-\psi_h(\theta_0)} \overset{\text{(E5)}}{\le}(H+1)M\norm{\delta}$, so by (E2) and Step 2 the second term has Frobenius norm at most $\GJ(H+1)M\norm{\delta}$.
\item $\big\lVert\sum_hR_i^h(\Psi_h(\theta_1)-\Psi_h(\theta_0))\big\rVert_F \le\sum_h|R_i^h|\,\norm{\Psi_h(\theta_1)-\Psi_h(\theta_0)}_F \overset{\text{Asm.~\ref{as:policy}(c)}}{\le}\Lambda\norm{\delta}S_R(\xi) \overset{\eqref{eq:SR}}{\le}\bar\Lambda\norm{\delta}$.
\end{itemize}
Adding the three and averaging over $q_i(\cdot;\theta_0)$,
\begin{equation}\label{eq:Bbound}
\norm{B}_F\le\big[(H+1)\gmax\bar M+(H+1)M\GJ+\bar\Lambda\big]\norm{\delta}.
\end{equation}

\emph{Conclusion.} Inserting \eqref{eq:Abound} and \eqref{eq:Bbound} into \eqref{eq:AB} and using
(E3),
\[
\norm{\nabla^2J_i(\theta_1)-\nabla^2J_i(\theta_0)}
\le\norm{A}_F+\norm{B}_F
\le\underbrace{\big[(H+1)\gmax(\LJ+\bar M)+(H+1)M\GJ+\bar\Lambda\big]}_{=\,\rJ\ \text{by }\eqref{eq:constants}}\norm{\theta_1-\theta_0},
\]
where the two $(H+1)\gmax$ contributions, one from $A$ and one from $B$, have been collected into the factor $(H+1)\gmax(\LJ+\bar M)$. This is (iv).
\end{proof}

\begin{remark}[Where each assumption is used]\label{rem:whereused} The three parts of Assumption~\ref{as:policy} enter cumulatively rather than one per statement. Part (i) uses Assumption~\ref{as:reward} alone. Part (a) is first needed in (ii) and reappears in (iii) and (iv), through the factor $(H+1)\gmax^2$ of $\LJ$ and the factor $(H+1)\gmax$ of $\rJ$; part (b) is first needed in (iii) and reappears in (iv), through $\bar M$ and $(H+1)M\GJ$; part (c) is confined to (iv), entering only through $\bar\Lambda$. It is therefore (c) alone, and hence the constant $\rJ$ alone, that is required only for the meta-smoothness constant $\LF$ and the sharpened bias bound of Proposition~\ref{prop:sharpbias}. Every bound is \emph{pathwise}, i.e.\ uniform over $\xi$ as well as $\theta$, which is what makes the drift bound of Lemma~\ref{lem:drift} deterministic rather than in expectation.
\end{remark}

\subsection[Proof of the meta-regularity lemma]{Proof of Lemma~\ref{lem:meta-reg} (regularity of the meta-objective)}\label{app:lem2}

\begin{proof}
Fix $i\in[n]$ and $\alpha>0$ throughout, and abbreviate the \emph{adaptation map} and the
\emph{adaptation Jacobian}
\[
T_i(\theta):=\theta_i^+(\theta)=\theta+\alpha\nabla J_i(\theta),
\qquad
\mathcal J_i(\theta):=I+\alpha\nabla^2J_i(\theta),
\]
so that $F_i=J_i\circ T_i$.

\medskip
\noindent\emph{Step 0: differentiability, symmetry, and the chain rule \eqref{eq:metagrad}.}
By Assumption~\ref{as:policy}, $\theta\mapsto\log\pi(a|s;\theta)$ is twice continuously differentiable, hence so is each trajectory weight $q_i(\cdot;\theta)$ of \eqref{eq:trajlaw}, being a finite product of such factors. As $\mathcal S_i,\mathcal A_i$ are finite and the horizon is $H$, the sum $J_i(\theta)=\sum_\xi q_i(\xi;\theta)R_i(\xi)$ is finite, so $J_i\in C^2(\R^d)$. Two consequences are used below. First, $\nabla^2J_i(\theta)$ is \emph{symmetric} for every $\theta$, by Schwarz's theorem. (Note that the per-trajectory matrix $u_i(\xi;\theta)$ of \eqref{eq:hessexplicit} is \emph{not} symmetric; only its expectation is.) Second, $T_i$ is continuously differentiable with Jacobian
\begin{equation}\label{eq:jacobian}
DT_i(\theta)=I+\alpha\,\nabla^2J_i(\theta)=\mathcal J_i(\theta),
\end{equation}
since differentiating $\theta\mapsto\theta$ gives $I$ and differentiating $\theta\mapsto\nabla J_i(\theta)$ gives $\nabla^2J_i(\theta)$.

Therefore $F_i=J_i\circ T_i$ is differentiable, and the chain rule for a scalar function composed with a vector map gives
\[
\nabla F_i(\theta)=\big[DT_i(\theta)\big]^{\!\top}\,\nabla J_i\big(T_i(\theta)\big).
\]
By \eqref{eq:jacobian} and the symmetry of $\nabla^2J_i(\theta)$ we have $[DT_i(\theta)]^{\!\top}=\mathcal J_i(\theta)^{\!\top}=\mathcal J_i(\theta)$, so the transpose may be dropped and
\begin{equation}\label{eq:metagradproof}
\nabla F_i(\theta)=\mathcal J_i(\theta)\,\nabla J_i\big(T_i(\theta)\big)
=\big(I+\alpha\nabla^2J_i(\theta)\big)\nabla J_i\big(\theta_i^+(\theta)\big),
\end{equation}
which is exactly \eqref{eq:metagrad}. The symmetry is needed at exactly this point: without it the meta-gradient would carry $\nabla^2J_i(\theta)^{\!\top}$, and the estimator \eqref{eq:estimator} would have to transpose $\hat H$ accordingly.

\medskip
\noindent\emph{Step 1: spectral bounds on the adaptation Jacobian.}
Because $\nabla^2J_i(\theta)$ is symmetric with $\norm{\nabla^2J_i(\theta)}\le\LJ$ by Lemma~\ref{lem:pg-bounds}(iii), its eigenvalues $\lambda_1,\dots,\lambda_d$ are real and lie in $[-\LJ,\LJ]$. The eigenvalues of $\mathcal J_i(\theta)$ are therefore $1+\alpha\lambda_j$, and since $\mathcal J_i(\theta)$ is symmetric its operator norm is the largest absolute eigenvalue:
\begin{equation}\label{eq:specbound}
\norm{\mathcal J_i(\theta)}=\max_{j\le d}\big|1+\alpha\lambda_j\big|\le 1+\alpha\LJ,
\qquad
1+\alpha\lambda_j\in[\,1-\alpha\LJ,\;1+\alpha\LJ\,]\ \ \text{for all }j .
\end{equation}
The same bound follows from the triangle inequality alone, $\norm{I+\alpha\nabla^2J_i(\theta)}\le\norm{I}+\alpha\norm{\nabla^2J_i(\theta)}\le1+\alpha\LJ$, where $\norm{I}=1$ \emph{because $\norm{\cdot}$ is the operator norm}; the Frobenius norm would give the useless $\norm{I}_F=\sqrt d$, which is why (ii) and (iii) are stated in the operator norm while Lemma~\ref{lem:pg-bounds}(iii) supplies the dimension-free Frobenius bound. The two-sided form in
\eqref{eq:specbound} is what Remark~\ref{rem:interp} uses: when $\alpha\LJ<1$ every eigenvalue $1+\alpha\lambda_j$ is bounded below by $1-\alpha\LJ>0$, so $\mathcal J_i(\theta)$ is invertible with $\norm{\mathcal J_i(\theta)^{-1}}=\big(\min_j|1+\alpha\lambda_j|\big)^{-1}\le(1-\alpha\LJ)^{-1}$.

\medskip
\noindent\emph{Step 2: statement (i), and the bound $\Delta\le2\bar R$.}
Lemma~\ref{lem:pg-bounds}(i) states $|J_i(\theta')|\le\bar R$ for \emph{every} $\theta'\in\R^d$; applying it at the particular point $\theta'=T_i(\theta)$ gives
\[
\big|F_i(\theta)\big|=\big|J_i\big(T_i(\theta)\big)\big|\le\bar R
\qquad\text{for every }\theta\in\R^d .
\]
Averaging over $i$ and using the triangle inequality gives $|F(\theta)|\le\bar R$ for every $\theta$, hence
\[
-\bar R\;\le\;\inf_\theta F(\theta)
\qquad\text{and}\qquad
F^\star:=\sup_\theta F(\theta)\le\bar R<\infty .
\]
Both bounds are needed: finiteness of $F^\star$ makes the optimality gap $\Delta=F^\star-F(\theta^0)$ well defined, and combining the two gives the quantitative bound asserted in Theorem~\ref{thm:main},
\begin{equation}\label{eq:Deltabound}
\Delta=F^\star-F(\theta^0)\le\bar R-(-\bar R)=2\bar R=\frac{2\Rmax}{1-\gamma}.
\end{equation}
In particular $\Delta$ is bounded by an explicit problem constant, with no assumption on the initialization $\theta^0$; this is a feature of the RL setting, since in supervised personalized FL the analogous gap must be assumed finite.

\medskip
\noindent\emph{Step 3: statement (ii).}
Apply submultiplicativity of the operator norm, $\norm{Av}\le\norm{A}\norm{v}$, to \eqref{eq:metagradproof}, then \eqref{eq:specbound} and Lemma~\ref{lem:pg-bounds}(ii) evaluated at the point $T_i(\theta)$:
\[
\norm{\nabla F_i(\theta)}
\le\norm{\mathcal J_i(\theta)}\cdot\norm{\nabla J_i\big(T_i(\theta)\big)}
\;\overset{\eqref{eq:specbound}}{\le}\;(1+\alpha\LJ)\cdot\norm{\nabla J_i\big(T_i(\theta)\big)}
\;\overset{\text{Lem.~\ref{lem:pg-bounds}(ii)}}{\le}\;(1+\alpha\LJ)\GJ=\GF .
\]
Both factors are uniform in $\theta$ because Lemma~\ref{lem:pg-bounds}(ii)--(iii) hold for all parameters; in particular no bound on $\norm{\theta}$ or on the distance to a stationary point is needed. Averaging over $i$ gives $\norm{\nabla F(\theta)}\le\GF$ as well.

\medskip
\noindent\emph{Step 4: the adaptation map is $(1+\alpha\LJ)$-Lipschitz.}
For any $\theta,\theta'$, the definition of $T_i$, the triangle inequality and Lemma~\ref{lem:pg-bounds}(iii) give
\begin{equation}\label{eq:adaptlip}
\norm{T_i(\theta)-T_i(\theta')}
\;\overset{(*)}{\le}\;\norm{\theta-\theta'}+\alpha\norm{\nabla J_i(\theta)-\nabla J_i(\theta')}
\;\overset{(**)}{\le}\;(1+\alpha\LJ)\norm{\theta-\theta'},
\end{equation}
where $(*)$ is the triangle inequality applied to $T_i(\theta)-T_i(\theta')=(\theta-\theta')+\alpha(\nabla J_i(\theta)-\nabla J_i(\theta'))$ and $(**)$ uses the $\LJ$-Lipschitz continuity of $\nabla J_i$ from Lemma~\ref{lem:pg-bounds}(iii). Equivalently, \eqref{eq:adaptlip} follows from \eqref{eq:jacobian} and \eqref{eq:specbound} via the mean value inequality, $\norm{T_i(\theta)-T_i(\theta')}\le\sup_s\norm{DT_i(\theta_s)}\norm{\theta-\theta'}$. The interpretation is that a perturbation of the shared initialization is amplified by at most $1+\alpha\LJ$ before it reaches the adapted parameter, which is precisely the factor that will appear squared in $\LF$.

\medskip
\noindent\emph{Step 5: statement (iii) for each $F_i$.}
Write $v_i(\theta):=\nabla J_i(T_i(\theta))$, so that $\nabla F_i(\theta)=\mathcal J_i(\theta)v_i(\theta)$ by \eqref{eq:metagradproof}. Insert and remove the mixed term $\mathcal J_i(\theta)v_i(\theta')$:
\[
\nabla F_i(\theta)-\nabla F_i(\theta')
=\underbrace{\mathcal J_i(\theta)\big(v_i(\theta)-v_i(\theta')\big)}_{\text{adapted gradient moves}}
+\underbrace{\big(\mathcal J_i(\theta)-\mathcal J_i(\theta')\big)v_i(\theta')}_{\text{curvature moves}},
\qquad
\mathcal J_i(\theta)-\mathcal J_i(\theta')=\alpha\big(\nabla^2J_i(\theta)-\nabla^2J_i(\theta')\big),
\]
the last identity because the $I$ in $\mathcal J_i$ cancels. Hence, by the triangle inequality and submultiplicativity,
\begin{align*}
\norm{\nabla F_i(\theta)-\nabla F_i(\theta')}
&\overset{(+)}{\le}\;
\norm{\mathcal J_i(\theta)}\,\norm{v_i(\theta)-v_i(\theta')}
\;+\;\alpha\,\norm{\nabla^2J_i(\theta)-\nabla^2J_i(\theta')}\,\norm{v_i(\theta')}\\[2pt]
&\overset{(++)}{\le}\;
(1+\alpha\LJ)\cdot\LJ\norm{T_i(\theta)-T_i(\theta')}
\;+\;\alpha\cdot\rJ\norm{\theta-\theta'}\cdot\GJ\\[2pt]
&\overset{(-)}{\le}\;
(1+\alpha\LJ)\cdot\LJ\cdot(1+\alpha\LJ)\norm{\theta-\theta'}+\alpha\rJ\GJ\norm{\theta-\theta'}\\[2pt]
&=\;\Big[(1+\alpha\LJ)^2\LJ+\alpha\rJ\GJ\Big]\norm{\theta-\theta'}
\;=\;\LF\norm{\theta-\theta'},
\end{align*}
where the three steps use, respectively: $(+)$ the decomposition above with the triangle inequality and $\norm{Av}\le\norm{A}\norm{v}$; $(++)$ the bound \eqref{eq:specbound} on $\norm{\mathcal J_i(\theta)}$, then Lemma~\ref{lem:pg-bounds}(iii) applied to the \emph{pair of adapted points} $T_i(\theta),T_i(\theta')$ to obtain $\norm{v_i(\theta)-v_i(\theta')}=\norm{\nabla J_i(T_i(\theta))-\nabla J_i(T_i(\theta'))}\le\LJ\norm{T_i(\theta)-T_i(\theta')}$, then Lemma~\ref{lem:pg-bounds}(iv) for the Hessian difference and Lemma~\ref{lem:pg-bounds}(ii) at $T_i(\theta')$ for $\norm{v_i(\theta')}\le\GJ$; and $(-)$ the Lipschitz bound \eqref{eq:adaptlip} on the adaptation map. This proves that $\nabla F_i$ is $\LF$-Lipschitz. The two summands of $\LF$ are structurally distinct: $(1+\alpha\LJ)^2\LJ$ is the base smoothness of $J_i$ inflated by the adaptation map acting twice (once on the input perturbation and once through the Jacobian factor), while $\alpha\rJ\GJ$ is the genuinely new third-order contribution, absent when $\alpha=0$, coming from the movement of the curvature factor itself.

\medskip
\noindent\emph{Step 6: statement (iii) for $F$, and the resulting smoothness inequality.}
Since $F=\frac1n\sum_iF_i$ we have $\nabla F=\frac1n\sum_i\nabla F_i$, so by the triangle inequality
\[
\norm{\nabla F(\theta)-\nabla F(\theta')}
\le\frac1n\sum_{i=1}^n\norm{\nabla F_i(\theta)-\nabla F_i(\theta')}
\le\frac1n\sum_{i=1}^n\LF\norm{\theta-\theta'}=\LF\norm{\theta-\theta'} :
\]
averaging a family of $\LF$-Lipschitz maps preserves the constant, and in particular the constant does \emph{not} degrade with the number of agents $n$ or with the heterogeneity of their MDPs.

For later use we record the standard consequence of $\LF$-Lipschitz continuity of $\nabla F$. For any $\theta,p\in\R^d$ and any $\beta\in\R$, the function $s\mapsto F(\theta+s\beta p)$ is continuously differentiable, so by the fundamental theorem of calculus
\[
F(\theta+\beta p)-F(\theta)-\beta\inner{\nabla F(\theta)}{p}
=\beta\int_0^1\inner{\nabla F(\theta+s\beta p)-\nabla F(\theta)}{p}\,ds,
\]
and bounding the integrand by Cauchy--Schwarz and the Lipschitz property, $\norm{\nabla F(\theta+s\beta p)-\nabla F(\theta)}\le\LF s|\beta|\norm{p}$, together with $\int_0^1s\,ds=\tfrac12$, yields
\begin{equation}\label{eq:descentlemma}
\Big|F(\theta+\beta p)-F(\theta)-\beta\inner{\nabla F(\theta)}{p}\Big|\le\frac{\LF\beta^2}{2}\norm{p}^2 ,
\quad\text{hence}\quad
F(\theta+\beta p)\ge F(\theta)+\beta\inner{\nabla F(\theta)}{p}-\frac{\LF\beta^2}{2}\norm{p}^2 .
\end{equation}
The second form is the ascent inequality invoked in the proof of Lemma~\ref{lem:onestep}.
\end{proof}

\subsection[Proof of the estimator lemma]{Proof of Lemma~\ref{lem:estimator} (estimator properties)}\label{app:lem3}

\begin{proof}
Throughout, $\E$ denotes expectation conditional on the current iterate $\theta$, which is fixed and treated as deterministic; write $\theta^+:=\theta_i^+(\theta)=\theta+\alpha\nabla J_i(\theta)$ and $\mathcal J_i(\theta):=I+\alpha\nabla^2J_i(\theta)$, so that $\nabla F_i(\theta)=\mathcal J_i(\theta)\nabla J_i(\theta^+)$ by \eqref{eq:metagrad}. Recall the sampling protocol: $\mathcal D_{\mathrm{in}}$ and $\mathcal D_{\mathrm h}$ consist of $m_{\mathrm{in}}$ and $m_{\mathrm h}$ trajectories drawn i.i.d.\ from $q_i(\cdot;\theta)$, the two batches being mutually independent; $\tilde\theta:=\theta+\alpha\hat g_{\mathrm{in}}$ is a function of $\mathcal D_{\mathrm{in}}$ alone; and, \emph{given} $\mathcal D_{\mathrm{in}}$, the batch $\mathcal D_{\mathrm{out}}$ consists of $m_{\mathrm{out}}$ trajectories drawn i.i.d.\ from $q_i(\cdot;\tilde\theta)$.

Define the per-trajectory policy-gradient variance
\begin{equation}\label{eq:sigmadef}
\varsigma_i^2(\theta'):=\E_{\xi\sim q_i(\cdot;\theta')}\norm{g_i(\xi;\theta')-\nabla J_i(\theta')}^2
\;\overset{(\ast)}{\le}\;\E_{\xi\sim q_i(\cdot;\theta')}\norm{g_i(\xi;\theta')}^2
\;\overset{\text{Lem.~\ref{lem:pg-bounds}(ii)}}{\le}\;\GJ^2
\qquad\text{for every }\theta'\in\R^d,
\end{equation}
where $(\ast)$ is (F2) below. This quantity reappears in Proposition~\ref{prop:sharpbias}; keeping it symbolic rather than immediately replacing it by $\GJ^2$ is what makes the sharpened bias bound available later at no extra cost.

We use the following standard facts. In (F1)--(F2) the norm may be the Euclidean norm on $\R^d$ or the Frobenius norm on $\R^{d\times d}$, since both are induced by an inner product (the Hilbert--Schmidt inner product in the matrix case) and the proofs are identical.
\begin{enumerate}
\item[(F1)] \emph{(variance of an i.i.d.\ average)} If $Z_1,\dots,Z_m$ are i.i.d.\ with
      $\E Z_1=0$ and $\E\norm{Z_1}^2<\infty$, then
      $\E\lVert\frac1m\sum_{j}Z_j\rVert^2=\frac1{m^2}\sum_{j,k}\E\inner{Z_j}{Z_k}=\frac1m\E\norm{Z_1}^2$,
      because the cross terms $j\ne k$ vanish by independence and $\E Z_j=0$.
\item[(F2)] \emph{(variance $\le$ second moment)}
      $\E\norm{X-\E X}^2=\E\norm{X}^2-\norm{\E X}^2\le\E\norm{X}^2$.
\item[(F3)] \emph{(Cauchy--Schwarz)} $\E\norm{X}\le\big(\E\norm{X}^2\big)^{1/2}$.
\item[(F4)] $\big\lVert\sum_{j=1}^{k}x_j\big\rVert^2\le k\sum_{j=1}^{k}\norm{x_j}^2$, by convexity
      of $\norm{\cdot}^2$ applied to the average $\frac1k\sum_jx_j$.
\item[(F5)] \emph{(Jensen)} $\norm{\E X}\le\E\norm{X}$.
\item[(F6)] \emph{(bias--variance)} For deterministic $c$,
      $\E\norm{X-c}^2=\E\norm{X-\E X}^2+\norm{\E X-c}^2$, the cross term vanishing since
      $\E[X-\E X]=0$.
\end{enumerate}

\medskip
\noindent\emph{Step 1: statement (i), boundedness for every realization.}
The point to emphasize is that Lemma~\ref{lem:pg-bounds}(ii)--(iii) are \emph{pathwise} bounds: Assumption~\ref{as:policy}(a)--(b) hold uniformly in $(s,a,\theta)$, so $\norm{g_i(\xi;\theta')}\le\GJ$ and $\norm{u_i(\xi;\theta')}_F\le\LJ$ hold for \emph{every} trajectory $\xi$ and \emph{every} parameter $\theta'$, not merely on average. Two consequences.

First, averages inherit the bounds by the triangle inequality, for every realization of the batches:
\[
\norm{\hat g_{\mathrm{out}}}
=\Big\lVert\frac{1}{m_{\mathrm{out}}}\!\!\sum_{\xi\in\mathcal D_{\mathrm{out}}}\!\!g_i(\xi;\tilde\theta)\Big\rVert
\le\frac{1}{m_{\mathrm{out}}}\!\!\sum_{\xi\in\mathcal D_{\mathrm{out}}}\!\!\norm{g_i(\xi;\tilde\theta)}\le\GJ,
\qquad
\norm{\hat H}_F\le\frac{1}{m_{\mathrm h}}\!\sum_{\xi\in\mathcal D_{\mathrm h}}\!\norm{u_i(\xi;\theta)}_F\le\LJ .
\]
Note that the first bound is applied at the \emph{random} parameter $\tilde\theta$; this is legitimate precisely because Lemma~\ref{lem:pg-bounds}(ii) is uniform in the parameter, so no measurability or conditioning argument is needed.

Second, by submultiplicativity, the triangle inequality with $\norm{I}=1$ (operator norm), and
$\norm{\hat H}\le\norm{\hat H}_F$,
\[
\norm{\widehat{\nabla F}_i(\theta)}=\norm{(I+\alpha\hat H)\hat g_{\mathrm{out}}}
\le\big(1+\alpha\norm{\hat H}\big)\norm{\hat g_{\mathrm{out}}}
\le(1+\alpha\LJ)\GJ=\GF .
\]
Since the trajectory space is finite and the bounds above hold for every $\xi$, this inequality holds \emph{surely}, i.e.\ for every realization of $(\mathcal D_{\mathrm{in}},\mathcal D_{\mathrm h},\mathcal D_{\mathrm{out}})$ and not merely outside a null set; we retain the phrase ``almost surely'' only because it is the weaker and more standard formulation. This proves (i).

\medskip
\noindent\emph{Step 2: unbiasedness of the curvature and outer-gradient estimates.}
The curvature batch is an i.i.d.\ average of $u_i(\xi;\theta)$ with $\xi\sim q_i(\cdot;\theta)$, so by \eqref{eq:hess},
\begin{equation}\label{eq:Hunbiased}
\E\big[\hat H\big]=\E_{\xi\sim q_i(\cdot;\theta)}\big[u_i(\xi;\theta)\big]=\nabla^2J_i(\theta),
\end{equation}
and, since $\mathcal D_{\mathrm h}$ is independent of $\mathcal D_{\mathrm{in}}$, also $\E[\hat H\mid\mathcal D_{\mathrm{in}}]=\nabla^2J_i(\theta)$.

For the outer batch, condition on $\mathcal D_{\mathrm{in}}$. Then $\tilde\theta$ is determined, and $\mathcal D_{\mathrm{out}}$ is i.i.d.\ from $q_i(\cdot;\tilde\theta)$, so applying \eqref{eq:pg}
\emph{at the parameter value $\tilde\theta$} gives
\begin{equation}\label{eq:goutunbiased}
\E\big[\hat g_{\mathrm{out}}\,\big|\,\mathcal D_{\mathrm{in}}\big]
=\E_{\xi\sim q_i(\cdot;\tilde\theta)}\big[g_i(\xi;\tilde\theta)\big]=\nabla J_i(\tilde\theta).
\end{equation}
This is the role of the on-policy resampling of Section~\ref{sec:algo}: because $\mathcal D_{\mathrm{out}}$ is drawn under the \emph{adapted} policy, $\hat g_{\mathrm{out}}$ is
conditionally unbiased for $\nabla J_i(\tilde\theta)$ with no importance weights. Had $\mathcal D_{\mathrm{out}}$ been drawn under $q_i(\cdot;\theta)$ instead, \eqref{eq:goutunbiased} would fail and a likelihood-ratio correction, with its attendant variance, would be unavoidable.

\medskip
\noindent\emph{Step 3: factorization of the mean, and the role of independence of $\mathcal D_{\mathrm h}$.}
Conditionally on $\mathcal D_{\mathrm{in}}$, the batches $\mathcal D_{\mathrm h}$ and $\mathcal D_{\mathrm{out}}$ are independent: the conditional law of $\mathcal D_{\mathrm h}$ is $q_i(\cdot;\theta)^{\otimes m_{\mathrm h}}$ regardless of $\mathcal D_{\mathrm{in}}$ and $\mathcal D_{\mathrm{out}}$, while the conditional law of $\mathcal D_{\mathrm{out}}$ depends on $\mathcal D_{\mathrm{in}}$ only. Hence $\hat H$ and $\hat g_{\mathrm{out}}$ are conditionally independent given $\mathcal D_{\mathrm{in}}$, and the expectation of their product factors entrywise: for each coordinate $j$,
\[
\E\Big[\big((I+\alpha\hat H)\hat g_{\mathrm{out}}\big)_j\,\Big|\,\mathcal D_{\mathrm{in}}\Big]
=\sum_{k=1}^{d}\E\big[(I+\alpha\hat H)_{jk}\,\big|\,\mathcal D_{\mathrm{in}}\big]\;
\E\big[(\hat g_{\mathrm{out}})_k\,\big|\,\mathcal D_{\mathrm{in}}\big],
\]
so that, using \eqref{eq:Hunbiased} and \eqref{eq:goutunbiased},
\[
\E\big[\widehat{\nabla F}_i(\theta)\,\big|\,\mathcal D_{\mathrm{in}}\big]
=\big(I+\alpha\,\E[\hat H\mid\mathcal D_{\mathrm{in}}]\big)\,\E\big[\hat g_{\mathrm{out}}\mid\mathcal D_{\mathrm{in}}\big]
=\mathcal J_i(\theta)\,\nabla J_i(\tilde\theta).
\]
Taking expectations over $\mathcal D_{\mathrm{in}}$ (tower rule) and pulling out the matrix $\mathcal J_i(\theta)$, which is deterministic given $\theta$, yields
\begin{equation}\label{eq:factorization}
\E\big[\widehat{\nabla F}_i(\theta)\big]=\mathcal J_i(\theta)\,\E_{\mathcal D_{\mathrm{in}}}\big[\nabla J_i(\tilde\theta)\big].
\end{equation}
Two observations. First, \eqref{eq:factorization} is exactly the claim that \emph{inner-loop sampling is the only source of bias}: the curvature and outer batches contribute variance but no bias, since both were replaced by their exact means. Second, independence of $\mathcal D_{\mathrm h}$ is indispensable here. Had $\hat H$ been built from $\mathcal D_{\mathrm{in}}$ or $\mathcal D_{\mathrm{out}}$, the product would not factor and $\E[\widehat{\nabla F}_i]$ would carry an additional term $\alpha\,\mathrm{Cov}(\hat H,\hat g_{\mathrm{out}})$, which does not vanish as any batch size grows and would destroy the bias bound.

\medskip
\noindent\emph{Step 4: statement (ii).}
By \eqref{eq:metagrad}, $\nabla F_i(\theta)=\mathcal J_i(\theta)\nabla J_i(\theta^+)$ carries the \emph{same} matrix factor $\mathcal J_i(\theta)$ as \eqref{eq:factorization}. Subtracting, the factor comes out on the left:
\begin{equation}\label{eq:biasfactored}
\E\big[\widehat{\nabla F}_i(\theta)\big]-\nabla F_i(\theta)
=\mathcal J_i(\theta)\Big(\E_{\mathcal D_{\mathrm{in}}}\big[\nabla J_i(\tilde\theta)\big]-\nabla J_i(\theta^+)\Big).
\end{equation}
It remains to bound the two factors of \eqref{eq:biasfactored} separately.

\emph{The matrix factor.} By \eqref{eq:specbound},
$\norm{\mathcal J_i(\theta)}\le1+\alpha\LJ$.

\emph{The vector factor.} Since $\nabla J_i(\theta^+)$ is deterministic (it depends on $\theta$ only, not on any batch), it may be moved inside the expectation, after which (F5) and the $\LJ$-Lipschitz continuity of $\nabla J_i$ from Lemma~\ref{lem:pg-bounds}(iii) give
\[
\big\lVert\E\big[\nabla J_i(\tilde\theta)\big]-\nabla J_i(\theta^+)\big\rVert
=\big\lVert\E\big[\nabla J_i(\tilde\theta)-\nabla J_i(\theta^+)\big]\big\rVert
\;\overset{\text{(F5)}}{\le}\;\E\big\lVert\nabla J_i(\tilde\theta)-\nabla J_i(\theta^+)\big\rVert
\;\overset{\text{Lem.~\ref{lem:pg-bounds}(iii)}}{\le}\;\LJ\,\E\norm{\tilde\theta-\theta^+} .
\]
The displacement of the adapted parameter is exactly $\alpha$ times the inner-gradient error,
\begin{equation}\label{eq:displacement}
\tilde\theta-\theta^+=\big(\theta+\alpha\hat g_{\mathrm{in}}\big)-\big(\theta+\alpha\nabla J_i(\theta)\big)
=\alpha\big(\hat g_{\mathrm{in}}-\nabla J_i(\theta)\big),
\end{equation}
and that error is an i.i.d.\ average of the zero-mean vectors $g_i(\xi;\theta)-\nabla J_i(\theta)$, so by (F1) and \eqref{eq:sigmadef},
\begin{equation}\label{eq:innervar}
\E\norm{\hat g_{\mathrm{in}}-\nabla J_i(\theta)}^2
\;\overset{\text{(F1)}}{=}\;\frac{1}{m_{\mathrm{in}}}\,\varsigma_i^2(\theta)
\;\overset{\eqref{eq:sigmadef}}{\le}\;\frac{\GJ^2}{m_{\mathrm{in}}} .
\end{equation}
Combining \eqref{eq:displacement}, (F3) and \eqref{eq:innervar},
\[
\E\norm{\tilde\theta-\theta^+}=\alpha\,\E\norm{\hat g_{\mathrm{in}}-\nabla J_i(\theta)}
\;\overset{\text{(F3)}}{\le}\;\alpha\Big(\E\norm{\hat g_{\mathrm{in}}-\nabla J_i(\theta)}^2\Big)^{1/2}
\;\overset{\eqref{eq:innervar}}{\le}\;\frac{\alpha\GJ}{\sqrt{m_{\mathrm{in}}}} .
\]
Substituting both factors into \eqref{eq:biasfactored},
\[
\big\lVert\E\widehat{\nabla F}_i(\theta)-\nabla F_i(\theta)\big\rVert
\le(1+\alpha\LJ)\cdot\LJ\cdot\frac{\alpha\GJ}{\sqrt{m_{\mathrm{in}}}}
=\frac{(1+\alpha\LJ)\,\alpha\,\LJ\,\GJ}{\sqrt{m_{\mathrm{in}}}}=\bhat,
\]
which proves (ii). The $\sqrt{m_{\mathrm{in}}}$, rather than $m_{\mathrm{in}}$, enters through the single use of (F3): we bounded $\E\norm{\tilde\theta-\theta^+}$, a first moment, by the square root of a second moment. Proposition~\ref{prop:sharpbias} recovers the missing factor by avoiding (F3) altogether.

\medskip
\noindent\emph{Step 5: statement (iii).}
Insert and remove the two intermediate quantities $\mathcal J_i(\theta)\hat g_{\mathrm{out}}$ and $\mathcal J_i(\theta)\nabla J_i(\tilde\theta)$ to obtain the decomposition
\begin{equation}\label{eq:T123}
\widehat{\nabla F}_i-\nabla F_i
=\underbrace{\alpha\big(\hat H-\nabla^2J_i(\theta)\big)\hat g_{\mathrm{out}}}_{T_1\ \text{(curvature)}}
+\underbrace{\mathcal J_i(\theta)\big(\hat g_{\mathrm{out}}-\nabla J_i(\tilde\theta)\big)}_{T_2\ \text{(outer gradient)}}
+\underbrace{\mathcal J_i(\theta)\big(\nabla J_i(\tilde\theta)-\nabla J_i(\theta^+)\big)}_{T_3\ \text{(inner gradient)}} .
\end{equation}
That \eqref{eq:T123} is an identity is verified as follows, since the decomposition is central to the argument. Expanding $T_1$ and collecting the two occurrences of $\hat g_{\mathrm{out}}$,
\[
T_1+T_2
=\alpha\hat H\hat g_{\mathrm{out}}-\alpha\nabla^2J_i(\theta)\hat g_{\mathrm{out}}
+\big(I+\alpha\nabla^2J_i(\theta)\big)\hat g_{\mathrm{out}}-\mathcal J_i(\theta)\nabla J_i(\tilde\theta)
=(I+\alpha\hat H)\hat g_{\mathrm{out}}-\mathcal J_i(\theta)\nabla J_i(\tilde\theta),
\]
and adding $T_3=\mathcal J_i(\theta)\nabla J_i(\tilde\theta)-\mathcal J_i(\theta)\nabla J_i(\theta^+)$ cancels the middle term, leaving $(I+\alpha\hat H)\hat g_{\mathrm{out}}-\mathcal J_i(\theta)\nabla J_i(\theta^+) =\widehat{\nabla F}_i-\nabla F_i$ as claimed. The three terms are \emph{not} independent ($T_2$ and $T_3$ both involve $\tilde\theta$), so we do not attempt to cancel cross terms and instead apply (F4) with $k=3$:
\begin{equation}\label{eq:F4applied}
\E\norm{\widehat{\nabla F}_i-\nabla F_i}^2\le3\Big(\E\norm{T_1}^2+\E\norm{T_2}^2+\E\norm{T_3}^2\Big).
\end{equation}

\emph{Bounding $T_1$.} By submultiplicativity, $\norm{A}\le\norm{A}_F$, and $\norm{\hat g_{\mathrm{out}}}\le\GJ$
from Step 1, we have surely $\norm{T_1}\le\alpha\norm{\hat H-\nabla^2J_i(\theta)}_F\norm{\hat g_{\mathrm{out}}}\le\alpha\GJ\norm{\hat H-\nabla^2J_i(\theta)}_F$. Since $\hat H$ is an i.i.d.\ average of matrices with mean $\nabla^2J_i(\theta)$ by \eqref{eq:Hunbiased}, (F1), (F2) and Lemma~\ref{lem:pg-bounds}(iii) give
\[
\E\norm{\hat H-\nabla^2J_i(\theta)}_F^2
\overset{\text{(F1)}}{=}\frac{1}{m_{\mathrm h}}\E\norm{u_i(\xi;\theta)-\nabla^2J_i(\theta)}_F^2
\overset{\text{(F2)}}{\le}\frac{1}{m_{\mathrm h}}\E\norm{u_i(\xi;\theta)}_F^2
\le\frac{\LJ^2}{m_{\mathrm h}},
\qquad\text{so}\qquad
\E\norm{T_1}^2\le\frac{\alpha^2\GJ^2\LJ^2}{m_{\mathrm h}} .
\]

\emph{Bounding $T_2$.} Condition on $\mathcal D_{\mathrm{in}}$. Given $\mathcal D_{\mathrm{in}}$, the vector $\hat g_{\mathrm{out}}$ is an i.i.d.\ average with mean $\nabla J_i(\tilde\theta)$ by \eqref{eq:goutunbiased}, so (F1) and \eqref{eq:sigmadef} applied at the parameter $\tilde\theta$ give
\[
\E\Big[\norm{\hat g_{\mathrm{out}}-\nabla J_i(\tilde\theta)}^2\,\Big|\,\mathcal D_{\mathrm{in}}\Big]
=\frac{\varsigma_i^2(\tilde\theta)}{m_{\mathrm{out}}}\le\frac{\GJ^2}{m_{\mathrm{out}}} ,
\]
the last step again using that \eqref{eq:sigmadef} holds for \emph{every} parameter, hence at the random point $\tilde\theta$. Multiplying by $\norm{\mathcal J_i(\theta)}^2\le(1+\alpha\LJ)^2$ and taking expectations over $\mathcal D_{\mathrm{in}}$, $\E\norm{T_2}^2\le(1+\alpha\LJ)^2\GJ^2/m_{\mathrm{out}}$.

\emph{Bounding $T_3$.} By submultiplicativity, Lemma~\ref{lem:pg-bounds}(iii) and
\eqref{eq:displacement},
$\norm{T_3}\le(1+\alpha\LJ)\LJ\norm{\tilde\theta-\theta^+}=(1+\alpha\LJ)\LJ\alpha\norm{\hat g_{\mathrm{in}}-\nabla J_i(\theta)}$, so squaring and applying \eqref{eq:innervar} gives $\E\norm{T_3}^2\le(1+\alpha\LJ)^2\LJ^2\alpha^2\GJ^2/m_{\mathrm{in}}$. Note that here, unlike in Step 4, no square root is lost: $T_3$ enters through a \emph{second} moment, so (F3) is not needed.

\emph{Collecting.} Substituting the three bounds into \eqref{eq:F4applied} and factoring out
$3\GJ^2$,
\[
\E\norm{\widehat{\nabla F}_i(\theta)-\nabla F_i(\theta)}^2
\le3\GJ^2\bigg[\frac{\alpha^2\LJ^2}{m_{\mathrm h}}+\frac{(1+\alpha\LJ)^2}{m_{\mathrm{out}}}
+\frac{(1+\alpha\LJ)^2\alpha^2\LJ^2}{m_{\mathrm{in}}}\bigg]=\shat^2 ,
\]
which is (iii). The three summands are in one-to-one correspondence with the three batches, and each is reducible only by enlarging its own batch; in particular the curvature term is the only one an agent cannot reduce by reusing on-policy data it already collects at $\theta$.

Finally, for the centered version, apply (F6) with $X=\widehat{\nabla F}_i(\theta)$ and the deterministic $c=\nabla F_i(\theta)$:
\[
\E\norm{\widehat{\nabla F}_i-\E\widehat{\nabla F}_i}^2
=\E\norm{\widehat{\nabla F}_i-\nabla F_i}^2-\norm{\E\widehat{\nabla F}_i-\nabla F_i}^2
\le\E\norm{\widehat{\nabla F}_i-\nabla F_i}^2\le\shat^2 ,
\]
i.e.\ the conditional variance is bounded by the same $\shat^2$; equivalently, the conditional mean minimizes mean squared distance, so replacing $\nabla F_i(\theta)$ by $\E\widehat{\nabla F}_i(\theta)$ can only decrease the left-hand side. This is the form used in Step~3 of the proof of Lemma~\ref{lem:onestep}.

\medskip
\noindent\emph{Step 6: the role of statement (i) in the analysis.}
Statement (i) is a \emph{strengthening} of what the convergence analysis strictly requires, not an additional hypothesis. Its only use is to make the drift bound of Lemma~\ref{lem:drift} deterministic. Had we only the second-moment bound $\E\norm{\widehat{\nabla F}_i(\theta)}^2\le2\norm{\nabla F_i(\theta)}^2+2\E\norm{\widehat{\nabla F}_i(\theta)-\nabla F_i(\theta)}^2\le2\GF^2+2\shat^2$, which follows from (F4) with $k=2$, Lemma~\ref{lem:meta-reg}(ii) and part (iii) above, without any pathwise argument, then (F4) applied to the unrolled local recursion would still give $\E\norm{\theta_i^{k,t}-\theta^k}^2\le\beta^2t^2(2\GF^2+2\shat^2)$ and the analysis would close with $\GF^2$ replaced by $2\GF^2+2\shat^2$ in the drift term of \eqref{eq:mainbound}. We use the pathwise form because it is available at no cost in the policy-gradient setting and yields the cleaner constant; by contrast, in supervised personalized FL the corresponding bound must be imposed as an assumption on the loss.
\end{proof}

\subsection[Proof of the sharp bias bound]{Proof of Proposition~\ref{prop:sharpbias} (sharp inner-loop bias)}\label{app:prop1}

\begin{proof}
Only one link in the chain of Lemma~\ref{lem:estimator}(ii) changes, namely the passage from the factored form \eqref{eq:biasfactored} to a bound on the vector factor. Everything preceding it, in particular the factorization \eqref{eq:factorization} and its consequence
\begin{equation}\label{eq:biasfactored2}
\E\big[\widehat{\nabla F}_i(\theta)\big]-\nabla F_i(\theta)
=\mathcal J_i(\theta)\Big(\E_{\mathcal D_{\mathrm{in}}}\big[\nabla J_i(\tilde\theta)\big]-\nabla J_i(\theta^+)\Big),
\qquad \mathcal J_i(\theta)=I+\alpha\nabla^2J_i(\theta),
\end{equation}
is unaffected and is used verbatim. As before $\E$ is conditional on $\theta$, $\theta^+=\theta+\alpha\nabla J_i(\theta)$ and $\tilde\theta=\theta+\alpha\hat g_{\mathrm{in}}$. Unlike Lemma~\ref{lem:estimator}(ii), the argument below uses Assumption~\ref{as:policy}(c), through the constant $\rJ$ of Lemma~\ref{lem:pg-bounds}(iv); this is the only additional requirement the improvement imposes.

\medskip
\noindent\emph{Step 1: unbiasedness of the adapted parameter, and insertion of the linear term.}
Since $\mathcal D_{\mathrm{in}}$ consists of trajectories drawn i.i.d.\ from $q_i(\cdot;\theta)$, the inner estimate is unbiased, $\E[\hat g_{\mathrm{in}}]=\nabla J_i(\theta)$ by \eqref{eq:pg}, and therefore
\begin{equation}\label{eq:thetatildeunbiased}
\E\big[\tilde\theta\big]=\theta+\alpha\,\E\big[\hat g_{\mathrm{in}}\big]=\theta+\alpha\nabla J_i(\theta)=\theta^+,
\qquad\text{equivalently}\qquad
\E\big[\tilde\theta-\theta^+\big]=0 .
\end{equation}
The matrix $\nabla^2J_i(\theta^+)$ is \emph{deterministic}: $\theta^+$ is a function of $\theta$ alone and carries no dependence on any batch. Consequently the map $\tilde\theta\mapsto\nabla^2J_i(\theta^+)(\tilde\theta-\theta^+)$ is \emph{linear} in the zero-mean random vector $\tilde\theta-\theta^+$, so its expectation vanishes:
\begin{equation}\label{eq:linearvanishes}
\E\Big[\nabla^2J_i(\theta^+)\big(\tilde\theta-\theta^+\big)\Big]
=\nabla^2J_i(\theta^+)\;\E\big[\tilde\theta-\theta^+\big]
\overset{\eqref{eq:thetatildeunbiased}}{=}0 .
\end{equation}
We may therefore subtract this term inside the expectation without changing its value:
\begin{equation}\label{eq:remainderidentity}
\E\big[\nabla J_i(\tilde\theta)\big]-\nabla J_i(\theta^+)
=\E\Big[\underbrace{\nabla J_i(\tilde\theta)-\nabla J_i(\theta^+)-\nabla^2J_i(\theta^+)\big(\tilde\theta-\theta^+\big)}_{=:\,r(\tilde\theta),\ \text{the second-order Taylor remainder}}\Big].
\end{equation}
This identity is the mechanism behind the proposition. Lemma~\ref{lem:estimator}(ii) bounded the left-hand side of \eqref{eq:remainderidentity} by controlling $\nabla J_i(\tilde\theta)-\nabla J_i(\theta^+)$, a quantity of size $\Theta(\norm{\tilde\theta-\theta^+})$. Identity \eqref{eq:remainderidentity} shows that the part of that difference which is \emph{linear} in $\tilde\theta-\theta^+$ averages out, leaving only the remainder $r(\tilde\theta)$, of size $\mathcal O(\norm{\tilde\theta-\theta^+}^2)$. Since $\norm{\tilde\theta-\theta^+}$ has size $m_{\mathrm{in}}^{-1/2}$, replacing a first power by a second is precisely the promised $\sqrt{m_{\mathrm{in}}}$ gain.

\medskip
\noindent\emph{Step 2: quadratic bound on the remainder, from the Lipschitz Hessian.}
We claim that for all $x,y\in\R^d$,
\begin{equation}\label{eq:taylorquad}
\big\lVert\nabla J_i(y)-\nabla J_i(x)-\nabla^2J_i(x)(y-x)\big\rVert\le\frac{\rJ}{2}\norm{y-x}^2 .
\end{equation}
To see this, fix $x,y$, put $\delta:=y-x$ and consider $\varphi(s):=\nabla J_i(x+s\delta)$ for $s\in[0,1]$. Since $J_i\in C^2(\R^d)$ (Step~0 of the proof of Lemma~\ref{lem:meta-reg}), $\varphi$ is continuously differentiable with $\varphi'(s)=\nabla^2J_i(x+s\delta)\,\delta$, so the fundamental theorem of calculus gives
\[
\nabla J_i(y)-\nabla J_i(x)=\varphi(1)-\varphi(0)=\int_0^1\nabla^2J_i(x+s\delta)\,\delta\,ds .
\]
Writing the subtracted term as $\nabla^2J_i(x)\delta=\int_0^1\nabla^2J_i(x)\,\delta\,ds$ (the integrand is constant in $s$) and combining the two integrals,
\[
\nabla J_i(y)-\nabla J_i(x)-\nabla^2J_i(x)(y-x)
=\int_0^1\Big(\nabla^2J_i(x+s\delta)-\nabla^2J_i(x)\Big)\delta\,ds .
\]
Now bound the integrand. By submultiplicativity, Lemma~\ref{lem:pg-bounds}(iv) applied to the pair $x+s\delta$ and $x$, and $\norm{(x+s\delta)-x}=s\norm{\delta}$,
\[
\Big\lVert\Big(\nabla^2J_i(x+s\delta)-\nabla^2J_i(x)\Big)\delta\Big\rVert
\le\big\lVert\nabla^2J_i(x+s\delta)-\nabla^2J_i(x)\big\rVert\,\norm{\delta}
\;\overset{\text{Lem.~\ref{lem:pg-bounds}(iv)}}{\le}\;\rJ\,s\norm{\delta}\cdot\norm{\delta}
=\rJ\,s\,\norm{\delta}^2 .
\]
Integrating and using $\int_0^1s\,ds=\tfrac12$ yields \eqref{eq:taylorquad}. The factor $\tfrac12$ therefore comes from the integration in $s$, not from any additional estimate.

\medskip
\noindent\emph{Step 3: bounding the remainder in expectation.}
We now bound the right-hand side of \eqref{eq:remainderidentity}. Each step is labeled, since the combination is what caused difficulty:
\begin{align*}
\big\lVert\E\big[\nabla J_i(\tilde\theta)\big]-\nabla J_i(\theta^+)\big\rVert
&\;\overset{\eqref{eq:remainderidentity}}{=}\;\big\lVert\E\big[r(\tilde\theta)\big]\big\rVert\\[2pt]
&\;\overset{(a)}{\le}\;\E\big\lVert r(\tilde\theta)\big\rVert\\[2pt]
&\;\overset{(b)}{\le}\;\frac{\rJ}{2}\,\E\big\lVert\tilde\theta-\theta^+\big\rVert^2\\[2pt]
&\;\overset{(c)}{=}\;\frac{\rJ}{2}\,\alpha^2\,\E\big\lVert\hat g_{\mathrm{in}}-\nabla J_i(\theta)\big\rVert^2\\[2pt]
&\;\overset{(d)}{=}\;\frac{\rJ}{2}\,\alpha^2\,\frac{\varsigma_i^2(\theta)}{m_{\mathrm{in}}} ,
\end{align*}
where:
\begin{itemize}
\item[$(a)$] is Jensen's inequality (F5), $\norm{\E X}\le\E\norm{X}$, applied to the random vector $X=r(\tilde\theta)$;
\item[$(b)$] applies the \emph{pathwise} bound \eqref{eq:taylorquad} with $x=\theta^+$ and $y=\tilde\theta$ (legitimate for every realization of $\mathcal D_{\mathrm{in}}$, because \eqref{eq:taylorquad} holds for \emph{all} $x,y\in\R^d$ and not merely on average), and then takes expectations, monotonicity of $\E$ preserving the inequality;
\item[$(c)$] substitutes the exact displacement identity \eqref{eq:displacement}, $\tilde\theta-\theta^+=\alpha(\hat g_{\mathrm{in}}-\nabla J_i(\theta))$, and pulls the deterministic scalar $\alpha^2$ out of the expectation;
\item[$(d)$] is the variance identity \eqref{eq:innervar} for an i.i.d.\ average of the zero-mean vectors $g_i(\xi;\theta)-\nabla J_i(\theta)$, i.e.\ fact (F1) together with the definition \eqref{eq:sigmadef} of $\varsigma_i^2$.
\end{itemize}
Observe that no square root is taken anywhere, in contrast with Lemma~\ref{lem:estimator}(ii), which invoked (F3) to pass from $\E\norm{\tilde\theta-\theta^+}$ to $(\E\norm{\tilde\theta-\theta^+}^2)^{1/2}$. That single application of Cauchy--Schwarz is what cost the $\sqrt{m_{\mathrm{in}}}$, and Step~1 has made it unnecessary.

\medskip
\noindent\emph{Step 4: conclusion.}
Insert the bound of Step~3 into \eqref{eq:biasfactored2}, using submultiplicativity and $\norm{\mathcal J_i(\theta)}\le1+\alpha\LJ$ from \eqref{eq:specbound}:
\[
\big\lVert\E\widehat{\nabla F}_i(\theta)-\nabla F_i(\theta)\big\rVert
\le\norm{\mathcal J_i(\theta)}\cdot\frac{\rJ\alpha^2\varsigma_i^2(\theta)}{2m_{\mathrm{in}}}
\le\frac{(1+\alpha\LJ)\,\alpha^2\,\rJ\,\varsigma_i^2(\theta)}{2\,m_{\mathrm{in}}}=\bhat_\star ,
\]
and the final inequality of the statement is $\varsigma_i^2(\theta)\le\GJ^2$ from \eqref{eq:sigmadef}, itself a consequence of Lemma~\ref{lem:pg-bounds}(ii). Note that $\bhat_\star$ is stated in terms of $\varsigma_i^2(\theta)$ rather than $\GJ^2$ because the former can be far smaller in practice: it is the true per-trajectory policy-gradient variance at the current iterate, whereas $\GJ^2$ is a worst-case surrogate.

\medskip
\noindent\emph{Step 5: comparison of the two bias bounds.}
Both $\bhat$ and $\bhat_\star$ are valid, so the bias is at most $\min\{\bhat,\bhat_\star\}$. Their ratio is
\[
\frac{\bhat_\star}{\bhat}
=\frac{(1+\alpha\LJ)\alpha^2\rJ\varsigma_i^2/(2m_{\mathrm{in}})}{(1+\alpha\LJ)\alpha\LJ\GJ/\sqrt{m_{\mathrm{in}}}}
=\frac{\alpha\,\rJ\,\varsigma_i^2}{2\,\LJ\,\GJ\,\sqrt{m_{\mathrm{in}}}} ,
\]
so $\bhat_\star<\bhat$ precisely when
\[
m_{\mathrm{in}}>m_0:=\Big(\frac{\alpha\rJ\varsigma_i^2(\theta)}{2\LJ\GJ}\Big)^{2}
\qquad\Big(\le\Big(\frac{\alpha\rJ\GJ}{2\LJ}\Big)^{2}\ \text{using }\varsigma_i^2\le\GJ^2\Big).
\]
The threshold $m_0$ depends only on problem constants and on $\alpha$, never on $\varepsilon$; since Corollary~\ref{cor:rates} drives $m_{\mathrm{in}}\to\infty$ as $\varepsilon\to0$, the sharpened bound governs for all sufficiently small $\varepsilon$, which is the regime the complexity statement concerns. For small $m_{\mathrm{in}}$ the crude bound $\bhat$ is the tighter of the two, and the honest reading of the pair is that the bias decays as $m_{\mathrm{in}}^{-1/2}$ initially and as $m_{\mathrm{in}}^{-1}$ eventually.

\medskip
\noindent\emph{Step 6: the improved complexity.}
The bias condition (iii) of Corollary~\ref{cor:rates} exists only to make the bias term of \eqref{eq:mainbound} at most $\varepsilon/4$, i.e.\ to enforce $\tfrac92\bhat^2\le\varepsilon/4$, equivalently $\bhat\le\sqrt{\varepsilon/18}=\sqrt\varepsilon/(3\sqrt2)$. Substituting $\bhat_\star$ for $\bhat$ and solving for the batch size, this becomes
\begin{equation}\label{eq:iiistar}
\text{(iii$^\star$)}\qquad
m_{\mathrm{in}}\;\ge\;\frac{3}{\sqrt2}\cdot\frac{(1+\alpha\LJ)\,\alpha^2\,\rJ\,\GJ^2}{\sqrt\varepsilon}
\;=\;\Theta\big(\varepsilon^{-1/2}\big),
\end{equation}
in place of the $m_{\mathrm{in}}\ge18(1+\alpha\LJ)^2\alpha^2\LJ^2\GJ^2/\varepsilon=\Theta(\varepsilon^{-1})$ required by (iii). Conditions (i), (ii) and the round count are untouched, since none of them involves $m_{\mathrm{in}}$; hence $K=\mathcal O(\Delta\LF\GF\varepsilon^{-3/2})$ and $\tau=\Theta(\varepsilon^{-1/2})$ are unchanged, and so is the total number of local steps $K\tau=\mathcal O(\varepsilon^{-2})$.

Condition (iv) deserves an explicit check, because it involves $\shat^2$, which \emph{grows} when $m_{\mathrm{in}}$ shrinks. It does not bind: by Lemma~\ref{lem:estimator}(iii),
\[
\shat^2=3\GJ^2\Big[\tfrac{\alpha^2\LJ^2}{m_{\mathrm h}}+\tfrac{(1+\alpha\LJ)^2}{m_{\mathrm{out}}}
+\tfrac{(1+\alpha\LJ)^2\alpha^2\LJ^2}{m_{\mathrm{in}}}\Big]
\le3\GJ^2\Big[\alpha^2\LJ^2+(1+\alpha\LJ)^2+(1+\alpha\LJ)^2\alpha^2\LJ^2\Big]
\]
for every $m_{\mathrm{in}},m_{\mathrm h},m_{\mathrm{out}}\ge1$, so $\shat^2=\Theta(1)$ in $\varepsilon$ irrespective of how $m_{\mathrm{in}}$ is chosen. Condition (iv) therefore reduces, as in the proof of Corollary~\ref{cor:rates}, to the same $\varepsilon$-free requirement $\shat^2\le\tfrac{\sqrt6}{4}n\GF$, and is met under the same mild condition.

Finally, the per-agent trajectory count. Each local step consumes $m_{\mathrm{in}}+m_{\mathrm h}+m_{\mathrm{out}}$ trajectories; with $m_{\mathrm h},m_{\mathrm{out}}=\Theta(1)$ and $m_{\mathrm{in}}=\Theta(\varepsilon^{-1/2})$ from \eqref{eq:iiistar}, this is $\Theta(\varepsilon^{-1/2})$. Multiplying by the $K\tau=\mathcal O(\varepsilon^{-2})$ local steps,
\[
K\tau\big(m_{\mathrm{in}}+m_{\mathrm h}+m_{\mathrm{out}}\big)
=\mathcal O\big(\varepsilon^{-2}\big)\cdot\Theta\big(\varepsilon^{-1/2}\big)
=\mathcal O\big(\varepsilon^{-5/2}\big),
\]
against $\mathcal O(\varepsilon^{-2})\cdot\Theta(\varepsilon^{-1})=\mathcal O(\varepsilon^{-3})$ under (iii). This is the claimed improvement.
\end{proof}

\begin{remark}[The rate $m_{\mathrm{in}}^{-1}$ is not improvable by further expansion]\label{rem:biastight}
One might hope to iterate the device of Step~1 of the proof of Proposition~\ref{prop:sharpbias} and subtract higher-order Taylor terms. This does not remove the obstruction. The next term in the expansion of $\E[\nabla J_i(\tilde\theta)]$ about $\theta^+$ is the quadratic form $\tfrac12\,\E\big[\nabla^3J_i(\theta^+)[\tilde\theta-\theta^+,\tilde\theta-\theta^+]\big] =\tfrac{\alpha^2}{2m_{\mathrm{in}}}\,\nabla^3J_i(\theta^+)\big[\Sigma_i(\theta)\big]$, where $\Sigma_i(\theta)$ is the covariance of $g_i(\xi;\theta)$; being quadratic rather than linear in the zero-mean vector $\tilde\theta-\theta^+$, its expectation does \emph{not} vanish, and it is of order $m_{\mathrm{in}}^{-1}$ exactly. Hence the inner-loop bias is genuinely $\Theta(m_{\mathrm{in}}^{-1})$ and not $o(m_{\mathrm{in}}^{-1})$: Proposition~\ref{prop:sharpbias} is order-optimal among bounds of this type. Removing the bias altogether requires changing the objective, as in the debiased stochastic-adaptation formulation of~\cite{sgmrl}, rather than sharpening the analysis. We state this informally because making it rigorous would require third-order differentiability of $J_i$, which Assumption~\ref{as:policy} does not provide.
\end{remark}

\subsection[Proof of the drift bound]{Proof of Lemma~\ref{lem:drift} (deterministic drift bound)}\label{app:lem4}

\begin{proof}
Fix a round index $k$ throughout and abbreviate the per-agent \emph{displacement from the round's
starting point},
\[
d_i^{\,t}:=\theta_i^{k,t}-\theta^k\in\R^d,
\qquad i\in[n],\quad 0\le t\le\tau,
\]
so that, since $\bar\theta^{k,t}=\frac1n\sum_{i=1}^n\theta_i^{k,t}$ and every agent is initialized at
$\theta_i^{k,0}=\theta^k$,
\begin{equation}\label{eq:driftrecentre}
\bar\theta^{k,t}-\theta^k=\frac1n\sum_{i=1}^n\big(\theta_i^{k,t}-\theta^k\big)=\bar d^{\,t},
\qquad
\theta_i^{k,t}-\bar\theta^{k,t}=d_i^{\,t}-\bar d^{\,t},
\end{equation}
where $\bar d^{\,t}:=\frac1n\sum_{j=1}^nd_j^{\,t}$. Thus client drift is exactly the deviation of the displacements from their own average, and the common offset $\theta^k$ cancels. Note that the lemma makes no probabilistic claim beyond a norm bound: neither unbiasedness (Lemma~\ref{lem:estimator}(ii)) nor independence across agents is used anywhere below, so drift control is entirely decoupled from the stochastic structure of the estimator.

\medskip
\noindent\emph{Step 1: unrolling the local recursion.}
We claim that for every $i\in[n]$ and every $0\le t\le\tau$,
\begin{equation}\label{eq:unrolled}
\theta_i^{k,t}=\theta^k+\beta\sum_{l=0}^{t-1}\widehat{\nabla F}_i\big(\theta_i^{k,l}\big),
\end{equation}
a sum of \emph{exactly $t$ terms}, with the convention that an empty sum is zero. The proof is by induction on $t$. For $t=0$ the sum is empty and \eqref{eq:unrolled} reduces to $\theta_i^{k,0}=\theta^k$, which is the initialization in Algorithm~\ref{alg:main}. Assuming \eqref{eq:unrolled} at some $t\le\tau-1$, the local update gives
\[
\theta_i^{k,t+1}=\theta_i^{k,t}+\beta\,\widehat{\nabla F}_i\big(\theta_i^{k,t}\big)
=\theta^k+\beta\sum_{l=0}^{t-1}\widehat{\nabla F}_i\big(\theta_i^{k,l}\big)+\beta\,\widehat{\nabla F}_i\big(\theta_i^{k,t}\big)
=\theta^k+\beta\sum_{l=0}^{t}\widehat{\nabla F}_i\big(\theta_i^{k,l}\big),
\]
which is \eqref{eq:unrolled} at $t+1$, with $t+1$ terms. The count matters: an off-by-one here would produce $\beta(t+1)\GF$ below and, at $t=0$, the false conclusion that agents have drifted apart before taking any local step.

\medskip
\noindent\emph{Step 2: linear growth of the displacements.}
Every increment in \eqref{eq:unrolled} is bounded by $\GF$: by Lemma~\ref{lem:estimator}(i),
\[
\big\lVert\widehat{\nabla F}_i(\theta')\big\rVert\le\GF
\qquad\text{for every }\theta'\in\R^d\text{ and every realization of the batches.}
\]
The uniformity in $\theta'$ is what permits substituting the \emph{random} iterate $\theta'=\theta_i^{k,l}$: no measurability or conditioning argument is needed, because the bound holds pointwise on the parameter space rather than merely at deterministic points. Hence, by \eqref{eq:unrolled} and the triangle inequality,
\begin{equation}\label{eq:displbound}
\norm{d_i^{\,t}}=\big\lVert\theta_i^{k,t}-\theta^k\big\rVert
\le\beta\sum_{l=0}^{t-1}\big\lVert\widehat{\nabla F}_i\big(\theta_i^{k,l}\big)\big\rVert
\le\beta\,t\,\GF
\qquad\text{for every }i\in[n],\ 0\le t\le\tau .
\end{equation}
Averaging \eqref{eq:displbound} over $i$ and using \eqref{eq:driftrecentre} with Jensen's inequality gives the same bound for the virtual average, $\norm{\bar\theta^{k,t}-\theta^k}=\norm{\bar d^{\,t}}\le\beta t\GF$. Since $\beta\GF$ is the maximum distance a single local step can move an agent, \eqref{eq:displbound} states only that $t$ steps move it at most $t$ times as far; no cancellation is claimed or needed.

\medskip
\noindent\emph{Step 3: statement (i), and a per-agent drift bound.}
Bound \eqref{eq:displbound} is exactly assertion (i). Combining it with \eqref{eq:driftrecentre} and the triangle inequality also gives the per-agent drift bound
\begin{equation}\label{eq:peragentdrift}
\big\lVert\theta_i^{k,t}-\bar\theta^{k,t}\big\rVert=\norm{d_i^{\,t}-\bar d^{\,t}}
\le\norm{d_i^{\,t}}+\norm{\bar d^{\,t}}\le2\beta t\,\GF
\qquad\text{for every }i\in[n],
\end{equation}
which we record for interpretation but do not use: no agent is ever more than $2\beta t\GF$ from the virtual average. In the degenerate case $n=1$ we have $\theta_1^{k,t}=\bar\theta^{k,t}$ and the drift is identically zero, as it must be; the sharper constant $2(1-\tfrac1n)\beta t\GF$ making this visible follows by writing $d_i^{\,t}-\bar d^{\,t}=(1-\tfrac1n)d_i^{\,t}-\tfrac1n\sum_{j\ne i}d_j^{\,t}$ and bounding each group separately.

\medskip
\noindent\emph{Step 4: statement (ii), via the empirical variance identity.}
Squaring \eqref{eq:peragentdrift} termwise would yield $4\beta^2t^2\GF^2$, which over-counts by a factor of four: $d_i^{\,t}-\bar d^{\,t}$ is a deviation from the average of the very vectors being bounded, so the two contributions in \eqref{eq:peragentdrift} cannot both be extremal for all $i$ simultaneously. The exact accounting is the empirical bias--variance identity. Expanding the square and using $\frac1n\sum_i\inner{d_i^{\,t}}{\bar d^{\,t}}=\inner{\bar d^{\,t}}{\bar d^{\,t}}=\norm{\bar d^{\,t}}^2$,
\[
\frac1n\sum_{i=1}^n\norm{d_i^{\,t}-\bar d^{\,t}}^2
=\frac1n\sum_{i=1}^n\Big[\norm{d_i^{\,t}}^2-2\inner{d_i^{\,t}}{\bar d^{\,t}}+\norm{\bar d^{\,t}}^2\Big]
=\frac1n\sum_{i=1}^n\norm{d_i^{\,t}}^2-\norm{\bar d^{\,t}}^2 .
\]
Discarding the non-positive term $-\norm{\bar d^{\,t}}^2$ and inserting \eqref{eq:displbound},
\begin{equation}\label{eq:sharpdrift}
\frac1n\sum_{i=1}^n\big\lVert\theta_i^{k,t}-\bar\theta^{k,t}\big\rVert^2
\;\le\;\frac1n\sum_{i=1}^n\norm{d_i^{\,t}}^2
\;\le\;\beta^2t^2\GF^2 ,
\end{equation}
which is (ii). The bound is tight under the stated hypotheses, so no further improvement is available from this argument: with $n=2$, $d_1^{\,t}=\beta t\GF e$ and $d_2^{\,t}=-\beta t\GF e$ for a unit vector $e$ (realizable when the two agents' meta-gradient estimates point in opposite directions at every local step, each of maximal norm $\GF$), one has $\bar d^{\,t}=0$ and the left-hand side of \eqref{eq:sharpdrift} equals $\beta^2t^2\GF^2$ exactly.

\medskip
\noindent\emph{Step 5: sure validity of the bounds, and its role in the analysis.}
Every inequality above is a deterministic consequence of $\norm{\widehat{\nabla F}_i(\cdot)}\le\GF$, which Lemma~\ref{lem:estimator}(i) establishes for \emph{every} realization of the batches rather than in expectation or outside a null set. Hence \eqref{eq:peragentdrift} and \eqref{eq:sharpdrift} hold surely, and in particular they may be inserted inside a conditional expectation, as is done in \eqref{eq:e-bound} in the proof of Lemma~\ref{lem:onestep}, without any further argument.

This is the step at which the present analysis and the supervised one diverge. In Per-FedAvg~\cite{perfedavg} the analogous drift bound requires both a uniform gradient bound and a gradient-dissimilarity bound, because only the \emph{mean} local direction is controlled there and the deviation of an individual agent's realized direction from that mean must be accounted for separately. Here Lemma~\ref{lem:estimator}(i) controls every realized direction, so \eqref{eq:displbound} follows from the triangle inequality alone and no heterogeneity quantity ever appears. Should one wish to dispense with the pathwise bound, a second-moment version remains available at the cost of a worse constant: by (F4), \eqref{eq:unrolled} gives $\E\norm{d_i^{\,t}}^2\le t\beta^2\sum_{l<t}\E\lVert\widehat{\nabla F}_i(\theta_i^{k,l})\rVert^2 \le\beta^2t^2\big(2\GF^2+2\shat^2\big)$, using $\E\lVert\widehat{\nabla F}_i\rVert^2\le2\norm{\nabla F_i}^2+2\E\lVert\widehat{\nabla F}_i-\nabla F_i\rVert^2 \le2\GF^2+2\shat^2$ from Lemma~\ref{lem:meta-reg}(ii) and Lemma~\ref{lem:estimator}(iii); the analysis then closes with $\GF^2$ replaced by $2\GF^2+2\shat^2$ throughout.

\medskip
\noindent\emph{Step 6: interpretation.}
Bound \eqref{eq:sharpdrift} is the formal content of the statement that local computation is free only up to a point. Drift enters \eqref{eq:mainbound} through $\beta^2t^2$ with $t<\tau$, so it is governed by the product $\beta\tau$ and not by $\beta$ and $\tau$ separately: doubling the number of local steps while halving the outer step size leaves the drift term unchanged. It also vanishes at $t=0$, i.e.\ immediately after each synchronization, and grows monotonically until the next one, which is why the analysis telescopes over $(k,t)$ jointly rather than treating rounds in isolation.
\end{proof}

\subsection[Proof of the one-step progress lemma]{Proof of Lemma~\ref{lem:onestep} (one-step progress)}\label{app:lem5}

\begin{proof}
Fix a round $k$ and a local step $0\le t\le\tau-1$, and abbreviate
\[
\bar\theta:=\bar\theta^{k,t},
\qquad
\theta_i:=\theta_i^{k,t},
\qquad
a:=\nabla F(\bar\theta),
\qquad
P:=\frac1n\sum_{i=1}^n\widehat{\nabla F}_i(\theta_i),
\]
so that the averaged iterate advances as
\begin{equation}\label{eq:avgupdate}
\bar\theta^{k,t+1}
=\frac1n\sum_{i=1}^n\theta_i^{k,t+1}
=\frac1n\sum_{i=1}^n\Big(\theta_i+\beta\,\widehat{\nabla F}_i(\theta_i)\Big)
=\bar\theta+\beta P .
\end{equation}
Identity \eqref{eq:avgupdate} is what makes the virtual sequence tractable: although no agent computes $P$ and the server never materializes $\bar\theta^{k,t}$ for $0<t<\tau$, the average \emph{behaves} as if a single ascent step of size $\beta$ had been taken along the averaged direction $P$.

\medskip
\noindent\emph{Step 0: measurability of the iterates and independence of the batches.}
Recall that $\mathcal F^{k,t}$ is generated by all randomness up to and including the computation of $\{\theta_i^{k,t}\}_{i\in[n]}$. Hence $\bar\theta$, $\theta_1,\dots,\theta_n$ and $a$ are $\mathcal F^{k,t}$-measurable, i.e.\ deterministic under $\E_{k,t}[\cdot]=\E[\cdot\mid\mathcal F^{k,t}]$, whereas the batches $\big(\mathcal D_{\mathrm{in}}^{i},\mathcal D_{\mathrm h}^{i},\mathcal D_{\mathrm{out}}^{i}\big)_{i\in[n]}$ drawn at local step $t$ are fresh: they are independent of $\mathcal F^{k,t}$ and, by the sampling protocol of Algorithm~\ref{alg:main}, independent \emph{across} agents. Two consequences are used below and are worth isolating, since the whole proof rests on them.

\emph{(M1) Conditional application of Lemma~\ref{lem:estimator}.} For each $i$, the estimator $\widehat{\nabla F}_i(\theta_i)$ is formed at the $\mathcal F^{k,t}$-measurable point $\theta_i$ from batches independent of $\mathcal F^{k,t}$. Lemma~\ref{lem:estimator} was proved conditionally on a fixed evaluation point, so it applies verbatim under $\E_{k,t}$: writing
\begin{equation}\label{eq:bi}
b_i:=\E_{k,t}\big[\widehat{\nabla F}_i(\theta_i)\big]-\nabla F_i(\theta_i),
\qquad\text{Lemma~\ref{lem:estimator}(ii) gives }\norm{b_i}\le\bhat ,
\end{equation}
and Lemma~\ref{lem:estimator}(iii) gives $\E_{k,t}\lVert\widehat{\nabla F}_i(\theta_i)-\E_{k,t}\widehat{\nabla F}_i(\theta_i)\rVert^2\le\shat^2$, for every $i$. Both bounds are uniform in the evaluation point, which is essential: $\theta_i$ is random, so a bound holding only at deterministic points would be unusable here.

\emph{(M2) Conditional independence across agents.} Given $\mathcal F^{k,t}$, the random vectors $\widehat{\nabla F}_1(\theta_1),\dots,\widehat{\nabla F}_n(\theta_n)$ are mutually independent, since each is a function of that agent's own fresh batches alone and the evaluation points $\theta_i$ have been fixed by the conditioning. Note that the $\theta_i$ are \emph{not} independent of one another (they share the history $\theta^k$), which is precisely why the conditioning is needed before independence can be invoked.

\medskip
\noindent\emph{Step 1: the smoothness inequality.}
By Lemma~\ref{lem:meta-reg}(iii) the map $\nabla F$ is $\LF$-Lipschitz, so the ascent form \eqref{eq:descentlemma} of the descent lemma, applied with the point $\bar\theta$, direction $P$ and step $\beta$, gives the pathwise inequality
\begin{equation}\label{eq:ascent}
F\big(\bar\theta+\beta P\big)\ge F(\bar\theta)+\beta\inner{a}{P}-\frac{\LF\beta^2}{2}\norm{P}^2 .
\end{equation}
This holds for every realization of the batches, so we may take $\E_{k,t}$ of both sides; by \eqref{eq:avgupdate} the left-hand side is $F(\bar\theta^{k,t+1})$. It remains to bound $\E_{k,t}\inner{a}{P}$ from below and $\E_{k,t}\norm{P}^2$ from above, which are Steps~3 and~4.

\medskip
\noindent\emph{Step 2: the conditional mean of $P$, and its two error terms.}
Averaging \eqref{eq:bi} over $i$ and inserting $\nabla F(\bar\theta)=\frac1n\sum_i\nabla F_i(\bar\theta)$,
\begin{equation}\label{eq:meanP}
\E_{k,t}P=\frac1n\sum_{i=1}^n\Big(\nabla F_i(\theta_i)+b_i\Big)=a+e+\bar b,
\qquad
e:=\frac1n\sum_{i=1}^n\Big(\nabla F_i(\theta_i)-\nabla F_i(\bar\theta)\Big),
\quad
\bar b:=\frac1n\sum_{i=1}^nb_i .
\end{equation}
The decomposition \eqref{eq:meanP} isolates the two ways in which the averaged direction fails to be the meta-gradient at the averaged point. The term $e$ is a \emph{drift} error: each agent evaluates its own meta-gradient at its own iterate $\theta_i$ rather than at $\bar\theta$. The term $\bar b$ is a \emph{bias} error, inherited from inner-loop sampling. They are controlled by different mechanisms and behave differently: $e$ vanishes at $t=0$ and can be made small by shrinking $\beta\tau$, whereas $\bar b$ is invariant to $(\beta,\tau,K,n)$ and yields only to a larger inner batch.

For the drift error, Jensen's inequality (F4 with the uniform average, or equivalently convexity of $\norm{\cdot}^2$), the $\LF$-Lipschitz continuity of each $\nabla F_i$ from Lemma~\ref{lem:meta-reg}(iii), and the sharpened drift bound \eqref{eq:sharpdrift} give
\begin{equation}\label{eq:e-bound}
\norm{e}^2
\;\overset{\text{(Jensen)}}{\le}\;\frac1n\sum_{i=1}^n\norm{\nabla F_i(\theta_i)-\nabla F_i(\bar\theta)}^2
\;\overset{\text{Lem.~\ref{lem:meta-reg}(iii)}}{\le}\;\LF^2\cdot\frac1n\sum_{i=1}^n\norm{\theta_i-\bar\theta}^2
\;\overset{\eqref{eq:sharpdrift}}{\le}\;\LF^2\beta^2t^2\GF^2 .
\end{equation}
Bound \eqref{eq:e-bound} is where the sharp form of Lemma~\ref{lem:drift} is used: the drift enters only through the \emph{averaged squared} deviation $\frac1n\sum_i\norm{\theta_i-\bar\theta}^2$, exactly the quantity that Lemma~\ref{lem:drift}(ii) bounds by $\beta^2t^2\GF^2$ rather than by the four-times-larger square of the per-agent bound. For the bias error, the triangle inequality and \eqref{eq:bi} give
\begin{equation}\label{eq:barb-bound}
\norm{\bar b}\le\frac1n\sum_{i=1}^n\norm{b_i}\le\bhat .
\end{equation}
Both \eqref{eq:e-bound} and \eqref{eq:barb-bound} hold surely, since \eqref{eq:sharpdrift} does.

\medskip
\noindent\emph{Step 3: lower bound on the inner product.}
Using \eqref{eq:meanP} and that $a$ is $\mathcal F^{k,t}$-measurable,
\[
\E_{k,t}\inner{a}{P}=\inner{a}{\E_{k,t}P}=\inner{a}{a+e+\bar b}=\norm{a}^2+\inner{a}{e+\bar b}.
\]
The cross term is the only place where the errors could destroy progress, and it is handled by Young's inequality. For any $u,v$ and any $\lambda>0$, $|\inner{u}{v}|\le\frac{\lambda}{2}\norm{u}^2+\frac{1}{2\lambda}\norm{v}^2$; taking $\lambda=\tfrac12$, $u=a$, $v=e+\bar b$ gives
\begin{equation}\label{eq:young}
\inner{a}{e+\bar b}\ge-\big|\inner{a}{e+\bar b}\big|\ge-\tfrac14\norm{a}^2-\norm{e+\bar b}^2 .
\end{equation}
The choice $\lambda=\tfrac12$ is deliberate: it forfeits only a quarter of the ascent term $\norm{a}^2$, retaining enough to absorb the second-moment cost incurred in Step~4 and still retain a positive multiple of $\norm{a}^2$. Combining \eqref{eq:young} with $\norm{e+\bar b}^2\le2\norm{e}^2+2\norm{\bar b}^2$ (fact (F4) with $k=2$) and
\eqref{eq:barb-bound},
\begin{equation}\label{eq:inner}
\E_{k,t}\inner{a}{P}\ \ge\ \tfrac34\norm{a}^2-2\norm{e}^2-2\bhat^2 .
\end{equation}

\medskip
\noindent\emph{Step 4: upper bound on the second moment, and the dependence on $n$.}
Decompose $P$ into its conditional mean and fluctuation and apply the bias--variance identity (F6) conditionally:
\begin{equation}\label{eq:secmomsplit}
\E_{k,t}\norm{P}^2=\norm{\E_{k,t}P}^2+\E_{k,t}\norm{P-\E_{k,t}P}^2 .
\end{equation}
For the fluctuation, write $X_i:=\widehat{\nabla F}_i(\theta_i)-\E_{k,t}\widehat{\nabla F}_i(\theta_i)$, so that $P-\E_{k,t}P=\frac1n\sum_iX_i$ with $\E_{k,t}X_i=0$. By (M2) the $X_i$ are conditionally independent, hence conditionally uncorrelated, so the cross terms vanish:
\[
\E_{k,t}\Big\lVert\frac1n\sum_{i=1}^nX_i\Big\rVert^2
=\frac{1}{n^2}\sum_{i=1}^n\sum_{j=1}^n\E_{k,t}\inner{X_i}{X_j}
=\frac{1}{n^2}\sum_{i=1}^n\E_{k,t}\norm{X_i}^2
\;\overset{\text{(M1)}}{\le}\;\frac{1}{n^2}\cdot n\,\shat^2=\frac{\shat^2}{n} .
\]
\begin{equation}\label{eq:varP}
\text{That is,}\qquad \E_{k,t}\norm{P-\E_{k,t}P}^2\le\frac{\shat^2}{n} .
\end{equation}
This is the \emph{only} step in the entire analysis at which the population size $n$ improves a bound, and it does so for the familiar reason that averaging $n$ independent zero-mean vectors divides the variance by $n$. Note carefully that no such cancellation is available for the drift term $e$ or the bias term $\bar b$: the $b_i$ are not zero-mean and need not cancel across agents (hence \eqref{eq:barb-bound} carries no $1/n$), and the drift deviations are deterministic given $\mathcal F^{k,t}$. This asymmetry is exactly what Remark~\ref{rem:compare} refers to in declining to claim a linear speedup at the $\tau=\Theta(\varepsilon^{-1/2})$ frontier.

For the mean part of \eqref{eq:secmomsplit}, \eqref{eq:meanP}, fact (F4) with $k=3$ and \eqref{eq:barb-bound} give $\norm{\E_{k,t}P}^2=\norm{a+e+\bar b}^2\le3\norm{a}^2+3\norm{e}^2+3\bhat^2$. Substituting this and \eqref{eq:varP} into \eqref{eq:secmomsplit},
\begin{equation}\label{eq:secmom}
\E_{k,t}\norm{P}^2\ \le\ 3\norm{a}^2+3\norm{e}^2+3\bhat^2+\frac{\shat^2}{n} .
\end{equation}

\medskip
\noindent\emph{Step 5: assembly, and the role of the step-size condition.} Take $\E_{k,t}$ in \eqref{eq:ascent} and insert \eqref{eq:inner} and \eqref{eq:secmom}. Collecting the coefficient of $\norm{a}^2$ and of $\norm{e}^2+\bhat^2$ separately,
\begin{align}
\E_{k,t}F\big(\bar\theta^{k,t+1}\big)
&\ \ge\ F(\bar\theta)+\beta\Big(\tfrac34\norm{a}^2-2\norm{e}^2-2\bhat^2\Big)
-\frac{\LF\beta^2}{2}\Big(3\norm{a}^2+3\norm{e}^2+3\bhat^2+\frac{\shat^2}{n}\Big)\notag\\
&\ =\ F(\bar\theta)
+\beta\underbrace{\Big(\tfrac34-\tfrac{3\LF\beta}{2}\Big)}_{=:c_1}\norm{a}^2
-\beta\underbrace{\Big(2+\tfrac{3\LF\beta}{2}\Big)}_{=:c_2}\big(\norm{e}^2+\bhat^2\big)
-\frac{\LF\beta^2\shat^2}{2n} .\label{eq:assembled}
\end{align}
The step-size condition $\beta\le\frac{1}{6\LF}$ enters here and nowhere else in this proof. It gives $\LF\beta\le\tfrac16$, hence $\tfrac{3\LF\beta}{2}\le\tfrac14$, and therefore
\[
c_1=\tfrac34-\tfrac{3\LF\beta}{2}\ \ge\ \tfrac34-\tfrac14=\tfrac12>0,
\qquad
c_2=2+\tfrac{3\LF\beta}{2}\ \le\ 2+\tfrac14=\tfrac94 .
\]
The first inequality is the substantive one: it guarantees that the ascent coefficient $c_1$ remains bounded away from zero, so that each local step makes progress proportional to $\norm{\nabla F(\bar\theta^{k,t})}^2$; without an upper bound on $\beta$ the curvature penalty $\tfrac{3\LF\beta}{2}$ could exceed $\tfrac34$ and $c_1$ could turn negative, at which point \eqref{eq:assembled} would assert nothing. The constant $6$ is not canonical (any $\beta\le c/\LF$ with $c<\tfrac12$ yields $c_1>0$) and is chosen because it yields the simple values $c_1\ge\tfrac12$ and $c_2\le\tfrac94$.

Substituting $c_1\ge\tfrac12$, $c_2\le\tfrac94$ into \eqref{eq:assembled}, and then
\eqref{eq:e-bound} for $\norm{e}^2$,
\[
\E_{k,t}F\big(\bar\theta^{k,t+1}\big)
\ \ge\ F\big(\bar\theta^{k,t}\big)+\frac{\beta}{2}\norm{\nabla F(\bar\theta^{k,t})}^2
-\frac{9\beta}{4}\Big(\LF^2\beta^2t^2\GF^2+\bhat^2\Big)-\frac{\LF\beta^2\shat^2}{2n},
\]
which is the assertion of the lemma. Note that both substitutions go in the correct direction, since $\norm{e}^2+\bhat^2\ge0$ and the term carries a minus sign.

\medskip
\noindent\emph{Step 6: interpretation of the three error terms.}
The inequality just proved shows that a single local step gains $\tfrac{\beta}{2}\norm{\nabla F}^2$ and loses three quantities, each traceable to a distinct mechanism and each with a distinct remedy.
\begin{itemize}
\item \emph{Drift}, $\tfrac94\beta\LF^2\beta^2t^2\GF^2$: quadratic in the local-step index $t$, so it is zero immediately after synchronization and grows until the next one. It is governed by the product $\beta t$, hence by $\beta\tau$, and is the term that forbids taking $\tau$ arbitrarily large.
\item \emph{Bias}, $\tfrac94\beta\bhat^2$: independent of $t$, of $\beta$ beyond the common factor, of $K$ and of $n$. Only $m_{\mathrm{in}}$ reduces it, which is why it becomes the bias floor of Corollary~\ref{cor:fo} when curvature is dropped and $\bhat$ ceases to decay in $m_{\mathrm{in}}$ at all.
\item \emph{Variance}, $\tfrac{\LF\beta^2\shat^2}{2n}$: the only term carrying $1/n$, by Step~4, and the only one quadratic in $\beta$, so it is suppressed by taking $\beta$ small.
\end{itemize}
Dividing by $\beta$, the three losses scale as $\beta^2t^2$, $1$ and $\beta$ respectively; the competition between the first two, together with the optimization term recovered by telescoping in Theorem~\ref{thm:main}, is what fixes the frontier $\tau=\Theta(\varepsilon^{-1/2})$, $K=\Theta(\varepsilon^{-3/2})$.
\end{proof}

\subsection[Proof of the main theorem]{Proof of Theorem~\ref{thm:main} (convergence of \textsc{Per-FedAvg-PG})}\label{app:thm}

\begin{proof}
Write $\mathcal E_{k,t}:=\E\norm{\nabla F(\bar\theta^{k,t})}^2$ for the quantity to be bounded, where $\E$ is the total expectation over all randomness of the run. Throughout, $\alpha>0$ and $\beta\le\frac1{6\LF}$ are fixed, so Lemma~\ref{lem:onestep} is available at every $(k,t)$.

\medskip
\noindent\emph{Step 0: the virtual sequence is a single trajectory.}
The sum in \eqref{eq:mainbound} runs over the doubly indexed family $\{\bar\theta^{k,t}\}_{0\le k<K,\,0\le t<\tau}$, and the telescoping below requires that consecutive rounds be glued together. We verify this. By definition $\bar\theta^{k,t}=\frac1n\sum_i\theta_i^{k,t}$; at the end of round $k$ the server aggregates, and at the start of round $k+1$ it broadcasts, so
\begin{equation}\label{eq:glue}
\bar\theta^{k,\tau}=\frac1n\sum_{i=1}^n\theta_i^{k,\tau}\overset{\text{(aggregation)}}{=}\theta^{k+1} \overset{\text{(re-initialization)}}{=}\frac1n\sum_{i=1}^n\theta_i^{k+1,0}=\bar\theta^{k+1,0},
\end{equation}
the middle two equalities being lines~10 and~4 of Algorithm~\ref{alg:main}: every agent is re-initialized at the common point $\theta^{k+1}$, so the average of the re-initialized parameters is $\theta^{k+1}$ itself. Consequently the concatenated sequence
\begin{equation}\label{eq:concat}
\bar\theta^{0,0},\ \bar\theta^{0,1},\ \dots,\ \bar\theta^{0,\tau-1},\ \bar\theta^{1,0},\ \bar\theta^{1,1},\ \dots,\ \bar\theta^{K-1,\tau-1},\ \bar\theta^{K-1,\tau}
\end{equation}
is a single chain of $K\tau+1$ points in which each element is obtained from its predecessor by one application of \eqref{eq:avgupdate}, with $\bar\theta^{0,0}=\theta^0$ and $\bar\theta^{K-1,\tau}=\theta^{K}$. Note that \eqref{eq:glue} is a statement about the \emph{virtual} average and not about anything the server holds: the intermediate points $\bar\theta^{k,t}$ with $0<t<\tau$ are never formed by any party, since the agents are apart during a round. What matters for the analysis is only that they exist as a mathematical object and that \eqref{eq:concat} has no gaps at the round boundaries; were the aggregation rule anything other than the plain average (a weighted average with round-dependent weights, say), \eqref{eq:glue} would fail and the potential terms would not cancel.

\medskip
\noindent\emph{Step 1: from conditional to unconditional one-step progress.}
Fix $(k,t)$ with $0\le k<K$ and $0\le t<\tau$. Lemma~\ref{lem:onestep} asserts the conditional inequality
\[
\E_{k,t}F\big(\bar\theta^{k,t+1}\big)\ \ge\ F\big(\bar\theta^{k,t}\big)+\frac{\beta}{2}\norm{\nabla F(\bar\theta^{k,t})}^2
-\frac{9\beta}{4}\Big(\LF^2\beta^2t^2\GF^2+\bhat^2\Big)-\frac{\LF\beta^2\shat^2}{2n}.
\]
Both sides are $\mathcal F^{k,t}$-measurable random variables, and the right-hand side is integrable, being bounded: $|F|\le\bar R$ by Lemma~\ref{lem:meta-reg}(i) and $\norm{\nabla F}\le\GF$ by Lemma~\ref{lem:meta-reg}(ii), so all expectations below are finite and no integrability caveat is needed. Taking total expectations and using the tower property
$\E\big[\E_{k,t}F(\bar\theta^{k,t+1})\big]=\E F(\bar\theta^{k,t+1})$,
\[
\E F\big(\bar\theta^{k,t+1}\big)\ \ge\ \E F\big(\bar\theta^{k,t}\big)+\frac{\beta}{2}\mathcal E_{k,t}
-\frac{9\beta}{4}\LF^2\beta^2t^2\GF^2-\frac{9\beta}{4}\bhat^2-\frac{\LF\beta^2\shat^2}{2n}.
\]
Rearranging so that the quantity of interest stands alone on the left,
\begin{equation}\label{eq:onestep-rearranged}
\frac{\beta}{2}\,\mathcal E_{k,t}
\ \le\ \underbrace{\E F\big(\bar\theta^{k,t+1}\big)-\E F\big(\bar\theta^{k,t}\big)}_{\text{potential increment}}
\ +\ \underbrace{\tfrac94\LF^2\beta^3\GF^2\,t^2}_{\text{drift}}
\ +\ \underbrace{\tfrac94\beta\,\bhat^2}_{\text{bias}}
\ +\ \underbrace{\tfrac{\LF\beta^2\shat^2}{2n}}_{\text{variance}} ,
\end{equation}
where $\tfrac94\beta\cdot\LF^2\beta^2\GF^2t^2=\tfrac94\LF^2\beta^3\GF^2t^2$. Only the first term on the right depends on $(k,t)$ through the iterates; the remaining three are deterministic, and only the drift term depends on the indices at all, through $t^2$.

\medskip
\noindent\emph{Step 2: summation and telescoping.}
Sum \eqref{eq:onestep-rearranged} over $t=0,\dots,\tau-1$ and $k=0,\dots,K-1$, a total of $K\tau$ inequalities:
\begin{equation}\label{eq:summed}
\frac{\beta}{2}\sum_{k=0}^{K-1}\sum_{t=0}^{\tau-1}\mathcal E_{k,t}
\ \le\ \underbrace{\sum_{k=0}^{K-1}\sum_{t=0}^{\tau-1}\Big(\E F\big(\bar\theta^{k,t+1}\big)-\E F\big(\bar\theta^{k,t}\big)\Big)}_{=:\,\Sigma_{\mathrm{pot}}}
\ +\ \tfrac94\LF^2\beta^3\GF^2\underbrace{\sum_{k=0}^{K-1}\sum_{t=0}^{\tau-1}t^2}_{=:\,\Sigma_{t^2}}
\ +\ K\tau\Big(\tfrac94\beta\bhat^2+\tfrac{\LF\beta^2\shat^2}{2n}\Big),
\end{equation}
the last group because the bias and variance terms are constant in $(k,t)$ and there are $K\tau$ summands. We evaluate $\Sigma_{\mathrm{pot}}$ and $\Sigma_{t^2}$ in turn.

\emph{The potential sum.} By Step~0 the increments in $\Sigma_{\mathrm{pot}}$ are consecutive differences along the single chain \eqref{eq:concat}. Explicitly, the inner sum over $t$ telescopes within round $k$ to
\[
\sum_{t=0}^{\tau-1}\Big(\E F\big(\bar\theta^{k,t+1}\big)-\E F\big(\bar\theta^{k,t}\big)\Big)
=\E F\big(\bar\theta^{k,\tau}\big)-\E F\big(\bar\theta^{k,0}\big),
\]
and summing this over $k$ telescopes again, because \eqref{eq:glue} identifies the terminal point of round $k$ with the initial point of round $k+1$, so that $\E F(\bar\theta^{k,\tau})=\E F(\bar\theta^{k+1,0})$ and consecutive contributions cancel:
\begin{equation}\label{eq:telescope}
\Sigma_{\mathrm{pot}}
=\sum_{k=0}^{K-1}\Big(\E F\big(\bar\theta^{k,\tau}\big)-\E F\big(\bar\theta^{k,0}\big)\Big)
\overset{\eqref{eq:glue}}{=}\E F\big(\bar\theta^{K-1,\tau}\big)-F\big(\bar\theta^{0,0}\big)
=\E F\big(\theta^{K}\big)-F\big(\theta^{0}\big).
\end{equation}
Thus the entire accumulated potential collapses to the difference between the endpoints of the run, regardless of $K$ and $\tau$; had the telescoping been available only \emph{within} rounds, the right-hand side would instead have carried $K$ separate increments and the bound would have degraded by a factor of $K$. Finally, since $F^\star=\sup_\theta F(\theta)$ dominates $F(\theta^K)$ pointwise, hence in expectation,
\begin{equation}\label{eq:potbound}
\Sigma_{\mathrm{pot}}\ \le\ F^\star-F\big(\theta^0\big)=\Delta,
\end{equation}
and $\Delta$ is finite by Lemma~\ref{lem:meta-reg}(i); indeed $\Delta\le2\bar R$ by \eqref{eq:Deltabound}. This is the sole point at which boundedness of the meta-objective is used, and it is what removes any need to assume a finite optimality gap.

\emph{The drift sum.} The summand $t^2$ does not depend on $k$, so
\[
\Sigma_{t^2}=K\sum_{t=0}^{\tau-1}t^2=K\cdot\frac{(\tau-1)\tau(2\tau-1)}{6}\ \le\ K\,\frac{\tau^3}{3},
\]
using the closed form $\sum_{t=0}^{\tau-1}t^2=\frac{(\tau-1)\tau(2\tau-1)}{6}$ and the crude but sufficient estimate $(\tau-1)\tau(2\tau-1)\le\tau\cdot\tau\cdot2\tau=2\tau^3$. (Alternatively, $\sum_{t=0}^{\tau-1}t^2\le\int_0^{\tau}s^2\,ds=\tau^3/3$ by monotonicity of $s\mapsto s^2$.) The cubic growth in $\tau$, against the linear growth of the $K\tau$ prefactor on the left, is what will make the drift term scale as $\tau^2$ after normalization, and hence what caps the number of useful local steps.

\medskip
\noindent\emph{Step 3: normalizing.}
Insert \eqref{eq:potbound} and the bound on $\Sigma_{t^2}$ into \eqref{eq:summed} and multiply both sides by $\frac{2}{\beta K\tau}>0$, which preserves the inequality:
\[
\frac{1}{K\tau}\sum_{k=0}^{K-1}\sum_{t=0}^{\tau-1}\mathcal E_{k,t}
\ \le\ \frac{2\Delta}{\beta K\tau}
\ +\ \frac{2}{\beta K\tau}\cdot\tfrac94\LF^2\beta^3\GF^2\cdot\frac{K\tau^3}{3}
\ +\ \frac{2}{\beta K\tau}\cdot K\tau\Big(\tfrac94\beta\bhat^2+\tfrac{\LF\beta^2\shat^2}{2n}\Big).
\]
The three terms simplify as follows. For the drift term,
\[
\frac{2}{\beta K\tau}\cdot\frac{9\LF^2\beta^3\GF^2}{4}\cdot\frac{K\tau^3}{3}
=\frac{2\cdot9}{4\cdot3}\,\LF^2\GF^2\,\beta^{3-1}\,\tau^{3-1}
=\frac{3}{2}\,\LF^2\GF^2\,\beta^2\tau^2 .
\]
For the bias term, $\frac{2}{\beta K\tau}\cdot K\tau\cdot\tfrac94\beta\bhat^2=\tfrac92\bhat^2$, the factors $K\tau$ and $\beta$ cancelling entirely, which is the algebraic reason the bias survives as a floor immune to $K$, $\tau$ and $\beta$. For the variance term,
$\frac{2}{\beta K\tau}\cdot K\tau\cdot\frac{\LF\beta^2\shat^2}{2n}=\frac{\LF\beta\shat^2}{n}$. Collecting,
\[
\frac{1}{K\tau}\sum_{k=0}^{K-1}\sum_{t=0}^{\tau-1}\E\norm{\nabla F(\bar\theta^{k,t})}^2
\ \le\ \frac{2\Delta}{\beta K\tau}+\frac32\,\LF^2\GF^2\,\beta^2\tau^2+\frac92\,\bhat^2+\frac{\LF\beta\shat^2}{n},
\]
which is \eqref{eq:mainbound}.

\medskip
\noindent\emph{Step 4: the randomized output.}
Let $\hat\theta:=\bar\theta^{\kappa,\iota}$ where the pair $(\kappa,\iota)$ is drawn uniformly from $\{0,\dots,K-1\}\times\{0,\dots,\tau-1\}$, independently of everything else. Conditioning on $(\kappa,\iota)$ and using that each of the $K\tau$ values is taken with probability $1/(K\tau)$,
\[
\E\norm{\nabla F(\hat\theta)}^2
=\sum_{k=0}^{K-1}\sum_{t=0}^{\tau-1}\Pr\big[(\kappa,\iota)=(k,t)\big]\,\E\norm{\nabla F(\bar\theta^{k,t})}^2
=\frac{1}{K\tau}\sum_{k=0}^{K-1}\sum_{t=0}^{\tau-1}\mathcal E_{k,t},
\]
so the bound just established applies to $\E\norm{\nabla F(\hat\theta)}^2$ verbatim. Hence
\eqref{eq:mainbound} is not merely a statement about an average over the run but certifies a single returned point in the sense of Definition~\ref{def:fosp}. Two practical remarks follow. First, the index pair may be drawn \emph{in advance}, since the bound does not require $(\kappa,\iota)$ to be chosen adaptively; the server then needs one extra aggregation, at the designated local step, to materialize $\hat\theta$. Second, with $\tau=1$ every $\bar\theta^{k,0}=\theta^k$ is already held by the server, so no extra communication is required at all.

\medskip
\noindent\emph{Step 5: scope of the guarantee.}
Since $F$ is in general nonconcave, \eqref{eq:mainbound} certifies approximate first-order stationarity and nothing stronger; it does not locate a global maximizer, nor exclude saddle points of $F$. Note also that all four terms are non-negative and that the bound is monotone in the natural directions: it improves as $K$ grows, as $n$ grows, and as $m_{\mathrm{in}}$ grows through $\bhat$, but is \emph{not} monotone in $\tau$ or $\beta$, each of which appears both in a numerator and in a denominator. That non-monotonicity is precisely the trade-off Corollary~\ref{cor:rates} resolves: raising $\tau$ reduces the optimization term $2\Delta/(\beta K\tau)$ but inflates the drift term $\tfrac32\LF^2\GF^2\beta^2\tau^2$, and the two are balanced at $\beta\tau=\Theta(\sqrt\varepsilon/(\LF\GF))$.
\end{proof}

\subsection[Proof of the rate corollary]{Proof of Corollary~\ref{cor:rates} (rates)}\label{app:cor1}

\begin{proof}
Write $\mathcal E:=\frac{1}{K\tau}\sum_{k,t}\mathcal E_{k,t} =\frac{1}{K\tau}\sum_{k<K}\sum_{t<\tau}\E\norm{\nabla F(\bar\theta^{k,t})}^2$ for the left-hand side of \eqref{eq:mainbound}. The proof has two parts: first, that conditions (i)--(iv) together with the stated round count force $\mathcal E\le\varepsilon$; second, that the particular instantiation of $(\beta,\tau,m_{\mathrm{in}},m_{\mathrm h},m_{\mathrm{out}})$ given in the statement satisfies those conditions and yields the advertised complexities.

\medskip
\noindent\emph{Step 1: allocation of the error budget.}
Condition (i), $\beta\le\frac1{6\LF}$, is the hypothesis of Theorem~\ref{thm:main}, so \eqref{eq:mainbound} is available:
\[
\mathcal E\ \le\ \underbrace{\frac{2\Delta}{\beta K\tau}}_{(\mathrm A)\ \text{optimization}}
+\underbrace{\tfrac32\LF^2\GF^2\beta^2\tau^2}_{(\mathrm B)\ \text{drift}}
+\underbrace{\tfrac92\bhat^2}_{(\mathrm C)\ \text{bias}}
+\underbrace{\tfrac{\LF\beta\shat^2}{n}}_{(\mathrm D)\ \text{variance}} .
\]
We show that each of the four terms is at most $\varepsilon/4$, whence $\mathcal E\le\varepsilon$. Allocating the budget in four equal parts is a choice, not a necessity (any allocation summing to $\varepsilon$ would do, and an unequal one would trade the constants in (ii)--(iv) against one another), but it is the choice that keeps every constant explicit and none of the four requirements artificially severe.

\emph{Term (B), from condition (ii).} Squaring $\beta\tau\le\frac{\sqrt\varepsilon}{\sqrt6\LF\GF}$
gives $\beta^2\tau^2\le\frac{\varepsilon}{6\LF^2\GF^2}$, hence
\[
\tfrac32\LF^2\GF^2\beta^2\tau^2\ \le\ \tfrac32\LF^2\GF^2\cdot\frac{\varepsilon}{6\LF^2\GF^2}
=\frac{\varepsilon}{4} .
\]
Note that (ii) constrains only the \emph{product} $\beta\tau$, consistent with the observation after \eqref{eq:sharpdrift} that drift depends on $\beta$ and $\tau$ only through their product.

\emph{Term (C), from condition (iii).} By Lemma~\ref{lem:estimator}(ii),
$\bhat^2=(1+\alpha\LJ)^2\alpha^2\LJ^2\GJ^2/m_{\mathrm{in}}$, so
$m_{\mathrm{in}}\ge18(1+\alpha\LJ)^2\alpha^2\LJ^2\GJ^2/\varepsilon$ gives
\[
\tfrac92\bhat^2=\frac{9(1+\alpha\LJ)^2\alpha^2\LJ^2\GJ^2}{2\,m_{\mathrm{in}}}
\ \le\ \frac92\cdot\frac{(1+\alpha\LJ)^2\alpha^2\LJ^2\GJ^2\,\varepsilon}{18(1+\alpha\LJ)^2\alpha^2\LJ^2\GJ^2}
=\frac{\varepsilon}{4} .
\]
Equivalently, (iii) is exactly the requirement $\bhat\le\sqrt{\varepsilon}/(3\sqrt2)$; the constant $18$ is $9/2$ divided by $1/4$. Observe that $m_{\mathrm{in}}$ appears in no other condition except through $\shat^2$ in (iv), and that (iii) is the \emph{only} condition involving the adaptation step size $\alpha$: the required inner batch grows quadratically in $\alpha$, so stronger personalization is statistically more expensive, as noted in regime (R3) of Section~\ref{sec:regimes}.

\emph{Term (D), from condition (iv).} Immediate: $\beta\le\frac{n\varepsilon}{4\LF\shat^2}$ rearranges
to $\frac{\LF\beta\shat^2}{n}\le\frac{\varepsilon}{4}$.

\emph{Term (A), from the round count.} With $K\ge\frac{8\Delta}{\beta\tau\varepsilon}$ we obtain $\beta K\tau\ge\frac{8\Delta}{\varepsilon}$, hence $\frac{2\Delta}{\beta K\tau}\le\frac{2\Delta\varepsilon}{8\Delta}=\frac{\varepsilon}{4}$. If $\Delta=0$ then $\theta^0$ already maximizes $F$ and the claim is trivial, so we may assume $\Delta>0$ and the displayed quotient is well defined.

Summing the four quarters gives $\mathcal E\le\varepsilon$, which is the first assertion.

\medskip
\noindent\emph{Step 2: the instantiation satisfies (ii) and (iii).}
Take $m_{\mathrm h},m_{\mathrm{out}}=\Theta(1)$, $m_{\mathrm{in}}=\Theta(\varepsilon^{-1})$ chosen to meet (iii), $\tau=\lceil\varepsilon^{-1/2}\rceil$, and
\begin{equation}\label{eq:betachoice}
\beta:=\frac{\sqrt\varepsilon}{\sqrt6\,\LF\GF\,\tau} .
\end{equation}
Then $\beta\tau=\frac{\sqrt\varepsilon}{\sqrt6\LF\GF}$, so (ii) holds \emph{with equality}: the drift budget is saturated, which is what makes this choice extremal rather than merely feasible. Condition (iii) holds by the choice of $m_{\mathrm{in}}$. Since $\varepsilon^{-1/2}\le\tau\le\varepsilon^{-1/2}+1$ and we assume throughout that $\varepsilon\le1$ (otherwise the conclusion is vacuous, as $\norm{\nabla F}^2\le\GF^2$ always and one may simply rescale), we record for later the two-sided estimate
\begin{equation}\label{eq:taubounds}
\varepsilon^{-1/2}\ \le\ \tau\ \le\ 2\varepsilon^{-1/2},
\qquad\text{hence}\qquad
\tau=\Theta\big(\varepsilon^{-1/2}\big),
\end{equation}
the upper bound because $1\le\varepsilon^{-1/2}$ when $\varepsilon\le1$. Combining \eqref{eq:betachoice} with the lower bound in \eqref{eq:taubounds},
\begin{equation}\label{eq:betasize}
\beta=\frac{\sqrt\varepsilon}{\sqrt6\LF\GF\tau}\ \le\ \frac{\sqrt\varepsilon\cdot\sqrt\varepsilon}{\sqrt6\LF\GF}
=\frac{\varepsilon}{\sqrt6\,\LF\GF},
\qquad\text{and similarly}\qquad
\beta\ \ge\ \frac{\varepsilon}{2\sqrt6\,\LF\GF},
\end{equation}
so $\beta=\Theta\big(\varepsilon/(\LF\GF)\big)$.

\medskip
\noindent\emph{Step 3: the instantiation satisfies (i).}
By \eqref{eq:betasize} it suffices that $\frac{\varepsilon}{\sqrt6\LF\GF}\le\frac{1}{6\LF}$. Cancelling $\LF>0$ from both sides and rearranging,
\[
\frac{\varepsilon}{\sqrt6\,\GF}\le\frac16
\iff
\varepsilon\le\frac{\sqrt6\,\GF}{6}=\frac{\GF}{\sqrt6},
\]
which is the side condition stated in the corollary. This is a restriction only in the trivial regime: by Lemma~\ref{lem:meta-reg}(ii) we have $\norm{\nabla F(\theta)}^2\le\GF^2$ for every $\theta$, so any $\varepsilon$ of order $\GF^2$ or larger is met by \emph{every} point and the guarantee is empty. The condition $\varepsilon\le\GF/\sqrt6$ therefore excludes only accuracies that were never interesting.

\medskip
\noindent\emph{Step 4: the instantiation satisfies (iv), and the cost of its failure.}
Condition (iv) reads $\beta\le\frac{n\varepsilon}{4\LF\shat^2}$. Substituting \eqref{eq:betachoice} and cross-multiplying by the positive quantities $\sqrt6\LF\GF\tau$ and $4\LF\shat^2$,
\[
\frac{\sqrt\varepsilon}{\sqrt6\LF\GF\tau}\le\frac{n\varepsilon}{4\LF\shat^2}
\iff
4\,\shat^2\,\sqrt\varepsilon\ \le\ \sqrt6\,n\,\GF\,\tau\,\varepsilon
\iff
\shat^2\ \le\ \frac{\sqrt6}{4}\,n\,\GF\,\tau\sqrt{\varepsilon} .
\]
By \eqref{eq:taubounds}, $\tau\sqrt\varepsilon\ge1$, so it suffices that \begin{equation}\label{eq:sigmacond} \shat^2\ \le\ \frac{\sqrt6}{4}\,n\,\GF ,
\end{equation}
which is the second side condition stated in the corollary. Two observations make \eqref{eq:sigmacond} mild. First, it is $\varepsilon$-free: by Lemma~\ref{lem:estimator}(iii),
\[
\shat^2\le3\GJ^2\Big[\alpha^2\LJ^2+(1+\alpha\LJ)^2+(1+\alpha\LJ)^2\alpha^2\LJ^2\Big]
\qquad\text{for all }m_{\mathrm{in}},m_{\mathrm h},m_{\mathrm{out}}\ge1,
\]
so $\shat^2$ is bounded by a problem constant regardless of the batch sizes and cannot grow as $\varepsilon\to0$. Second, \eqref{eq:sigmacond} is satisfied for all sufficiently large $n$, and for any fixed $n$ it is satisfied once $m_{\mathrm h}$ and $m_{\mathrm{out}}$ are large enough, since $\shat^2\to3\GJ^2(1+\alpha\LJ)^2\alpha^2\LJ^2/m_{\mathrm{in}}$ as $m_{\mathrm h},m_{\mathrm{out}}\to\infty$.

Suppose nonetheless that \eqref{eq:sigmacond} fails. One may then replace \eqref{eq:betachoice} by $\beta':=\min\big\{\beta,\ \frac{n\varepsilon}{4\LF\shat^2}\big\}$, which satisfies (i), (ii) and (iv) simultaneously, since decreasing $\beta$ preserves (i) and (ii). The cost is confined to the round count, which by Step~5 below is inversely proportional to $\beta$: writing $K'=8\Delta/(\beta'\tau\varepsilon)$ and using
$\frac{1}{\min\{x,y\}}\le\frac1x+\frac1y$,
\[
K'\ \le\ \frac{8\Delta}{\tau\varepsilon}\Big(\frac{1}{\beta}+\frac{4\LF\shat^2}{n\varepsilon}\Big)
\ =\ \underbrace{\frac{8\sqrt6\,\Delta\LF\GF}{\varepsilon^{3/2}}}_{\text{as in Step~5}}
+\ \frac{32\,\Delta\LF\shat^2}{n\,\tau\,\varepsilon^{2}}
\ \overset{\eqref{eq:taubounds}}{\le}\ \frac{8\sqrt6\,\Delta\LF\GF}{\varepsilon^{3/2}}+\frac{32\,\Delta\LF\shat^2}{n\,\varepsilon^{3/2}},
\]
i.e.\ $K'=\mathcal O\big(\Delta\LF(\GF+\shat^2/n)\varepsilon^{-3/2}\big)$. The exponent $\varepsilon^{-3/2}$ is therefore unaffected and only the prefactor changes, growing by an additive $\shat^2/n$; we state \eqref{eq:sigmacond} explicitly rather than absorbing this into the $\mathcal O(\cdot)$ so that the dependence is visible.

\medskip
\noindent\emph{Step 5: the resulting complexities.}
Take $K$ at its smallest admissible value, $K=\lceil 8\Delta/(\beta\tau\varepsilon)\rceil$. Using $\beta\tau=\frac{\sqrt\varepsilon}{\sqrt6\LF\GF}$ from Step~2,
\begin{equation}\label{eq:Kvalue}
\frac{8\Delta}{\beta\tau\varepsilon}
=\frac{8\Delta}{\varepsilon}\cdot\frac{\sqrt6\,\LF\GF}{\sqrt\varepsilon}
=\frac{8\sqrt6\,\Delta\,\LF\,\GF}{\varepsilon^{3/2}}
=\mathcal O\big(\Delta\LF\GF\,\varepsilon^{-3/2}\big),
\end{equation}
so $K=\mathcal O(\Delta\LF\GF\varepsilon^{-3/2})$ communication rounds suffice, with $\tau=\Theta(\varepsilon^{-1/2})$ local steps per round by \eqref{eq:taubounds}. Multiplying, the total number of local meta-gradient steps taken by each agent is
\begin{equation}\label{eq:Ktau}
K\tau=\mathcal O\big(\varepsilon^{-3/2}\big)\cdot\Theta\big(\varepsilon^{-1/2}\big)=\mathcal O\big(\varepsilon^{-2}\big).
\end{equation}

For the trajectory complexity, each call to \textsc{MetaGrad} draws $m_{\mathrm{in}}+m_{\mathrm h}+m_{\mathrm{out}}$ trajectories, and with $m_{\mathrm{in}}=\Theta(\varepsilon^{-1})$ and $m_{\mathrm h},m_{\mathrm{out}}=\Theta(1)$ this is $\Theta(\varepsilon^{-1})$. Hence each agent consumes
\[
K\tau\big(m_{\mathrm{in}}+m_{\mathrm h}+m_{\mathrm{out}}\big)
=\mathcal O\big(\varepsilon^{-2}\big)\cdot\Theta\big(\varepsilon^{-1}\big)=\mathcal O\big(\varepsilon^{-3}\big)
\]
trajectories, which is the last assertion. By Proposition~\ref{prop:sharpbias} the inner batch may instead be taken as $m_{\mathrm{in}}=\Theta(\varepsilon^{-1/2})$, which reduces this to $\mathcal O(\varepsilon^{-5/2})$ while leaving \eqref{eq:Kvalue} and \eqref{eq:Ktau} untouched, since neither $K$ nor $\tau$ depends on $m_{\mathrm{in}}$.

\medskip
\noindent\emph{Step 6: optimality of the choice of $\tau$.}
The choice $\tau=\Theta(\varepsilon^{-1/2})$ is not arbitrary: it is the unique scaling, up to constants, that minimizes communication subject to $\mathcal E\le\varepsilon$. Treating $\tau$ as free and saturating (ii) with $\beta\tau=c\sqrt\varepsilon/(\LF\GF)$, the round count \eqref{eq:Kvalue} becomes $K=\Theta(\Delta\LF\GF\varepsilon^{-3/2})$ \emph{independently of $\tau$}, so pushing $\tau$ higher gives no further reduction in $K$; it merely forces $\beta$ down in proportion, since only the product is constrained. Pushing $\tau$ lower, on the other hand, leaves the drift budget unused and forces $K$ up. At $\tau=1$ the drift term is not merely small but exactly zero: the sum evaluated in Step~2 of the proof of Theorem~\ref{thm:main} is $\Sigma_{t^2}=K\sum_{t=0}^{0}t^2=0$, so term~(B) of \eqref{eq:mainbound} vanishes and condition (ii) is \emph{vacuous} rather than binding. (The estimate $\Sigma_{t^2}\le K\tau^3/3$ used to state \eqref{eq:mainbound} is loose precisely at $\tau=1$, and one should read off the exact value instead.) The active constraints are then (i) and (iv), and Step~1 gives $K=\mathcal O\big(\LF\Delta/\varepsilon+\LF\Delta\shat^2/(n\varepsilon^2)\big)$, as reported in regime (R3) of Section~\ref{sec:regimes}. The value $\tau=\Theta(\varepsilon^{-1/2})$ is thus the smallest number of local steps at which the drift constraint (ii), rather than the step-size ceiling (i), is the active one, and $(K,\tau)=\big(\Theta(\varepsilon^{-3/2}),\Theta(\varepsilon^{-1/2})\big)$ is the corresponding frontier. Note finally that the total local work $K\tau=\mathcal O(\varepsilon^{-2})$ in \eqref{eq:Ktau} is the same at every point of this frontier: federation reduces communication by a factor of $\tau$ at no cost in computation, up to the ceiling, which is the precise sense in which local steps are ``free''.
\end{proof}

\subsection[Proof of the Hessian-free corollary]{Proof of Corollary~\ref{cor:fo} (convergence of the Hessian-free variant)}\label{app:fo}

\begin{proof}
Throughout, $\E$ is conditional on the current iterate $\theta$, and we retain the notation of the proof of Lemma~\ref{lem:estimator}: $\theta^+=\theta+\alpha\nabla J_i(\theta)$, $\tilde\theta=\theta+\alpha\hat g_{\mathrm{in}}$ and $\mathcal J_i(\theta)=I+\alpha\nabla^2J_i(\theta)$. The first-order estimator is
\[
\widehat{\nabla F}^{\mathrm{FO}}_i(\theta):=\hat g_{\mathrm{out}},
\]
that is, \eqref{eq:estimator} with the curvature factor $(I+\alpha\hat H)$ replaced by the identity. The inner batch $\mathcal D_{\mathrm{in}}$ is still drawn, since $\tilde\theta$ is still needed to sample $\mathcal D_{\mathrm{out}}$ under the adapted policy; only the curvature batch is dispensed with, so $m_{\mathrm h}=0$.

The plan is to re-establish, for this estimator, the three properties of Lemma~\ref{lem:estimator} on which Lemma~\ref{lem:drift}, Lemma~\ref{lem:onestep} and Theorem~\ref{thm:main} depend, namely an almost sure norm bound, a bias bound and a second-moment bound. Once these are in place with constants $(\GF,\bhat_{\mathrm{FO}},\shat^2_{\mathrm{FO}})$, every subsequent step of the analysis applies verbatim, because those three properties are the \emph{only} features of the estimator that the convergence proof uses.

\medskip
\noindent\emph{Step 1: verification of the almost sure bound.}
Lemma~\ref{lem:estimator}(i) is a statement about \eqref{eq:estimator} and does not transfer automatically, so we verify it directly. By Lemma~\ref{lem:pg-bounds}(ii), applied pathwise at the random parameter $\tilde\theta$ exactly as in Step~1 of the proof of Lemma~\ref{lem:estimator},
\begin{equation}\label{eq:foas}
\big\lVert\widehat{\nabla F}^{\mathrm{FO}}_i(\theta)\big\rVert=\norm{\hat g_{\mathrm{out}}}\le\GJ\le(1+\alpha\LJ)\GJ=\GF
\end{equation}
for every realization of the batches. The bound is thus not merely inherited but \emph{improved}: the FO estimator obeys the sharper constant $\GJ$ in place of $\GF$. Since Lemma~\ref{lem:drift} uses only that the local increments are bounded by $\GF$, it carries over unchanged; a reader who wishes to exploit the improvement may replace $\GF$ by $\GJ$ throughout the drift and one-step lemmas, which tightens the drift term of \eqref{eq:mainbound} by the factor $(1+\alpha\LJ)^{-2}$. We do not do so, in order to state \eqref{eq:mainbound} once with a single pair of constants.

\medskip
\noindent\emph{Step 2: the bias.}
The factorization argument of Step~3 of the proof of Lemma~\ref{lem:estimator} applies with the curvature factor removed, and is simpler: no conditional independence is required, because there is no product of two random objects to decouple. Taking expectations in
\eqref{eq:goutunbiased} by the tower rule,
\begin{equation}\label{eq:fofactorization}
\E\big[\widehat{\nabla F}^{\mathrm{FO}}_i(\theta)\big]
=\E\big[\hat g_{\mathrm{out}}\big]
=\E_{\mathcal D_{\mathrm{in}}}\Big[\E\big[\hat g_{\mathrm{out}}\mid\mathcal D_{\mathrm{in}}\big]\Big]
=\E_{\mathcal D_{\mathrm{in}}}\big[\nabla J_i(\tilde\theta)\big].
\end{equation}
Subtracting $\nabla F_i(\theta)=\mathcal J_i(\theta)\nabla J_i(\theta^+)$ from \eqref{eq:fofactorization} and inserting the mixed term $\nabla J_i(\theta^+)$, the error splits into two pieces of quite different character:
\begin{equation}\label{eq:fobias}
\E\big[\widehat{\nabla F}^{\mathrm{FO}}_i\big]-\nabla F_i
=\underbrace{\Big(\E_{\mathcal D_{\mathrm{in}}}\big[\nabla J_i(\tilde\theta)\big]-\nabla J_i(\theta^+)\Big)}_{\text{inner-loop sampling bias}}
\;-\;\underbrace{\alpha\,\nabla^2J_i(\theta)\,\nabla J_i(\theta^+)}_{\text{discarded curvature correction}} ,
\end{equation}
where we used $\mathcal J_i(\theta)\nabla J_i(\theta^+)=\nabla J_i(\theta^+)+\alpha\nabla^2J_i(\theta)\nabla J_i(\theta^+)$. The first bracket is precisely the quantity bounded in Step~4 of the proof of Lemma~\ref{lem:estimator}, namely by $\LJ\,\E\norm{\tilde\theta-\theta^+}\le\alpha\LJ\GJ/\sqrt{m_{\mathrm{in}}}$. The second is \emph{deterministic} given $\theta$ and, by Lemma~\ref{lem:pg-bounds}(ii)--(iii) and submultiplicativity, bounded by
\[
\alpha\,\norm{\nabla^2J_i(\theta)}\,\norm{\nabla J_i(\theta^+)}\le\alpha\LJ\GJ .
\]
The triangle inequality applied to \eqref{eq:fobias} therefore gives
\begin{equation}\label{eq:bhatfo}
\big\lVert\E\widehat{\nabla F}^{\mathrm{FO}}_i(\theta)-\nabla F_i(\theta)\big\rVert
\ \le\ \frac{\alpha\LJ\GJ}{\sqrt{m_{\mathrm{in}}}}+\alpha\LJ\GJ
=\alpha\LJ\GJ\Big(1+\frac{1}{\sqrt{m_{\mathrm{in}}}}\Big)=\bhat_{\mathrm{FO}} .
\end{equation}
The structure of \eqref{eq:bhatfo} is the substance of the corollary. Its first summand is a \emph{sampling} error, which decays as more inner trajectories are drawn; its second is an \emph{approximation} error, arising because the estimator targets the wrong quantity, and it does not decay in any batch size because no batch size is involved in it. Setting $m_{\mathrm{in}}=\infty$ leaves $\bhat_{\mathrm{FO}}=\alpha\LJ\GJ>0$ whenever $\alpha>0$ and the curvature-times-gradient product is nonzero.

\medskip
\noindent\emph{Step 3: the second moment.}
Insert and remove $\nabla J_i(\tilde\theta)$ and $\nabla J_i(\theta^+)$ to obtain
\begin{equation}\label{eq:foT123}
\widehat{\nabla F}^{\mathrm{FO}}_i-\nabla F_i
=\underbrace{\big(\hat g_{\mathrm{out}}-\nabla J_i(\tilde\theta)\big)}_{T_1'\ \text{(outer gradient)}}
+\underbrace{\big(\nabla J_i(\tilde\theta)-\nabla J_i(\theta^+)\big)}_{T_2'\ \text{(inner gradient)}}
-\underbrace{\alpha\,\nabla^2J_i(\theta)\,\nabla J_i(\theta^+)}_{T_3'\ \text{(deterministic)}} ,
\end{equation}
an identity verified by cancelling the inserted terms and using $\nabla F_i=\nabla J_i(\theta^+)+\alpha\nabla^2J_i(\theta)\nabla J_i(\theta^+)$ as in Step~2. By (F4) with $k=3$, $\E\norm{\widehat{\nabla F}^{\mathrm{FO}}_i-\nabla F_i}^2\le3\big(\E\norm{T_1'}^2+\E\norm{T_2'}^2+\norm{T_3'}^2\big)$, no expectation being needed on the last term. The three are bounded exactly as $T_2$, $T_3$ and the curvature correction above:
\[
\E\norm{T_1'}^2\ \overset{(a)}{\le}\ \frac{\GJ^2}{m_{\mathrm{out}}},
\qquad
\E\norm{T_2'}^2\ \overset{(b)}{\le}\ \frac{\LJ^2\alpha^2\GJ^2}{m_{\mathrm{in}}},
\qquad
\norm{T_3'}^2\ \overset{(c)}{\le}\ \alpha^2\LJ^2\GJ^2 ,
\]
where $(a)$ conditions on $\mathcal D_{\mathrm{in}}$ and applies (F1) with \eqref{eq:sigmadef} at the parameter $\tilde\theta$, as in the bound on $T_2$ in Step~5 of the proof of Lemma~\ref{lem:estimator}, the factor $\norm{\mathcal J_i(\theta)}^2$ being absent here; $(b)$ combines Lemma~\ref{lem:pg-bounds}(iii) with \eqref{eq:displacement} and \eqref{eq:innervar}, again without the factor $(1+\alpha\LJ)^2$; and $(c)$ is the bound on the deterministic term from Step~2. Collecting and factoring out $3\GJ^2$,
\begin{equation}\label{eq:sigmafo}
\E\big\lVert\widehat{\nabla F}^{\mathrm{FO}}_i(\theta)-\nabla F_i(\theta)\big\rVert^2
\ \le\ 3\GJ^2\bigg[\frac{1}{m_{\mathrm{out}}}+\frac{\alpha^2\LJ^2}{m_{\mathrm{in}}}+\alpha^2\LJ^2\bigg]
=\shat^2_{\mathrm{FO}} .
\end{equation}
Comparing \eqref{eq:sigmafo} with $\shat^2$ of Lemma~\ref{lem:estimator}(iii), the curvature term $3\GJ^2\alpha^2\LJ^2/m_{\mathrm h}$ has been \emph{replaced} by the constant $3\GJ^2\alpha^2\LJ^2$, which is its value at $m_{\mathrm h}=1$. The first-order variant may therefore be read as the exact variant with a curvature batch of size one, in which the single-sample Hessian estimate has been replaced by zero; dropping the term does not remove its contribution to the error, it freezes it. Finally, the centered bound $\E\norm{\widehat{\nabla F}^{\mathrm{FO}}_i-\E\widehat{\nabla F}^{\mathrm{FO}}_i}^2\le\shat^2_{\mathrm{FO}}$ follows from (F6) exactly as before.

\medskip
\noindent\emph{Step 4: transfer of the convergence analysis.}
Steps~1--3 supply, for $\widehat{\nabla F}^{\mathrm{FO}}_i$, the three properties
\[
\big\lVert\widehat{\nabla F}^{\mathrm{FO}}_i\big\rVert\le\GF\ \text{surely},
\qquad
\big\lVert\E\widehat{\nabla F}^{\mathrm{FO}}_i-\nabla F_i\big\rVert\le\bhat_{\mathrm{FO}},
\qquad
\E\big\lVert\widehat{\nabla F}^{\mathrm{FO}}_i-\E\widehat{\nabla F}^{\mathrm{FO}}_i\big\rVert^2\le\shat^2_{\mathrm{FO}},
\]
which are the hypotheses (M1) and the pathwise bound used in the proofs of Lemma~\ref{lem:drift}, Lemma~\ref{lem:onestep} and Theorem~\ref{thm:main}. Those proofs invoke no other property of the estimator (in particular they never use that $\hat H$ exists, nor the specific form \eqref{eq:estimator}), so they apply verbatim with $\bhat$ and $\shat^2$ replaced by $\bhat_{\mathrm{FO}}$ and $\shat^2_{\mathrm{FO}}$. This yields \eqref{eq:mainbound} in the stated form. Note that conditional independence across agents, used in Step~4 of the proof of Lemma~\ref{lem:onestep}, still holds, since each agent still draws its own batches; only the within-agent independence of $\mathcal D_{\mathrm h}$ from $\mathcal D_{\mathrm{in}}$ has become vacuous, there being no curvature batch.

\medskip
\noindent\emph{Step 5: the bias term does not vanish.}
It remains to establish the final assertion. From \eqref{eq:bhatfo}, for every $m_{\mathrm{in}}\ge1$,
\begin{equation}\label{eq:fofloor}
\tfrac92\bhat_{\mathrm{FO}}^2
=\tfrac92\,\alpha^2\LJ^2\GJ^2\Big(1+\tfrac{1}{\sqrt{m_{\mathrm{in}}}}\Big)^2,
\qquad
\Big(1+\tfrac{1}{\sqrt{m_{\mathrm{in}}}}\Big)^2\in[1,4],
\end{equation}
the interval because $m_{\mathrm{in}}^{-1/2}\in(0,1]$, with the value $4$ attained at $m_{\mathrm{in}}=1$ and the infimum $1$ approached as $m_{\mathrm{in}}\to\infty$. Hence
\[
\tfrac92\bhat_{\mathrm{FO}}^2\ \ge\ \tfrac92\,\alpha^2\LJ^2\GJ^2\ >\ 0
\qquad\text{whenever }\alpha>0 ,
\]
and this quantity is independent of $K$, $\tau$, $\beta$, $n$, $m_{\mathrm{out}}$ and $m_{\mathrm h}$, and varies by at most a factor of four over all choices of $m_{\mathrm{in}}$. Since the three remaining terms of \eqref{eq:mainbound} can each be driven to zero (the first by $K\to\infty$, the second by $\beta\tau\to0$, the fourth by $\beta\to0$), the right-hand side of \eqref{eq:mainbound} satisfies
\[
\inf_{K,\tau,\beta,m_{\mathrm{in}},m_{\mathrm{out}}}\Big[\text{RHS of }\eqref{eq:mainbound}\Big]
=\tfrac92\,\alpha^2\LJ^2\GJ^2=\Theta\big(\alpha^2\LJ^2\GJ^2\big),
\]
so the guarantee degrades from convergence to an $\varepsilon$-FOSP to convergence to a neighborhood of that radius. Retracing Step~1 of the proof of Corollary~\ref{cor:rates}, the bias budget $\tfrac92\bhat_{\mathrm{FO}}^2\le\varepsilon/4$ is achievable only if $\tfrac92\alpha^2\LJ^2\GJ^2\le\varepsilon/4$, i.e.\ only in the small-adaptation regime
\begin{equation}\label{eq:smalladapt}
\alpha\ \le\ \frac{\sqrt{\varepsilon}}{3\sqrt2\,\LJ\GJ}=\mathcal O\Big(\frac{\sqrt\varepsilon}{\LJ\GJ}\Big),
\end{equation}
and conversely, when \eqref{eq:smalladapt} holds with $m_{\mathrm{in}}$ large enough that the factor in \eqref{eq:fofloor} is close to $1$, Corollary~\ref{cor:rates} applies with $\bhat_{\mathrm{FO}}$ in place of $\bhat$ and the full $\varepsilon$-FOSP guarantee is recovered. This is the precise sense in which the first-order variant is admissible only for weak personalization: by Proposition~\ref{prop:alpha0}, $\alpha=\mathcal O(\sqrt\varepsilon)$ also forces $\norm{\nabla F-\nabla f}=\mathcal O(\sqrt\varepsilon)$, so in exactly the regime where the Hessian-free method attains stationarity of $F$, the personalized objective has become $\mathcal O(\sqrt\varepsilon)$-close to the non-personalized one.

\medskip
\noindent\emph{Step 6: the scope of this statement.}
The scope of Step~5 is as follows. It shows that the \emph{upper bound} \eqref{eq:mainbound}, as established here, cannot be driven below $\Theta(\alpha^2\LJ^2\GJ^2)$ by any choice of algorithm parameters. It does not by itself show that the iterates fail to be stationary, since an upper bound that fails to vanish may in principle be loose. Proposition~\ref{prop:fofix} supplies that complementary direction, exhibiting a genuine $\Theta(\alpha)$ residual meta-gradient at the fixed points of the exact first-order dynamics, whose square matches the order of \eqref{eq:fofloor}. The two statements together justify the conclusion that the neighborhood is a property of the method rather than an artifact of the analysis; neither suffices alone.
\end{proof}

\begin{remark}[The scope of Corollary~\ref{cor:fo}, and why Proposition~\ref{prop:fofix} is needed]\label{rem:scope}
Corollary~\ref{cor:fo} is a statement about the \emph{guarantee}, not directly about the algorithm: a term that fails to vanish in an \emph{upper} bound does not by itself prove that the bounded quantity stays large, since the bound could simply be loose. Proposition~\ref{prop:fofix} supplies the complementary direction, exhibiting a genuine $\Theta(\alpha)$ residual meta-gradient at the fixed points of the first-order dynamics. Taken together, the two results show that the $\Theta(\alpha^2\LJ^2\GJ^2)$ order is a property of the method rather than an artifact of the analysis.
\end{remark}

\subsection[Proof of the fixed-point residual]{Proof of Proposition~\ref{prop:fofix} (residual meta-gradient at first-order fixed points)}\label{app:fofix}

\begin{proof}
Define the \emph{averaged first-order ascent direction}
\begin{equation}\label{eq:Ddef}
D(\theta):=\frac1n\sum_{i=1}^n\nabla J_i\big(\theta_i^{+}(\theta)\big),
\qquad \theta_i^{+}(\theta)=\theta+\alpha\nabla J_i(\theta),
\end{equation}
so that the hypothesis of the proposition reads $D(\theta^\star)=0$. That $D$ is indeed the direction followed by the exact first-order dynamics is immediate from the definition of the estimator: by Step~2 of the proof of Corollary~\ref{cor:fo}, in the limit $m_{\mathrm{in}},m_{\mathrm{out}}\to\infty$ the sampling errors vanish and $\widehat{\nabla F}^{\mathrm{FO}}_i(\theta)\to\nabla J_i(\theta_i^{+}(\theta))$, whence the averaged local direction of Algorithm~\ref{alg:main} converges to $D(\theta)$. Fixed points of the deterministic dynamics $\theta\mapsto\theta+\beta D(\theta)$ are exactly the zeros of $D$, since $\beta>0$.

\medskip
\noindent\emph{Step 1: the identity.}
Average the exact meta-gradient \eqref{eq:metagrad} over $i$ and split off the identity part of the curvature factor:
\begin{align}
\nabla F(\theta)
&=\frac1n\sum_{i=1}^n\nabla F_i(\theta)
\ \overset{\eqref{eq:metagrad}}{=}\ \frac1n\sum_{i=1}^n\big(I+\alpha\nabla^2J_i(\theta)\big)\nabla J_i\big(\theta_i^{+}(\theta)\big)\notag\\
&=\underbrace{\frac1n\sum_{i=1}^n\nabla J_i\big(\theta_i^{+}(\theta)\big)}_{=\,D(\theta)\ \text{by}\ \eqref{eq:Ddef}}
\;+\;\frac{\alpha}{n}\sum_{i=1}^n\nabla^2J_i(\theta)\,\nabla J_i\big(\theta_i^{+}(\theta)\big)
\;=\;D(\theta)+\alpha\,C(\theta),\label{eq:gradsplit}
\end{align}
where we have named the \emph{curvature-weighted average of the adapted gradients}
\begin{equation}\label{eq:Cdef}
C(\theta):=\frac1n\sum_{i=1}^n\nabla^2J_i(\theta)\,\nabla J_i\big(\theta_i^{+}(\theta)\big).
\end{equation}
Identity \eqref{eq:gradsplit} holds at every $\theta\in\R^d$ and is exact, not an approximation: it merely regroups the terms of \eqref{eq:metagrad}. Evaluating it at $\theta=\theta^\star$ and using $D(\theta^\star)=0$,
\begin{equation}\label{eq:residual}
\nabla F(\theta^\star)=\alpha\,C(\theta^\star)
=\frac{\alpha}{n}\sum_{i=1}^n\nabla^2J_i(\theta^\star)\,\nabla J_i\big(\theta_i^{+}(\theta^\star)\big),
\end{equation}
which is the claimed identity. The bound follows from the triangle inequality, submultiplicativity, and Lemma~\ref{lem:pg-bounds}(iii) and (ii) applied at $\theta^\star$ and at
$\theta_i^{+}(\theta^\star)$ respectively:
\begin{equation}\label{eq:residualbound}
\norm{\nabla F(\theta^\star)}
\le\frac{\alpha}{n}\sum_{i=1}^n\norm{\nabla^2J_i(\theta^\star)}\,\big\lVert\nabla J_i\big(\theta_i^{+}(\theta^\star)\big)\big\rVert
\le\frac{\alpha}{n}\sum_{i=1}^n\LJ\GJ=\alpha\LJ\GJ .
\end{equation}
Note that \eqref{eq:residual} makes no use of Assumption~\ref{as:policy}(c) or of any stochastic argument; it is a purely algebraic consequence of the chain rule together with the fixed-point condition.

\medskip
\noindent\emph{Step 2: consequences of the identity.}
The content of \eqref{eq:residual} is that the two dynamics have \emph{different} stationarity conditions. The first-order method halts where $D(\theta)=0$; the meta-objective is stationary where $D(\theta)+\alpha C(\theta)=0$. These coincide only where $C$ vanishes as well, and there is no mechanism forcing that: $D$ and $C$ are built from the same ingredients but weighted differently, $C$ inserting the curvature $\nabla^2J_i(\theta)$ before averaging. Consequently the first-order method is not a method with a small error for the problem \eqref{eq:meta}; it is an exact method for a \emph{different} problem, whose stationary points are displaced from those of \eqref{eq:meta} by an amount proportional to $\alpha$. This is the qualitative distinction between Proposition~\ref{prop:fofix} and a statement about a loose bound, and it is why the neighborhood in
Corollary~\ref{cor:fo} cannot be removed by a sharper analysis.

\medskip
\noindent\emph{Step 3: non-degeneracy of the residual.}
Since \eqref{eq:residual} expresses $\nabla F(\theta^\star)$ as $\alpha$ times a quantity that has no reason to vanish, the residual is of order $\alpha$ in general. We make this precise in three increasingly concrete ways, since the qualifier ``in general'' is otherwise not verifiable.

\emph{(a) The single-agent case.} With $n=1$ the fixed-point condition is
$\nabla J\big(\theta^{+}(\theta^\star)\big)=0$ and \eqref{eq:residual} becomes
\[
\nabla F(\theta^\star)=\alpha\,\nabla^2J(\theta^\star)\,\nabla J\big(\theta^{+}(\theta^\star)\big).
\]
Here the vector $v:=\nabla J(\theta^{+}(\theta^\star))$ satisfies $v=0$ by hypothesis, so in this degenerate case the residual \emph{does} vanish: with a single agent, $D(\theta)=0$ forces $\nabla F(\theta)=0$. The single-agent case is therefore \emph{not} a witness, and we record it to delimit the claim: the discrepancy is a genuinely multiagent phenomenon, created by the averaging over heterogeneous agents. It arises because $D(\theta^\star)=0$ constrains only the \emph{average} of the adapted gradients, leaving the individual vectors $\nabla J_i(\theta_i^{+}(\theta^\star))$ free to be nonzero and to be reweighted by the distinct curvatures $\nabla^2J_i(\theta^\star)$.

\emph{(b) A characterization of the degenerate configurations.} Write
$v_i:=\nabla J_i(\theta_i^{+}(\theta^\star))$ 
and $\Hs_i:=\nabla^2J_i(\theta^\star)$ (the per-agent curvature), so that the fixed-point condition is $\sum_iv_i=0$ and the residual is $\frac{\alpha}{n}\sum_i\Hs_iv_i$. The residual vanishes precisely when
\begin{equation}\label{eq:degenerate}
\sum_{i=1}^n\Hs_iv_i=0
\qquad\text{subject to}\qquad
\sum_{i=1}^nv_i=0 .
\end{equation}
If all agents share the same curvature at $\theta^\star$, say $\Hs_i\equiv\Hs$, then $\sum_i\Hs_iv_i=\Hs\sum_iv_i=0$ automatically and the residual vanishes identically. Thus \emph{curvature heterogeneity at $\theta^\star$ is necessary} for the discrepancy: the first-order method is exact, at its fixed points, precisely to the extent that the agents' Hessians agree there. Conversely, whenever the $\Hs_i$ differ, \eqref{eq:degenerate} is a system of $d$ linear equations constraining the $(n-1)d$ free parameters of the constraint set $\{\sum_iv_i=0\}$, and its solution set is a proper linear subspace thereof whenever the map $(v_1,\dots,v_n)\mapsto\sum_i\Hs_iv_i$ does not annihilate that set. Hence the degenerate configurations form a set of measure zero within the constraint set, and the residual is nonzero for all configurations outside it. This is the sense in which we call the vanishing of the residual a non-generic coincidence.

\emph{(c) An explicit witness.} Degeneracy is not merely non-generic but avoidable by construction. Here, we exhibit an instance in full. Take $n=2$ agents whose MDPs each consist of a single state, three actions $\{a_1,a_2,a_3\}$ and horizon $H=0$, so that a trajectory is one action draw and $J_i(\theta)=\E_{a\sim\pi(\cdot;\theta)}[r_i(a)]$. Let the shared policy class be the scalar log-linear family $\pi(a_j\,|\,\cdot\,;\theta)\propto\exp(\theta\phi_j)$ with $d=1$ and features $\phi=(0,1,3)$, and let the agents differ only in reward,
\[
r_1=(1,0,0),\qquad r_2=(0,0,1).
\]
Assumption~\ref{as:reward} holds with $\Rmax=1$, and Assumption~\ref{as:policy} holds \emph{globally} by Example~\ref{ex:softmax} with $B=\max_j|\phi_j|=3$. Writing $v_i$ and $\Hs_i$ as in (b), the residual at any zero of $D$ is
\[
\nabla F(\theta^\star)=\tfrac{\alpha}{2}\big(\Hs_1-\Hs_2\big)v_1 ,
\]
since $v_2=-v_1$. The map $D$ has an interior zero at every $\alpha\in\{1,\tfrac12,\tfrac14,\tfrac1{10},\tfrac1{20},\tfrac1{100},10^{-3},10^{-4}\}$, and the residual is nonzero at each. At $\alpha=\tfrac12$, for instance,
\[
\theta^\star=-0.5555031,\quad
v_1=-v_2=-0.3195682,\quad
\Hs_1=-0.2568571,\quad
\Hs_2=0.4999539,
\]
giving $\nabla F(\theta^\star)=0.0604632\neq0$. As $\alpha\to0$ the fixed point converges to the stationary point $\theta_0=-0.2310491$ of the non-personalized objective $f$, at which $v_1=-0.4359767$ and $\Hs_1-\Hs_2=-0.7210902$, so
\[
\frac{\norm{\nabla F(\theta^\star)}}{\alpha}\ \longrightarrow\ \tfrac12\big|\Hs_1-\Hs_2\big|\,|v_1|\ =\ 0.1571893 .
\]
The residual is therefore $\Theta(\alpha)$ with a constant bounded away from zero, on a fully specified pair of MDPs satisfying every assumption of the paper.

\medskip
\noindent\emph{Step 4: matching the order of Corollary~\ref{cor:fo}.}
Combining Step~3 with \eqref{eq:residualbound}, at a fixed point of the exact first-order dynamics the meta-gradient satisfies
\[
\norm{\nabla F(\theta^\star)}=\Theta(\alpha)
\qquad\text{in general},
\qquad\text{with}\qquad
\norm{\nabla F(\theta^\star)}\le\alpha\LJ\GJ\ \text{always},
\]
so its squared norm is $\Theta(\alpha^2)$, with the upper bound $\alpha^2\LJ^2\GJ^2$. This matches, up to the absolute constant $\tfrac92$, the non-vanishing bias term $\tfrac92\bhat_{\mathrm{FO}}^2\ge\tfrac92\alpha^2\LJ^2\GJ^2$ identified in \eqref{eq:fofloor}. The agreement of the two orders is the substance of the pairing: Corollary~\ref{cor:fo} shows that the \emph{guarantee} stalls at $\Theta(\alpha^2\LJ^2\GJ^2)$, and Proposition~\ref{prop:fofix} shows that the \emph{method} stalls at squared meta-gradient of the same order. Neither statement is redundant, and neither would suffice alone: a non-vanishing upper bound could reflect a loose analysis, while a displaced fixed point of the deterministic dynamics implies nothing on its own about what the stochastic algorithm achieves in finite time.

\medskip
\noindent\emph{Step 5: scope.}
The statement has two limitations. First, it concerns fixed points of the \emph{exact} ($m\to\infty$) dynamics; the stochastic algorithm does not converge to a point, and the finite-sample statement is Corollary~\ref{cor:fo}. Second, it does not assert that fixed points exist: $D$ may have no zero, in which case the proposition is vacuous for that instance, and it is Corollary~\ref{cor:fo} that carries the guarantee. What the proposition does establish, and all it is used for, is that the displacement between the two stationarity conditions is real and of order $\alpha$, so that the neighborhood of Corollary~\ref{cor:fo} is not an artifact of the analysis. Empirically the displacement is visible in Figure~\ref{fig:tab}: the first-order variant plateaus at $F=0.73$ against $0.79$ for the exact variant, rather than converging to it more slowly. 
\end{proof}

\subsection[Proof of the small-adaptation limit]{Proof of Proposition~\ref{prop:alpha0} ($\alpha\to0$ limit)}\label{app:regimes}

\begin{proof}
Recall $f(\theta)=\frac1n\sum_{i=1}^nJ_i(\theta)$, the non-personalized federated objective, so that $\nabla f(\theta)=\frac1n\sum_i\nabla J_i(\theta)$. Fix $\theta\in\R^d$ and, for each $i$, write $\theta^+_i:=\theta_i^{+}(\theta)=\theta+\alpha\nabla J_i(\theta)$.

\medskip
\noindent\emph{Step 1: the per-agent difference.}
Both $\nabla F_i(\theta)$ and $\nabla J_i(\theta)$ involve $\nabla J_i$, but at different points and with different prefactors. Expanding \eqref{eq:metagrad} and inserting $\nabla J_i(\theta^+_i)$ as a pivot,
\begin{align}
\nabla F_i(\theta)-\nabla J_i(\theta)
&\overset{\eqref{eq:metagrad}}{=}\big(I+\alpha\nabla^2J_i(\theta)\big)\nabla J_i(\theta^+_i)-\nabla J_i(\theta)\notag\\
&=\underbrace{\Big(\nabla J_i(\theta^+_i)-\nabla J_i(\theta)\Big)}_{=:E_i^{(1)}\ \text{(displacement of the evaluation point)}}
\;+\;\underbrace{\alpha\,\nabla^2J_i(\theta)\,\nabla J_i(\theta^+_i)}_{=:E_i^{(2)}\ \text{(curvature prefactor)}} .\label{eq:alpha0split}
\end{align}
The two error sources are structurally distinct: $E_i^{(1)}$ arises because the meta-gradient is evaluated at the \emph{adapted} parameter rather than at $\theta$, and $E_i^{(2)}$ because the chain rule contributes the factor $I+\alpha\nabla^2J_i(\theta)$ rather than the identity. Both are $\mathcal O(\alpha)$, but for different reasons, and both vanish at $\alpha=0$.

\emph{Bounding $E_i^{(1)}$.} By Lemma~\ref{lem:pg-bounds}(iii) the map $\nabla J_i$ is
$\LJ$-Lipschitz, so
\[
\big\lVert E_i^{(1)}\big\rVert
\;\overset{\text{Lem.~\ref{lem:pg-bounds}(iii)}}{\le}\;\LJ\,\norm{\theta^+_i-\theta}
\;=\;\LJ\cdot\alpha\norm{\nabla J_i(\theta)}
\;\overset{\text{Lem.~\ref{lem:pg-bounds}(ii)}}{\le}\;\alpha\,\LJ\,\GJ ,
\]
where the middle equality is the definition $\theta^+_i-\theta=\alpha\nabla J_i(\theta)$. Note the mechanism: the adaptation step has length at most $\alpha\GJ$, and moving the evaluation point that far changes the gradient by at most $\LJ$ times that distance.

\emph{Bounding $E_i^{(2)}$.} By submultiplicativity and Lemma~\ref{lem:pg-bounds}(iii) and (ii), the latter applied at the point $\theta^+_i$,
\[
\big\lVert E_i^{(2)}\big\rVert\le\alpha\,\norm{\nabla^2J_i(\theta)}\,\big\lVert\nabla J_i(\theta^+_i)\big\rVert
\le\alpha\,\LJ\,\GJ .
\]

\medskip
\noindent\emph{Step 2: averaging.}
Adding the two bounds via the triangle inequality gives $\norm{\nabla F_i(\theta)-\nabla J_i(\theta)}\le2\alpha\LJ\GJ$ for every $i$ and every $\theta$. Since $\nabla F(\theta)-\nabla f(\theta)=\frac1n\sum_i\big(\nabla F_i(\theta)-\nabla J_i(\theta)\big)$, one more application of the triangle inequality yields
\begin{equation}\label{eq:alpha0bound}
\norm{\nabla F(\theta)-\nabla f(\theta)}
\le\frac1n\sum_{i=1}^n\norm{\nabla F_i(\theta)-\nabla J_i(\theta)}\le2\alpha\LJ\GJ
\qquad\text{for every }\theta\in\R^d,
\end{equation}
which is the claim. The bound is uniform in $\theta$ and independent of $n$: averaging neither amplifies nor attenuates the discrepancy, since the per agent bound already holds for each $i$ separately and no cancellation is claimed. In particular $\nabla F\to\nabla f$ \emph{uniformly} on $\R^d$ as $\alpha\to0$, not merely pointwise, and the rate is exactly first order in $\alpha$.

\medskip
\noindent\emph{Step 3: convergence of the objective values.}
Bound \eqref{eq:alpha0bound} concerns gradients; for completeness we note the corresponding statement for the objectives themselves, which is what makes ``the problem reduces to non-personalized FRL'' precise. By Lemma~\ref{lem:pg-bounds}(ii) and the mean value inequality applied along the segment from $\theta$ to $\theta^+_i$,
\[
\big|F_i(\theta)-J_i(\theta)\big|=\big|J_i(\theta^+_i)-J_i(\theta)\big|
\le\GJ\,\norm{\theta^+_i-\theta}=\alpha\GJ\norm{\nabla J_i(\theta)}\le\alpha\GJ^2 ,
\]
whence $|F(\theta)-f(\theta)|\le\alpha\GJ^2$ uniformly in $\theta$. Consequently $|F^\star-f^\star|\le\alpha\GJ^2$ for the two suprema as well, so the optimal values, and not only the stationarity conditions, agree to first order in $\alpha$.

\medskip
\noindent\emph{Step 4: degeneration of the guarantees.}
Setting $\alpha=0$ in the constants of Lemma~\ref{lem:meta-reg} and Lemma~\ref{lem:estimator} gives
\[
\GF\big|_{\alpha=0}=\GJ,
\qquad
\LF\big|_{\alpha=0}=\LJ,
\qquad
\bhat\big|_{\alpha=0}=0,
\qquad
\shat^2\big|_{\alpha=0}=\frac{3\GJ^2}{m_{\mathrm{out}}},
\]
since every $\alpha$-dependent factor in $\bhat$ and in the curvature and inner-gradient terms of $\shat^2$ carries an explicit factor of $\alpha$. Substituting into \eqref{eq:mainbound}, the bias term $\tfrac92\bhat^2$ disappears and Theorem~\ref{thm:main} becomes
\[
\frac{1}{K\tau}\sum_{k<K}\sum_{t<\tau}\E\norm{\nabla f(\bar\theta^{k,t})}^2
\ \le\ \frac{2\Delta}{\beta K\tau}+\tfrac32\LJ^2\GJ^2\beta^2\tau^2+\frac{3\LJ\beta\GJ^2}{n\,m_{\mathrm{out}}},
\]
a convergence guarantee for local-update federated policy gradient on $f$, with no bias floor and with the bias condition (iii) of Corollary~\ref{cor:rates} vacuous. The round complexity $K=\mathcal O(\Delta\LJ\GJ\varepsilon^{-3/2})$ and $\tau=\Theta(\varepsilon^{-1/2})$ are unchanged in order, and the per-agent trajectory complexity improves to $\mathcal O(\varepsilon^{-2})$, since $m_{\mathrm{in}}$ no longer needs to grow. Thus the personalized theory degenerates continuously to the non-personalized one, rather than merely resembling it.

\noindent\emph{Step 5: degeneration of the algorithm.}
The same degeneration occurs at the level of Algorithm~\ref{alg:main}. At $\alpha=0$ we have $\tilde\theta=\theta+\alpha\hat g_{\mathrm{in}}=\theta$, so $\mathcal D_{\mathrm{out}}$ is drawn under $q_i(\cdot;\theta)$, and the estimator \eqref{eq:estimator} collapses to
\[
\widehat{\nabla F}_i(\theta)\big|_{\alpha=0}=(I+0\cdot\hat H)\,\hat g_{\mathrm{out}}=\hat g_{\mathrm{out}}
=\frac{1}{m_{\mathrm{out}}}\sum_{\xi\in\mathcal D_{\mathrm{out}}}g_i(\xi;\theta),
\]
the ordinary REINFORCE-type policy-gradient estimator at $\theta$. Both the inner and the curvature batches become inert ($\hat g_{\mathrm{in}}$ no longer affects $\tilde\theta$ and $\hat H$ is multiplied by zero), so they need not be drawn at all, and \textsc{Per-FedAvg-PG} becomes \textsc{FedAvg} applied to policy gradient: $\tau$ local ascent steps on each agent's own return, followed by parameter averaging. The deployment step of Algorithm~\ref{alg:main} likewise returns $\theta^K+\alpha\hat g_{\mathrm{in}}=\theta^K$, i.e.\ every agent deploys the shared policy unmodified, which is the defining property of the non-personalized setting. Note that this is a genuine degeneration of the algorithm and not only of its analysis: the $\alpha=0$ run is \emph{identical}, trajectory for trajectory, to a run of federated policy gradient given the same random seed and outer batches.

\medskip
\noindent\emph{Step 6: interpretation.}
Taken together, Steps 2--5 justify treating $\alpha$ as a continuous personalization parameter rather than a discrete choice of mechanism. At $\alpha=0$ the method, the objective and the guarantee are those of non-personalized federated policy gradient; as $\alpha$ grows, the objective departs from $f$ at rate $\mathcal O(\alpha)$ in gradient and $\mathcal O(\alpha)$ in value, while three quantities identified elsewhere degrade: the smoothness constant $\LF$ grows (Lemma~\ref{lem:meta-reg}), the required inner batch grows quadratically in $\alpha$ (condition (iii) of Corollary~\ref{cor:rates}), and the Hessian-free variant's bias floor grows quadratically in $\alpha$ (Corollary~\ref{cor:fo}). This is the content of regime (R1) of Section~\ref{sec:regimes}, and it is confirmed empirically in Figure~\ref{fig:tab}, where the three methods coincide to within $0.01$ at $\alpha=0$ and separate monotonically as $\alpha$ increases.

Two caveats on the reading of \eqref{eq:alpha0bound}. First, closeness of gradients does not imply closeness of maximizers, since $F$ and $f$ are nonconcave. Quantitatively, if $\hat\theta$ is an $\varepsilon$-FOSP of $f$, then squaring the triangle inequality $\norm{\nabla F}\le\norm{\nabla f}+2\alpha\LJ\GJ$ supplied by \eqref{eq:alpha0bound}, taking expectations and bounding $\E\norm{\nabla f(\hat\theta)}\le\sqrt\varepsilon$ by (F3),
\[
\E\norm{\nabla F(\hat\theta)}^2\ \le\ \big(\sqrt\varepsilon+2\alpha\LJ\GJ\big)^2
\ =\ \varepsilon+4\alpha\LJ\GJ\sqrt\varepsilon+4\alpha^2\LJ^2\GJ^2 ,
\]
so $\hat\theta$ is an $\varepsilon'$-FOSP of $F$ with $\varepsilon'=\varepsilon+\mathcal O(\alpha\LJ\GJ\sqrt\varepsilon)+\mathcal O(\alpha^2\LJ^2\GJ^2)$. The cross term is the dominant correction whenever $\alpha\LJ\GJ\le\sqrt\varepsilon$, which is exactly the small-adaptation regime \eqref{eq:smalladapt}, so it cannot be absorbed into the quadratic one. Even so, the argmax sets need not be near one another. Second, the bound is a worst-case statement and does not establish that personalization is \emph{ineffective} for small $\alpha$: the constant $2\LJ\GJ$ is a uniform bound on how far the objectives can differ, whereas the realized benefit of adaptation, as measured in Figure~\ref{fig:tab}, may be substantial well before $\alpha$ reaches the moderate regime $\alpha\le1/(2\LJ)$.

\end{proof}

\section{Experimental details}\label{app:exp}

\subsection{Environments}

\paragraph{Tabular gridworld.} A $5\times5$ grid ($25$ states; $4$ actions U/D/L/R, with moves off the grid leaving the agent in place). Each agent's MDP differs \emph{only} in its goal cell: the reward is $1$ in the goal cell and $0$ elsewhere, the initial distribution is uniform over non-goal cells, $\gamma=0.9$, and the horizon is $15$. The $n=8$ agents place their goals at the four corners and the four edge midpoints (Figure~\ref{fig:grid}). The policy is a per-state softmax over the four actions, i.e.\ the one-hot case of Example~\ref{ex:softmax}, so Assumption~\ref{as:policy} holds with $\gmax\le2$, $M\le4$, $\Lambda\le8$. Because the family is finite and tabular, $F$, $\nabla F$ and $\nabla^2F$ are computed \emph{exactly} by dynamic programming, with no sampling noise; every tabular number reported in Section~\ref{sec:experiments} is therefore a property of the objective rather than of an estimator.

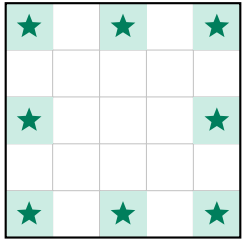
\begin{figure}[ht]\centering
\begin{tikzpicture}[scale=0.62,font=\small]
  \draw[step=1,ggray!45,thin] (0,0) grid (5,5);
  \foreach \x/\y in {0/0,4/0,0/4,4/4,2/0,0/2,4/2,2/4}{
    \fill[ggreen!20] (\x,\y) rectangle ++(1,1);
    \node[star,star points=5,star point ratio=2.3,fill=ggreen!80!black,inner sep=1.5pt] at (\x+0.5,\y+0.5){};}
  \draw[thick] (0,0) rectangle (5,5);
\end{tikzpicture}
\caption{Tabular family: the eight goal cells. Agents differ only in the goal.}
\Description{A five-by-five grid of cells. Eight cells are shaded and marked with a star: the four corners and the midpoint of each of the four edges. The remaining cells are empty.}
\label{fig:grid}
\end{figure}

\paragraph{Neural navigation.} The state is a position $x\in[-1,1]^2$; there are $8$ actions (unit compass directions), each moving the agent by $0.15$ and clipping to the box. The reward is $\mathbf 1\{\lVert x-g\rVert\le0.2\}$, collected at each step inside the goal region; $\gamma=0.9$ and the horizon is $20$. The policy is a multilayer perceptron $2\!\to\!32\!\to\!8$ with a single hidden layer of $32$ units, $\tanh$ activation and a softmax output, shared across agents; including biases this is $d=2\!\cdot\!32+32+32\!\cdot\!8+8=360$ parameters. Goals lie on an arc of the radius-$0.8$ circle spanning $\pm0.5$\,rad: six evenly spaced training goals, and three held-out goals at the interior midpoints (Figure~\ref{fig:arc}). This family has the shared structure a meta-initialization can exploit (``head toward the arc'' is common to all agents, while the precise goal is not), and returns and success rates are estimated by Monte Carlo.

\begin{figure}[ht]\centering
\begin{tikzpicture}[scale=0.95,font=\small]
  \draw[->,ggray!55] (-0.2,0)--(4.1,0) node[right,font=\footnotesize]{$x$};
  \draw[->,ggray!55] (0,-3.5)--(0,3.5) node[above,font=\footnotesize]{$y$};
  \draw[ggray!35,dashed] (0,0) circle (3.2);
  \draw[srv,very thick] (-28.6:3.2) arc (-28.6:28.6:3.2);
  \foreach \d in {-28.6,-17.2,-5.7,5.7,17.2,28.6}{\fill[srv] (\d:3.2) circle (2.3pt);}
  \foreach \d in {-22.9,0,22.9}{\filldraw[fill=white,draw=agt!85!black,line width=1pt] (\d:3.2) circle (3pt);}
  \node[anchor=west,font=\footnotesize,srv] at (0.7,-2.1){\tikz\fill[srv] circle(2.3pt); training goals (6)};
  \node[anchor=west,font=\footnotesize,agt!70!black] at (0.7,-2.6){\tikz\filldraw[fill=white,draw=agt!85!black,line width=1pt] circle(3pt); held-out goals (3)};
  \node[anchor=west,font=\footnotesize,ggray] at (0.7,-3.1){radius $0.8$, arc $\pm0.5$ rad};
\end{tikzpicture}
\caption{Neural arc family: shared-structure goals with a held-out split.}
\Description{A plot with horizontal and vertical axes and a dashed circle of radius 0.8 centered at the origin. A short arc on the right of the circle, spanning about plus and minus half a radian around the horizontal axis, is drawn in bold and carries six evenly spaced filled dots marking training goals, with three hollow dots at interior midpoints between them marking held-out goals. A legend below identifies the two marker types and states the radius and arc span.}
\label{fig:arc}
\end{figure}
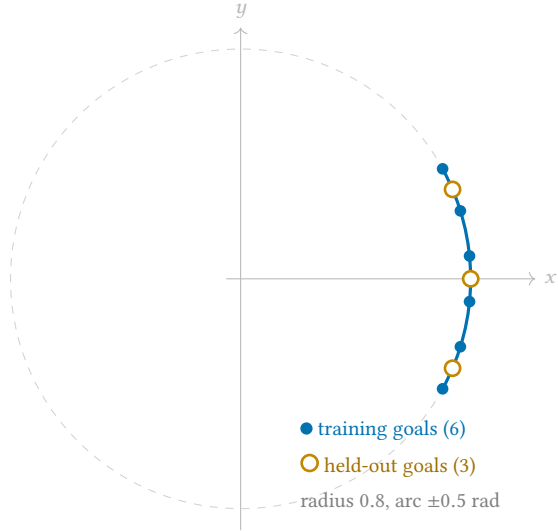

\subsection{Hyperparameters}
Table~\ref{tab:hyper} lists the settings used in every experiment.

\begin{table}[htbp]\centering\small\setlength{\tabcolsep}{5pt}
\caption{Hyperparameters. All runs use full participation and $10$ seeds, with $n=8$ agents (tabular) or $6$ training agents (neural). The first two rows run the exact variant, the first-order variant and non-personalized \textsc{FedAvg-PG}; the third runs the exact variant while sweeping the Hessian batch $m_{\mathrm h}$; the fourth runs the first-order variant. Here $m_{\mathrm h}$ is used by the exact variant and, within its single policy-gradient estimate, by \textsc{FedAvg-PG}, so the first-order variant spends $20$ trajectories per local step against $30$ for the other two.}
\label{tab:hyper}
\begin{tabular}{@{}llccccccc@{}}
\toprule
Experiment & Env. & $K$ & $\tau$ & $\alpha$ & $\beta$ & $m_{\mathrm{in}}$ & $m_{\mathrm h}$ & $m_{\mathrm{out}}$ \\
\midrule
Personalization; learning curves & tabular & $80$ & $5$ & $2$ & $0.3$ & $10$ & $10$ & $10$ \\
Adaptation step size $\alpha$ & tabular & $80$ & $5$ & $\{0,\tfrac14,\tfrac12,1,2,3\}$ & $0.3$ & $10$ & $10$ & $10$ \\
Curvature bottleneck & neural & $150$ & $5$ & $1$ & $0.2$ & $10$ & $\{10,50,100,200,500,1000\}$ & $10$ \\
Few-shot personalization & neural & $150$ & $5$ & $1$ & $0.2$ & $10$ & --- & $10$ \\
\bottomrule
\end{tabular}
\end{table}

At deployment the few-shot experiment adapts each held-out agent with a single policy-gradient step of batch $m_{\mathrm{adapt}}\in\{0,20,50,100,200,500\}$, where $m_{\mathrm{adapt}}=0$ denotes zero-shot evaluation. The from-scratch baseline trains each held-out agent independently from a random initialization, using $G\in\{1,5,25,100,200\}$ policy-gradient steps of batch $20$ at step size $\beta=0.2$, a total budget of $\{20,100,500,2000,4000\}$ trajectories; its largest budget is the specialist limit. The curvature-bottleneck row evaluates post-adaptation return on the six training goals, forming the evaluation adaptation step from $256$ trajectories and averaging $512$ rollouts per goal.

\subsection{Evaluation and metrics}
Tabular results use exact dynamic-programming values of $F$, so they carry no sampling noise. Neural results use Monte Carlo returns. ``Success'' is the fraction of evaluation episodes that enter the goal region $\{\lVert x-g\rVert\le0.2\}$ at some step. This is a standard goal-reaching criterion and it is deliberately lenient: the learned policies collect reward at the goal but do not always remain inside the region, so success rates exceed the fraction of episodes that \emph{end} there. All curves average $10$ seeds, and shaded bands are $\pm1$ standard deviation except where noted in Section~\ref{sec:experiments}.

\subsection{Tabular results in full}\label{app:tabfull}
Figure~\ref{fig:e6} reports the tabular comparison against communication rounds, the conventional meta-RL learning curve. It is the same data as Figure~\ref{fig:tab} (left) with the horizontal axis in rounds rather than per-agent trajectories, the two being proportional at fixed $\tau$ and batch sizes; we show both because the round axis is the one relevant to communication cost and the trajectory axis the one relevant to sample cost. The ordering of the three methods holds at every round, so the advantage reported in the main text is not an artifact of where training was stopped.

A per-state softmax shares no parameters across states, so a meta-initialization cannot encode a reusable skill and few-shot transfer to held-out agents is meaningful only under function approximation. This is why the few-shot experiment is neural, and why a per-agent specialist dominates any shared initialization in the tabular family; we therefore do not report that comparison there, as it would be uninformative rather than favorable.

\begin{figure}[ht]
\centering
\includegraphics[width=0.62\linewidth]{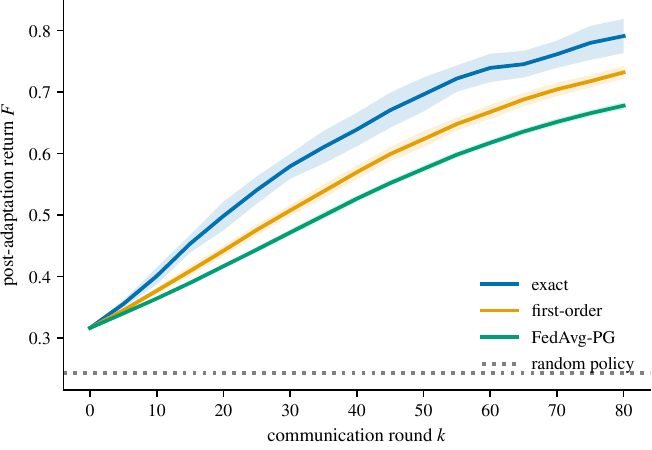}
\caption{\textbf{Learning curves} (tabular, exact evaluation). Post-adaptation return against communication round $k$. All methods improve monotonically and the personalized variants stay ahead of non-personalized \textsc{FedAvg-PG} at every round, so the advantage reported in Figure~\ref{fig:tab} is not an artifact of where training was stopped. This is the same data as Figure~\ref{fig:tab} (left) with the horizontal axis expressed in rounds rather than per-agent trajectories, the two being proportional at fixed $\tau$ and batch sizes.}
\Description{A line chart of post-adaptation return against communication round. Three curves rise monotonically from a common low starting value and never cross: exact highest, first-order just below it, and FedAvg-PG lowest throughout, with shaded seed bands around each.}
\label{fig:e6}
\end{figure}

\subsection{The curvature bottleneck}\label{app:curvfull}
The exact variant's recovery at large $m_{\mathrm h}$ is bimodal across seeds: a subset train to the \textsc{FedAvg-PG} level while the remainder stay collapsed. We therefore report the median together with per-seed points in Figure~\ref{fig:neural} rather than a mean and standard deviation, which would misrepresent a bimodal distribution. The bimodality is itself informative: it indicates that an inadequate curvature batch does not degrade training smoothly but destabilizes it, consistent with the multiplicative role of $\hat H$ in the estimator~\eqref{eq:estimator}.

\subsection{Few-shot personalization in full}\label{app:fewshotfull}
Figure~\ref{fig:success} reports success rates on held-out agents, the companion to the return curve of Figure~\ref{fig:neural} (right), and Figure~\ref{fig:e7} shows what those rates correspond to behaviorally at the level of individual rollouts on an unseen goal.

The analyzed objective adapts with the \emph{deterministic} gradient $\nabla J_i(\theta)$. A one-step update formed from very few trajectories ($20$) is too noisy to approximate it and can reduce performance, whereas $50$--$200$ trajectories recover and then improve on the initialization. This is the finite-sample counterpart of the inner-loop bias in Lemma~\ref{lem:estimator}(ii), which the deterministic-adaptation objective does not model; the debiased objective of~\cite{sgmrl} is the principled remedy.

The initialization used here consumed $K\tau n(m_{\mathrm{in}}+m_{\mathrm{out}})=90{,}000$ meta-training trajectories, that is $15{,}000$ per training agent. The deployment-time comparison in Section~\ref{sec:experiments} counts only adaptation samples and treats meta-training as a one-time cost amortized across the population of agents that later adapt from the initialization; it is favorable only when that population is large enough to absorb that cost.

\begin{figure}[ht]
\centering
\includegraphics[width=0.62\linewidth]{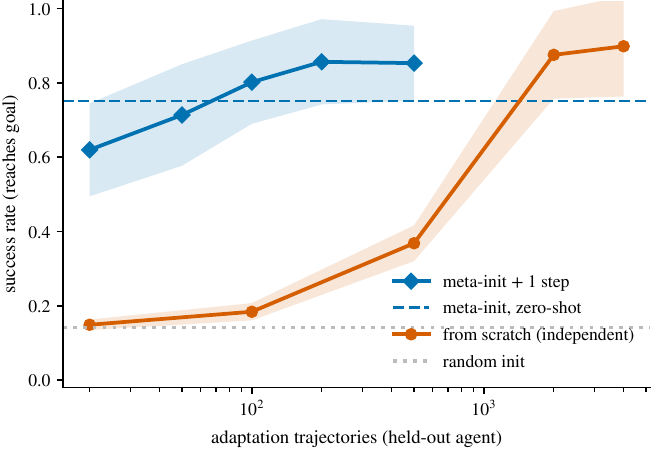}
\caption{\textbf{Success rate on held-out agents} (neural, Monte Carlo evaluation): fraction of held-out episodes that enter the goal region, against the number of trajectories used to form the single adaptation step. Companion to Figure~\ref{fig:neural} (right), which reports return on the same runs.}
\Description{A line chart of success rate on held-out agents against adaptation trajectories on a logarithmic axis. The meta-initialization curve starts around 0.62 at the smallest batch, rises through a horizontal dashed line marking the 0.75 zero-shot level, and flattens near 0.85. The from-scratch curve stays flat near 0.15, level with a dotted line for a random initialization, until the largest budgets, where it climbs steeply to about 0.9.}
\label{fig:success}
\end{figure}

\begin{figure}[ht]
\centering
\includegraphics[width=0.92\linewidth]{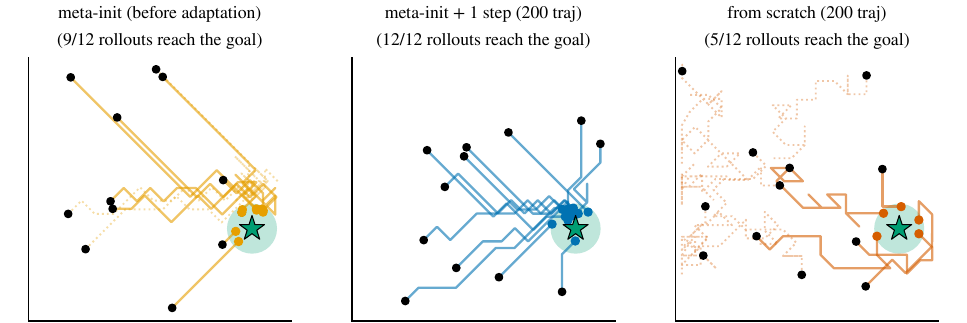}
\caption{\textbf{Held-out rollouts} (neural). Trajectories on an \emph{unseen} goal (green star, shaded goal region; black dots are episode starts), each drawn until it first reaches the goal region (dotted paths never reach; reach counts in the panel titles). Before adaptation the first-order initialization already steers toward the region and reaches it in most rollouts; after a single $200$-trajectory adaptation step every rollout reaches the goal; from-scratch training at the same budget reaches it in fewer than half.}
\Description{Three side-by-side square arenas, each showing twelve rollout paths from black starting dots toward a green star inside a shaded circular goal region on the right. Panel titles give the reach counts. Before adaptation, nine of twelve paths converge on the region and the rest wander. After one 200-trajectory adaptation step, all twelve converge in direct routes. From-scratch training at the same budget reaches the region in only five of twelve, with the remaining paths shown dotted and scattered across the arena.}
\label{fig:e7}
\end{figure}

\subsection{Reproducibility}
Every run is seeded. The released test suite asserts bitwise-identical results for repeated seeds on a fixed machine, and agreement between analytic gradients and finite differences to $5\times10^{-8}$ for the exact tabular gradients and $6\times10^{-7}$ for the neural score gradients.
All figures regenerate from committed result files, so they reproduce exactly without rerunning the stochastic experiments. Code, configurations, seeds and result files are available at \url{https://github.com/AliBeikmohammadi/Per-FedAvg-PG}.

\subsection{Scope of the empirical evidence}\label{app:scope}
The experiments show that the adaptation step contributes measurable value, that $\alpha$ behaves as a continuous personalization parameter with the predicted $\alpha\to0$ limit, that the curvature estimate is the binding practical constraint, and that a MAML initialization transfers to unseen agents in a shared-structure task family. They do not establish scale ($n\le8$ agents, $d=360$), and they include no comparison against alternative personalization mechanisms such as global--local mixing~\cite{zoapfpg,pfopg} or shared representations~\cite{xiong25}; a matched-budget comparison of mechanisms is a separate study. In the tabular family a per-agent specialist beats any shared initialization, so we do not report that comparison: with no parameters shared across states there is nothing for an initialization to transfer, and the comparison would be uninformative rather than favorable.

\section{Limitations and future work}\label{app:limits}
This section states the limitations of the analysis in full and the extensions they motivate. The scope of the empirical evidence is discussed separately in \suppexp.

\emph{(a) Stationarity, not optimality.} The guarantee is first-order stationarity of a nonconcave objective. Second-order guarantees would require combining the analysis with saddle-point escape.

\emph{(b) Full participation.} We analyze full participation throughout. Under uniform client sampling an additional heterogeneity term enters the drift bound, and that term is itself automatically bounded in terms of $\GF$ in the present setting. A control-variate correction in the spirit of~\cite{scaffold} would remove the sampling penalty using server-side controls and no extra uplink, but it modifies the aggregation rule and hence the telescoping argument of Theorem~\ref{thm:main}.

\emph{(c) Inner-loop bias.} The bias is $\mathcal O(1/m_{\mathrm{in}})$ and does not vanish at fixed batch size. Replacing the deterministic-adaptation objective by the debiased stochastic-adaptation objective of SG-MRL~\cite{sgmrl} would remove it by construction, at the cost of a different objective; variance-reduced inner estimators are the route to $\mathcal O(\varepsilon^{-2})$ trajectory complexity (Remark~\ref{rem:sharpcx}).

\emph{(d) Worst-case constants.} The automatic constants $\GJ,\LJ,\rJ$ are uniform worst-case quantities scaling as $(1-\gamma)^{-2}$ and linearly in $H$. They establish that no extra assumption is needed, not that the resulting rates are tight.

\fi

\end{document}